\documentclass[11pt]{article}
\usepackage[utf8]{inputenc}

\usepackage{natbib}
\usepackage{wustyle}

\usepackage[colorlinks,citecolor=blue,urlcolor=blue,linkcolor=blue,linktocpage=true,pagebackref=true]{hyperref}
\renewcommand*{\backref}[1]{}
\renewcommand*{\backrefalt}[4]{(\small%
    \ifcase #1 not cited%
          \or cited on page~#2%
          \else cited on pages #2%
    \fi%
    )}
  
\usepackage{multirow}
\usepackage{makecell}
\usepackage{fontawesome5} 
\usepackage{enumitem}
\usepackage[most]{tcolorbox}

\newtcolorbox{highlightedtheorem}{%
  enhanced,
  breakable,
  colback=black!2,
  colframe=black!72,
  boxrule=0.65pt,
  arc=2pt,
  outer arc=2pt,
  left=3pt,
  right=3pt,
  top=3pt,
  bottom=3pt,
  before skip=10pt,
  after skip=10pt
}

\usepackage[none]{hyphenat}
\definecolor{sgdblue}{HTML}{0072B2}
\definecolor{polyakorange}{HTML}{D55E00}
\definecolor{nestgreen}{HTML}{009E73}

\newcommand{\sgd}{\mathrm{sgd}}
\newcommand{\polyak}{\mathrm{polyak}}
\newcommand{\nesterov}{\mathrm{nesterov}}
\newcommand{\crit}{\mathrm{crit}}
\newcommand{\la}{\mathrm{la}}

\newcommand{\blfootnote}[1]{%
  \begingroup
  \renewcommand\thefootnote{}%
  \footnote{#1}%
  \addtocounter{footnote}{-1}%
  \endgroup
}

\title{Momentum in large-batch training: Polyak enlarges the  critical batch size, Nesterov improves data efficiency}

\author{
    Jia-Nan Wang$^{1,*}$\quad
    Zixun Huang$^{3,*}$\quad
    Kairui Li$^{1,*}$\quad
    Zilin Wang$^{1}$\quad
    Lei Wu$^{1,2,\dagger}$
      \\[.5em]
  $^1$Peking University\quad
  $^2$AI for Science Institute, Beijing\\[-.1em]
  $^3$University of Pennsylvania\\[.5em]
}
\date{\vspace*{-2em}}

\begin{document}

\maketitle

\blfootnote{%
\begin{tabular}[t]{@{}l@{\ }l@{}}
$^*$Equal contribution. Emails: &
\texttt{\{jiananwang25,kaisms\}@stu.pku.edu.cn},
\texttt{zixunh@wharton.upenn.edu}
\end{tabular}%
}

\blfootnote{$^\dagger$Corresponding author. Email:
\texttt{leiwu@math.pku.edu.cn}}

\begin{abstract}
We study when and how momentum improves large-batch training in the one-pass regime, using power-law kernel (PLK) regression as a tractable setting. We first characterize stochastic stability through the \emph{critical learning rate}, defined as the largest learning rate for stable training, and obtain \(\eta_{\sgd}^{\crit}\eqsim 1\), \(\eta_{\polyak}^{\crit}\eqsim\min\{1,B(1-\beta)\}\), and \(\eta_{\nesterov}^{\crit}\eqsim\min\{1,B^q(1-\beta)\}\), where \(B\) is the batch size, \(\beta\in[0,1)\) is the momentum factor, and \(q>1\) is the capacity exponent governing spectral decay. Within the stable region, we derive scaling laws for the full risk dynamics, capturing the progression from an early transient through power-law decay to a noise floor. We then optimize the final-step risk under a fixed data budget, obtaining a three-regime phase diagram that reveals how the benefits of momentum change with batch size. Notably, Polyak enlarges the critical batch size, the largest batch size preserving the optimal small-batch data-scaling exponent, thereby allowing fewer sequential updates without sacrificing data efficiency. In contrast, Nesterov improves the optimal data-scaling exponent in the large-batch regime through look-ahead-induced suppression of noise amplification. Experiments on PLK regression quantitatively support the theoretical scaling predictions, while neural-network experiments provide qualitative evidence for the predicted batch-size dependence and stability behavior.
\end{abstract}

\vspace*{.5em}
\begin{figure}[H]
\centering
\includegraphics[width=0.85\textwidth]{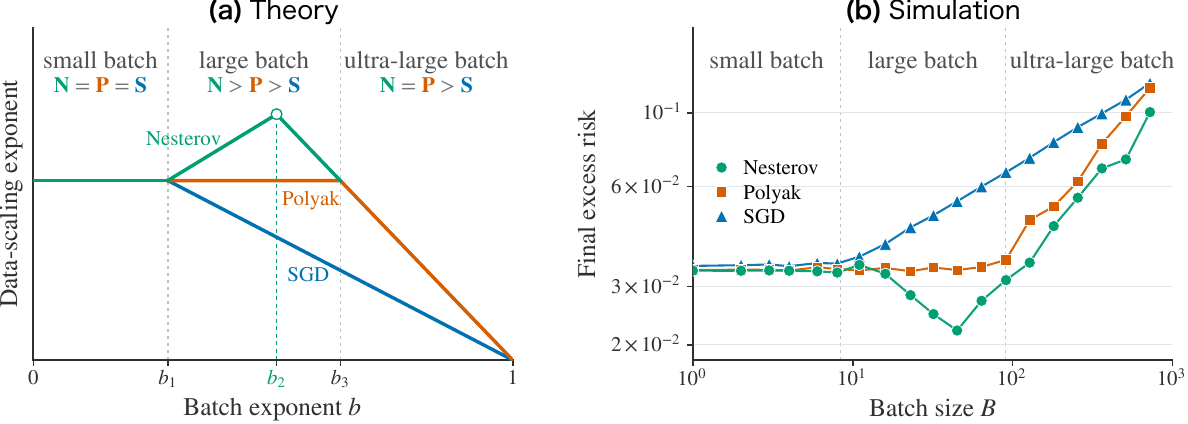}
\caption{\textbf{Batch-size dependence of momentum benefits:} \textcolor{polyakorange}{\textbf{Polyak}} enlarges the critical batch size, while \textcolor{nestgreen}{\textbf{Nesterov}} improves data efficiency. \textbf{(a)} Theoretical batch-size phase diagram from Theorem~\ref{thm:batch-size-scaling}, showing the optimal data-scaling exponent as a function of the batch exponent \(b\), defined by the scaling \(B=D^{b+o(1)}\). 
\textbf{(b)} Best-tuned final excess risk versus batch size $B$ at fixed data budget \(D=2^{14}\) for power-law kernel regression \((s,q)=(0.3,4)\). The simulation recovers the same qualitative three-regime structure predicted by the theory. Experimental details are provided in Section~\ref{sec:experiment}.
}
\label{fig:intro-result}
\end{figure}
\clearpage


\doparttoc
\faketableofcontents
\part{}

\vspace{-2em}
\section{Introduction}
Momentum has become a central mechanism in  optimizer design for modern large-scale training, underlying popular optimizers such as Adam~\citep{kingma2015adam} and Muon~\citep{jordan2024muon,liu2025muon}. Classical theory views momentum primarily as an acceleration mechanism,   quantified through  reduced iteration complexity
\citep{nesterov1983method,su2016differential,wibisono2016variational,
wilson2021lyapunov,gitman2019understanding,velikanov2023view,
wang2024marginal,liu2019acceleratingsgdmomentumoverparameterized,
gupta2024nesterov}. Yet iteration complexity alone does not fully capture the role of momentum in modern large-scale training, where the optimization regime itself changes as training resources, such as data and hardware memory, scale up.

Modern large-scale training, particularly language-model pretraining, often operates in a \emph{one-pass training} regime~\citep{brown2020language}. Under a fixed data budget \(D\), the batch size \(B\) determines the number of sequential updates, \(N=D/B\), directly trading greater hardware parallelism for fewer optimization steps. Recent language-model studies further show that the relative advantages of different optimizers can change substantially with batch size, with some exhibiting markedly stronger data efficiency in the large-batch regime~\citep{marek2026small,shah2025practical,semenov2025benchmarking}. This motivates a systematic characterization of momentum across batch sizes:
\vspace{-.3em}

\begin{highlightedtheorem}
\begin{center}
\textbf{Question:} {\it Can momentum enable larger batches without sacrificing data efficiency, or\\
even improve data efficiency itself?}
\end{center}
\end{highlightedtheorem}

\vspace{-.4em}
We address this question for both Polyak~\citep{polyak1964some} and Nesterov~\citep{nesterov1983method} methods, using power-law kernel (PLK) regression as a tractable learning problem. Its infinite-dimensional power-law spectrum spans a broad range of learning timescales, making the effects of momentum across spectral directions analytically accessible. The problem is characterized by two exponents: the capacity exponent $q>1$, which controls the decay of the feature spectrum, and the source exponent $s>0$, which controls how the target-function energy is distributed across the spectrum. Related power-law settings have been widely used to understand scaling laws in large-scale training across a variety of settings~\citep{lin2024scaling,bahri2024explaining,bordelon2025feature,li2026functional,wang2026fast,li2026optimal}. Building on the functional scaling law (FSL) framework of \citet{li2026functional}, we extend the analysis to momentum methods.

Our analysis of one-pass training follows the hierarchy 
\[ 
\mathcal A_B \;\longrightarrow\; \mathcal E(k,\eta,\beta,B) \;\longrightarrow\; \min_{(\eta,\beta)\in\mathcal A_B} \mathcal E\!\left(N,\eta,\beta,B\right),\quad N=\frac{D}{B}. 
\] 
Here, $\mathcal A_B$ denotes the admissible set of learning rates $\eta$ and momentum factors $\beta$ at batch size $B$, and $\mathcal E(k,\eta,\beta,B)$ denotes the expected excess risk. The three stages correspond, respectively, to \textbf{1)} identifying the admissible hyperparameter region through stability analysis, \textbf{2)} deriving scaling laws that characterize the risk dynamics, and \textbf{3)} optimizing the final-step risk subject to both the admissibility constraint $(\eta,\beta)\in\mathcal A_B$ and the one-pass constraint $N=D/B$.

\vspace{-.6em}
\paragraph{Our contributions.}
Our main contributions are summarized as follows.
\begin{itemize}
\vspace{-.3em}
\item \textbf{Risk stability.}
We characterize risk stability through the \emph{critical learning rate} \(\eta^{\crit}\), the largest learning rate for which the stochastic dynamics remain stable (Section~\ref{sec:stability}). In the large-batch, high-momentum regime, we obtain the sharp scalings
\begin{equation}\label{eqn: stability-0}
\eta_{\sgd}^{\crit}\eqsim 1,\quad
\eta_{\polyak}^{\crit}\eqsim \min\{1,B(1-\beta)\},\quad
\eta_{\nesterov}^{\crit}\eqsim \min\{1,B^q(1-\beta)\}.
\end{equation}
Thus, stronger momentum requires a  larger batch to maintain stability. This compensation scales as \(B\) for Polyak but as \(B^q\) for Nesterov. Since \(q>1\), Nesterov permits substantially higher momentum in large-batch training.
\vspace{-.5em}
\item \textbf{Scaling laws for risk dynamics.}
Within the stable region, we extend the functional scaling law framework~\citep{li2026functional} to Polyak and Nesterov and obtain accurate descriptions of their risk dynamics (Section~\ref{sec:scaling law}). In the high-momentum regime, both methods first exhibit an exponential transient on the timescale \(O((1-\beta)^{-1})\), after which they enter the scaling-law regime. Writing \(\eta_\beta:=\eta/(1-\beta)\), Polyak follows the same signal--noise tradeoff as SGD, whose risk scales as \(\mathcal E_{\sgd}(k)\eqsim(k\eta)^{-s}+\sigma^2\eta/B\), but with the enlarged effective learning rate \(\eta_\beta\):
\[
\mathcal E_{\polyak}(k)\eqsim(k\eta_\beta)^{-s}+\frac{\sigma^2\eta_\beta}{B}.
\]
Nesterov retains the same accelerated signal term while further suppressing noise accumulation:
\[
\mathcal E_{\nesterov}(k)\eqsim(k\eta_\beta)^{-s}+\frac{\sigma^2}{B}\min\{\eta_\beta,\eta_\beta^{1/q}\}.
\]
\vspace{-.9em}
\item \textbf{Data scaling and batch-size phase diagram.}
Let \(B=D^{b+o(1)}\). Optimizing the remaining hyperparameters yields three characteristic batch exponents \(b_1<b_2<b_3\) and a three-regime phase diagram (Section~\ref{sec:data-scaling}). For \(b\leq b_1\), SGD, Polyak, and Nesterov achieve the same data-scaling exponent. For \(b_1<b<b_3\), Nesterov outperforms Polyak, which in turn outperforms SGD, with Nesterov attaining its best data-scaling rate at \(b=b_2\). For \(b\geq b_3\), Polyak and Nesterov again share the same exponent and both outperform SGD. Hence, \(b_1\) and \(b_3\) mark the critical batch-size exponents of SGD and Polyak, respectively. Polyak retains the best small-batch scaling over a much wider range of batch sizes, while Nesterov further improves the data-scaling rate in the large-batch regime.
\end{itemize}
 Figure~\ref{fig:intro-result} summarizes how the relative advantage of the three methods changes with batch size, and Section~\ref{sec:experiment} empirically validates these theoretical predictions.

\begin{remark}[Why momentum helps large-batch training]\label{remark: persample-ELR}
The batch-size dependence of momentum can be understood through the \textit{per-sample effective learning rate} (ELR),
\[
\eta_{\rm sample}:=\frac{\eta_{\rm eff}}{B},\qquad
\eta_{\rm eff}=
\begin{cases}
\eta, & \textnormal{SGD},\\
\eta_\beta:=\eta/(1-\beta), & \textnormal{momentum}.
\end{cases}
\]
Ignoring the early transient, SGD and Polyak obey the same data-level signal--noise tradeoff,
\begin{equation}\label{eqn: risk-0}
\mathcal E_{D}(\eta)\eqsim(D\eta_{\rm sample})^{-s}+\sigma^2\eta_{\rm sample}.
\end{equation}
The difference is the admissible range of the per-sample ELR. Stability condition~\eqref{eqn: stability-0} gives \(\eta_{\rm sample}\lesssim1/B\) for SGD, whereas Polyak allows \(\eta_{\rm sample}\lesssim1\). Thus, as the batch size grows, the admissible per-sample ELR of SGD shrinks as \(1/B\), while Polyak can keep it at order one by increasing \(\eta_\beta\) proportionally with \(B\). Polyak therefore extends the SGD signal--noise tradeoff to substantially larger batch sizes without changing its optimum, explaining both the enlarged critical batch size and the unchanged best data-scaling exponent.

In contrast, Nesterov changes the picture more fundamentally. Its look-ahead mechanism suppresses noise accumulation, improving the signal--noise tradeoff itself and thereby achieving better data efficiency in the large-batch regime.
\end{remark} 

\section{Related work}

\paragraph{Theory of momentum.}
Classical analyses explain momentum through accelerated convergence, Lyapunov functions, and continuous-time analysis~\citep{nesterov1983method,su2016differential,wibisono2016variational,wilson2021lyapunov,shi2022understanding,kovachki2021continuous}. Under stochastic training, subsequent work has studied convergence, stability, and asymptotic behavior~\citep{vidyasagar2025convergence,liu2020improved,yuan2016influence,gitman2019understanding,velikanov2023view,zhang2024optimalityacceleratedsgdhighdimensional}. However, existing theory does not characterize how the role of momentum changes with batch size under a fixed data budget. We instead treat batch size as a scaling variable and compare Polyak and Nesterov across the resulting training regimes.

\paragraph{Scaling-law theory.}
Scaling laws characterize how learning performance scales with resources such as data, model, and compute~\citep{hestness2017deep,kaplan2020scaling,hoffmann2022training,grattafiori2024llama,deepseekai2026deepseekv4}. A growing theoretical literature has derived such laws in linear, kernel, and related tractable models~\citep{sharma2022scaling,hutter2021learning,maloney2022solvable,wei2022more,jain2024scaling,michaud2023quantization,nam2024exactly,atanasov2026scaling,bahri2024explaining,dohmatob2024tale,bordelon2024dynamical,lin2024scaling,paquette2024phases,bordelon2025feature,yan2026larger,lin2025improved,sous2026asymmetric,li2026functional,zhang2025critical,wang2026fast,li2026optimal,ding2025scalinglawstochasticgradient}. Functional scaling laws further characterize risk dynamics through signal learning and noise accumulation~\citep{li2026functional}. We extend this framework to momentum methods. Closely related, \citet{ferbach2026dimension} study momentum scaling laws in noiseless power-law random features at fixed batch size, showing that Polyak preserves the SGD exponents while adapted Nesterov-type methods can improve them. We instead study the noisy setting with batch size scaling under a fixed data budget, deriving a phase diagram for canonical Polyak and Nesterov momentum.

\paragraph{Stability of gradient-based methods.}
Stability characterizes the admissible hyperparameter region and has been studied extensively for gradient-based methods, with connections to implicit bias, generalization, and optimization geometry~\citep{zhu2019anisotropic,wu2018sgd,nar2018step,ma2021linear,wu2022alignment,wu2023implicit,chemnitz2025characterizing,mulayoff2021implicit,nacson2023implicit,qiao2024stable,kaur2023maximum,wang2023theoretical}. For stochastic momentum methods, the joint dependence of stability on batch size and momentum remains less understood. We derive sharp risk stability boundaries for Polyak and Nesterov.

\paragraph{Concurrent work.} In work that appeared while this manuscript was in preparation,
\citet{morwani2026compute} study stochastic momentum in finite-dimensional linear regression and derive lower bounds on contraction rates, suggesting that Polyak may remain data-efficient over a wider range of batch sizes; matching upper bounds remain open. In contrast, we derive sharp scaling laws for the full risk dynamics and, after optimizing the learning rate and momentum factor, obtain the data-scaling exponent across batch sizes. This yields a complete phase diagram, including a large-batch regime where Nesterov outperforms Polyak.

\section{Preliminaries}
\label{sec:preliminaries}
\vspace{-.5em}
\paragraph{Notation.} We use $\eqsim$ to denote equivalence up to a constant factor, and $\lesssim$ (resp.~$\gtrsim$) indicates inequality up to a constant factor. For two non-negative functions $f,g$, we write $f(x)\lesssim g(x)$ (resp. $f(x)\gtrsim g(x)$) if there exists a constant $C>0$ such that $f(x)\leq C g(x)$ (resp. $g(x)\leq C f(x)$) for all $x$ under consideration. We write $f(x)\eqsim g(x)$ if both $f(x)\lesssim g(x)$ and $f(x)\gtrsim g(x)$ hold. 
We denote the standard inner product by
\(\langle\cdot,\cdot\rangle\) and its induced norm by
\(\lVert\cdot\rVert_2\).

\subsection{Power-law kernel (PLK) regression}
Consider  regression  on a Hilbert space \(\mathcal H\) with the input \(\bx\sim\cD\) and the observed label \(y=\langle\bphi(\bx), \btheta^*\rangle + \epsilon\), where \(\epsilon \sim \cN(0,\sigma^2)\) is independent of \(\bx\). Assume that the feature function $\bphi(\bx)$ is Gaussian: $\bphi(\bx) \sim \cN(0,\bH)$, where $\bH = \EE_{\bx\sim \cD}[\bphi(\bx) \otimes\bphi(\bx)]$ denotes the feature covariance operator. Let \(\{\lambda_j\}_{j\geq1}\) denote the eigenvalues of \(\bH\), arranged in decreasing order.
We learn the target function using the linear student model $f(\bx;\btheta) = \langle\bphi(\bx), \btheta\rangle$ by minimizing the population risk:
$\cR(\btheta) = \frac{1}{2} \EE_{\bx,y}\left[(f(\bx;\btheta)-y)^2\right].$
A direct calculation gives $\cR(\btheta) = \cE(\btheta) + \frac{\sigma^2}{2}$, where $\cE(\btheta) = \frac{1}{2}\langle\btheta-\btheta^*,\bH (\btheta-\btheta^*)\rangle$ is the  excess risk.
Since SGD, Polyak, and Nesterov momentum are orthogonally equivariant, we make the following diagonal assumption without loss of generality; see Appendix~\ref{app:subsec:rotation invariant}.

\begin{assumption}[Diagonal covariance]
\label{ass:diagonal-covariance}
$
\bH=\diag(\lambda_1,\lambda_2,\cdots)$ with $
\lambda_1\geq\lambda_2\geq\cdots>0.
$
\end{assumption}

\begin{assumption}[Constant-order label noise]
Assume \(\sigma\eqsim1\).
\end{assumption}

\begin{assumption}[Source and capacity conditions] \label{ass:power-law}
    There exist  $q>1$ and $s>0$ such that 
    \[
        \lambda_j \eqsim j^{-q},\quad \lambda_j |\theta^*_j|^2\eqsim j^{-(1+sq)}.
    \]
\end{assumption}
Here \(q\) is the \textbf{capacity exponent}, which controls the decay rate of the feature spectrum, and \(s\) is the \textbf{source exponent}, which describes how the target energy decays across spectral directions. 
We also refer to \(s\) as the \textbf{relative difficulty} of the task: smaller \(s\) means that more target energy lies in low-curvature, small-eigenvalue directions, making the problem harder to learn. We remark that similar assumptions have been widely used in the analysis of kernel methods~\citep{caponnetto2005fast,caponnetto2007optimal,spigler2020asymptotic,bordelon2020spectrum,maloney2022solvable}.  

\subsection{Optimization methods}
\label{sec:err recursion}
At each step, we draw a mini-batch $\cB_k =\{(\bx_{k,i}, y_{k,i})\}_{i=1}^{B}$ and compute the stochastic gradient $\bg(\btheta;\cB_k) = \nabla_{\btheta}(\tfrac{1}{2B}\sum_{(\bx,y)\in\cB_k}(f(\bx;\btheta)-y)^2)$. SGD and Polyak update as follows:
\begin{align}\label{eq:sgd-update}
\textnormal{SGD:}&\qquad
\btheta_{k+1}
=
\btheta_k-\eta \bg(\btheta_k;\cB_k),\\
\label{eq:hb-update}
\textnormal{Polyak:}&\qquad
\btheta_{k+1}
=
\btheta_k
-\eta\bg(\btheta_k;\cB_k)+\beta(\btheta_k-\btheta_{k-1}),
\end{align}
where $\eta$ is the learning rate and $\beta\in [0,1)$ is the momentum factor. Nesterov momentum differs from Polyak momentum by evaluating the gradient at a look-ahead point:
\begin{align}
\notag \textnormal{Nesterov:}&\qquad\btheta_k^{\la}=\btheta_k+\beta(\btheta_k-\btheta_{k-1})\qquad \text{(look ahead)}\\
&\qquad \btheta_{k+1}=\btheta_k -\eta\bg(\btheta_k^{\la};\cB_k)+\beta(\btheta_k-\btheta_{k-1}),
\label{eq:nag-update}
\end{align}
where 
\(\btheta_k^{\la}\) is the look-ahead point obtained by extrapolating the current iterate along the previous update direction. This look-ahead mechanism distinguishes Nesterov from Polyak. 

For clarity, we decompose  \(
\bg(\btheta;\cB_k)
=
\nabla\cR(\btheta)+\bxi_k(\btheta),
\)
where \(\nabla\cR(\btheta)\) denotes the population gradient and \(\bxi_k(\btheta)\) represents  the stochastic gradient noise satisfying
$
\EE[\bxi_k(\btheta)]=\bm 0
$
and $
\EE\!\left[
\bxi_k(\btheta)\otimes\bxi_k(\btheta)
\right]
=
\frac{1}{B}\bSigma(\btheta),
$
where \(\bSigma(\btheta)\) denotes the gradient-noise covariance of batch size \(1\). 
We next recall a noise--curvature alignment property that will be used
throughout our analysis:
\begin{proposition}[Noise--curvature alignment]
\label{prop:noise-geometry}
For any unit vector \(\bnu\in\mathcal H\),
\begin{equation}
\label{eqn:noise-geometry}
\mathbb E\!\left[
\left|
\left\langle\bxi_k(\btheta_k),\bnu\right\rangle
\right|^2
\,\middle|\,
\btheta_k
\right]
\eqsim
\frac{1}{B}
\bigl(\cE(\btheta_k)+\sigma^2\bigr)
\left\langle\bnu,\bH\bnu\right\rangle.
\end{equation}
\end{proposition}

Here, \(\langle\bnu,\bH\bnu\rangle\) is the curvature along \(\bnu\). The two terms in Proposition~\ref{prop:noise-geometry} have distinct origins and roles. The term proportional to \(\cE(\btheta_k)\) arises from mini-batch sampling, depends on the current iterate, and vanishes at the optimum, acting as \emph{multiplicative noise}. The term proportional to \(\sigma^2\) arises from label noise and persists at the optimum, acting as \emph{additive noise}. Both are curvature-aligned but affect the dynamics differently: multiplicative noise feeds back into the dynamics and determines stability, whereas additive noise determines the asymptotic noise floor. This noise--curvature alignment has been observed in neural networks~\citep{zhu2019anisotropic,wu2022alignment} and established rigorously in toy models~\citep{wang2023theoretical}. We use this  constitutive relation throughout our analysis; see Appendix~\ref{app:hessian alignment} for a proof.

\clearpage
\subsection{Regime-dependent damping in momentum dynamics}
\label{sec:damping-regimes}
\vspace{-.5em}
\begin{minipage}[t]{0.69\textwidth}
\vspace{0pt}
Consider full-batch Polyak for minimizing
\(f(x)=\lambda x^2/2\),
where \(\lambda>0\) denotes the curvature.
Since \(\nabla f(x)=\lambda x\), the Polyak update is:
\begin{equation}\label{eqn: polyak-1}
x_{k+1}
= x_k-\eta\lambda x_k+\beta(x_k-x_{k-1}).
\end{equation}
The corresponding characteristic equation is
\(r^2-(1+\beta-\eta\lambda)r+\beta=0\).
The relevant transition from real to complex roots occurs at
\(\eta\lambda=(1-\sqrt{\beta})^2\).
Hence, different curvature directions can exhibit
two qualitatively different damping behaviors.
\end{minipage}\hfill%
\begin{minipage}[t]{0.27\textwidth}
\vspace{0pt}
\centering
\includegraphics[width=\linewidth]
  {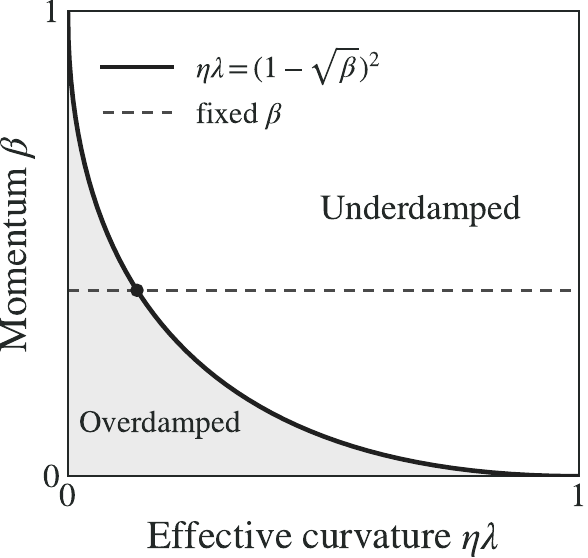}
\end{minipage}
\begin{itemize}[leftmargin=2em]
    \item \textbf{Overdamped regime}: $\lambda \eta  < (1-\sqrt{\beta})^2$. In this regime, the characteristic roots $r_{\pm}$  are real with the dominant root
satisfying
$
r_+
\approx
1-\frac{\eta\lambda}{1-\beta}
$
when $1-\beta=o(1)$.
For sufficiently large $k$, it holds that
$
x_k =c_{+} r_{+}^k + c_{-} r_{-}^k \approx c_{+} r_{+}^k 
\approx
c_{+}\exp(
-\frac{\eta\lambda}{1-\beta}k).
$
 Hence, the resulting dynamics behave like gradient descent with the {\it momentum-rescaled learning rate} 
\begin{equation}\label{eqn: elr}
\eta_{\beta} = \frac{\eta}{1-\beta}.
\end{equation}
\vspace{-1em}

\item \textbf{Underdamped regime}: $\lambda\eta > (1-\sqrt{\beta})^2$. In this regime, the characteristic roots form a complex-conjugate pair with modulus
\(\sqrt{\beta}\). The iterate therefore has the form
$
x_k
=
\beta^{k/2}
\bigl[c_1\cos(k\omega)+c_2\sin(k\omega)\bigr],
$
where \(c_1,c_2\), and \(\omega\) depend on the initialization and curvature. Thus, \(x_k\) exhibits oscillatory convergence to zero, with the amplitude decaying as \(\beta^{k/2}\). 
\end{itemize}


The key difference between the two regimes lies in how curvature affects
the dynamics. In the overdamped regime, the decay rate
\(\eta\lambda/(1-\beta)\) is proportional to the curvature, so flatter
directions converge more slowly. In the underdamped regime, the
oscillation frequency depends on the curvature, but the amplitude
envelope \(\beta^{k/2}\) does not.

\textbf{Nesterov look-ahead induces curvature-adaptive damping.} On the same quadratic objective, Nesterov updates using the gradient at the look-ahead point $x_k^{\mathrm{la}}= x_k+\beta(x_k-x_{k-1})$:
\begin{align}\label{eqn: curvature-adaptivity}
\notag x_{k+1} &= x_k - \eta \lambda [x_k+\beta(x_k-x_{k-1})] + \beta (x_k - x_{k-1})\\
&=x_k - \eta\lambda x_k + \beta(1-\eta\lambda) (x_k -x_{k-1}),
\end{align}
which is equivalent to Polyak with the momentum factor $\beta(1-\eta\lambda)$. 
For an underdamped direction, it accelerates convergence:
$
\beta^{k/2}
\to 
[\beta(1-\eta\lambda)]^{k/2},
$ 
which damps sharper directions more rapidly.
By contrast, in sufficiently flat directions where
\(\eta\lambda\ll1\), we have
\(\beta(1-\eta\lambda)\approx\beta\), and Nesterov retains
approximately the same  dynamics as Polyak.

\section{Risk stability and critical learning rates}
\label{sec:stability}

We first characterize the maximal learning rates that ensure stable training and study their dependence on the momentum factor and batch size. 

\begin{definition}[Risk stability]
Consider a stochastic optimization algorithm for minimizing a population risk $\mathcal R(\cdot)$, producing model parameters $\{\btheta_k\}_{k\ge 0}$. It is said to be \emph{risk stable} if, for every admissible initialization with finite population risk, 
$
\sup_{k\ge 0}\mathbb E[\mathcal R(\btheta_k)]<\infty.
$
\end{definition}

\begin{definition}[Critical learning rate]
Consider a stochastic optimization algorithm $\mathcal A$ parameterized by a learning rate $\eta$ and additional hyperparameters $\bh$. For fixed $\bh$,  let
\[
\eta_{\mathcal A}^{\crit}(\bh):=\sup\left\{\eta>0:\mathcal A\text{ is risk stable}\right\}.
\]
\end{definition}
The critical learning rate is the largest learning rate compatible with risk stability. When the dependence on $\bh$ is clear from context, we simply write $\eta_{\mathcal A}^{\crit}$. In our setting, $\bh=B$ for SGD and $\bh=(B,\beta)$ for Polyak and Nesterov. The following theorem characterizes how the critical learning rate scales with batch size and momentum; the proof is deferred to Appendix~\ref{app:sec:stability}.
\begin{highlightedtheorem}
\begin{theorem}[Critical learning-rate scaling]\label{thm:stability}
\normalfont
For sufficiently large \(B\) and sufficiently small \(1-\beta\), 
\begin{align}
\eta_{\sgd}^{\crit} &\eqsim 1,\\
\eta_{\polyak}^{\crit} &\eqsim \min\{1,\,B(1-\beta)\},\\
\eta_{\nesterov}^{\crit} &\eqsim \min\{1,\,B^q(1-\beta)\},
\end{align}
where the implicit constants are independent of \(B\) and \(\beta\).
\end{theorem}
\end{highlightedtheorem}

\vspace{-1.1em}
\paragraph{Joint dependence on momentum and batch size.}
Theorem~\ref{thm:stability} reveals a direct coupling among the learning
rate, momentum factor, and batch size in determining training stability.
Increasing $\beta$ toward one reduces the critical learning rate, whereas
increasing the batch size enlarges it. In the high-momentum regime,
$
\eta_{\polyak}^{\crit}\eqsim B(1-\beta),
\eta_{\nesterov}^{\crit}\eqsim B^q(1-\beta),
$
before reaching their order-one ceilings. Thus, the batch-size effect on
stability is substantially stronger for Nesterov, with a $B^q$
dependence rather than the $B$ dependence of Polyak. This difference originates from the curvature adaptivity induced by Nesterov's look-ahead mechanism.

\vspace{-.5em}
\paragraph{Multiplicative-noise feedback governs stability through the renewal kernel.} For all three methods, Appendix~\ref{app:sec:stability} shows that the risk dynamics admit, up to constant factors, a renewal equation
$
\ell_k\eqsim h_k+\sum_{j=0}^{k-1}\cK_{k-1-j}\bigl(\ell_j+\sigma^2\bigr),
$
where \(\ell_k\) is the risk moment, \(h_k\) is the free response, and the renewal kernel \(\cK\) describes how stochastic perturbations propagate through the optimizer dynamics. Additive and multiplicative noise play different roles: additive label noise acts only as an external forcing, whereas multiplicative noise feeds the current risk back through the renewal kernel, producing the feedback term \(\cK*\ell\). Risk stability is therefore determined by the strength of this feedback, characterized by \(\|\cK\|_1<1\); see Appendix~\ref{app:subsec:label-noise-stability}. Since \(\cK\) is independent of the label-noise scale \(\sigma\), the stability condition in Theorem~\ref{thm:stability} is also independent of \(\sigma\). 

\section{Scaling laws for risk dynamics across damping regimes}
\label{sec:scaling law}

In this section, we establish scaling laws for the risk dynamics of Polyak and Nesterov across different damping regimes. The analysis follows how increasing \(\beta\) drives transitions in damping and stability: from a fully overdamped regime with SGD-like dynamics, to a mixed-damping regime, and eventually to instability.

As discussed in Section~\ref{sec:damping-regimes}, the transition between overdamped and underdamped dynamics for Polyak occurs at
$
(1-\sqrt{\beta})^2=\eta\lambda.
$
For the scaling analysis below, where \(\beta\to1\), we use the equivalent high-momentum form
$
\eta\lambda\eqsim(1-\beta)^2.
$
This motivates the curvature cutoff
\vspace{-.5em}
\[
\lambda_{\mathrm{cut}}
\eqsim
\frac{(1-\beta)^2}{\eta},
\qquad
\II_{\mathrm{over}}
=
\left\{
j:\lambda_j\lesssim\lambda_{\mathrm{cut}}
\right\},
\qquad
\II_{\mathrm{under}}
=
\left\{
j:\lambda_j\gtrsim\lambda_{\mathrm{cut}}
\right\}.
\]
Hence, the sets \(\II_{\mathrm{over}}\) and \(\II_{\mathrm{under}}\) consist of flat overdamped directions and sharp underdamped directions, respectively. Since the PLK spectrum contains arbitrarily small eigenvalues, overdamped directions always remain present.

As $\beta$ increases, the first underdamped directions appear when \(\lambda_{\mathrm{cut}}\lesssim \lambda_1\), corresponding to
\(
1-\beta\eqsim\sqrt{\eta}.
\)
Combining this damping threshold with the risk stability conditions from Section~\ref{sec:stability} yields three regimes:
\[
\begin{array}{lll}
\textnormal{fully overdamped:}
& 1-\beta\gtrsim\sqrt{\eta},\\[1mm]
\textnormal{mixed damping:}
& \eta/B^\rho\lesssim1-\beta\lesssim\sqrt{\eta},\\[1mm]
\textnormal{unstable:}
& 1-\beta\lesssim\eta/B^\rho,
\end{array}
\qquad
\rho=
\begin{cases}
1, & \textnormal{Polyak},\\
q, & \textnormal{Nesterov}.
\end{cases}
\]
Thus, the damping threshold \(1-\beta\eqsim\sqrt{\eta}\) and the risk stability threshold \(1-\beta\eqsim\eta/B^\rho\) delimit the three regimes.

For comparison, we first recall the {\bf SGD scaling law} for PLK regression:
\begin{equation}
\label{eqn:scaling-law-sgd}
\mathbb E[\cE(\btheta_k)]
\eqsim
(\eta k)^{-s}+\frac{\sigma^2}{B}\eta,
\qquad k\eta\gtrsim1.
\end{equation}
This follows from \citet[Theorem 5.2]{li2026functional}; see Appendix~\ref{app:sgd_scaling_law} for details. The first term, \((\eta k)^{-s}\), describes signal learning, with the decay rate determined by the source exponent \(s\): larger \(s\) corresponds to an easier target and faster learning. The second term, \(\sigma^2\eta/B\), captures stochastic noise accumulation and determines the asymptotic noise floor.

For Polyak momentum, let $\eta_\beta=\eta/(1-\beta)$ denote the effective learning rate. We obtain the following scaling law across the damping regimes.

\vspace{-.25em}
\begin{highlightedtheorem}
\begin{theorem}[Scaling law for Polyak]
\label{thm:hbm scaling law}
\normalfont
Assume \(\eta\lesssim1\) and \(k\gtrsim\max(\eta^{-1}_\beta,(1-\beta)^{-1})\). Then
\begin{equation}
\label{eq:hb-loss-scaling}
\mathbb E[\cE(\btheta_k)]\eqsim
\begin{cases}
(\eta_\beta k)^{-s}+\dfrac{\sigma^2}{B}\eta_\beta, & 1-\beta\gtrsim\sqrt{\eta},\\[2mm]
\TT_P(k)+(\eta_\beta k)^{-s}+\dfrac{\sigma^2}{B}\eta_\beta, & \dfrac{\eta}{B}\lesssim1-\beta\lesssim\sqrt{\eta},
\end{cases}
\end{equation}
where $\TT_P$ denotes the underdamped transient and satisfies \(0\leq \TT_P(k)\lesssim\beta^k\). If \(1-\beta\lesssim\eta/B\), the dynamics are not risk stable.
\end{theorem}
\end{highlightedtheorem}
\vspace{-.25em}
The proof is deferred to Appendix~\ref{app:hbm scaling law}. In the fully overdamped regime, Polyak follows the SGD scaling law with effective learning rate \(\eta_\beta\). In the mixed-damping regime, the transient \(\TT_P(k)\) arises from underdamped sharp modes, whereas the power-law term \((\eta_\beta k)^{-s}\) accounts for overdamped flat modes. Balancing the two terms gives the crossover timescale
\vspace{-.2em}
\begin{equation}\label{eqn: crossover}
\tau_\beta\eqsim\frac{1}{1-\beta}\log\!\frac{\eta}{(1-\beta)^2},
\end{equation}
after which the underdamped transient becomes negligible and the usual scaling-law dynamics dominate. Thus, {\bf increasing \(\beta\) accelerates learning in the flat modes through a larger \(\eta_\beta\), but also prolongs the relaxation of the sharp modes}. Importantly, Polyak preserves the SGD signal--noise tradeoff while allowing \(\eta_\beta\lesssim B\), according to Theorem~\ref{thm:stability}, and hence enlarges the critical batch size without improving the optimal data-scaling exponent.

For Nesterov, the same damping structure persists, but its look-ahead mechanism modifies both the underdamped transient and noise accumulation in the mixed-damping regime.

\vspace{-.4em}
\begin{highlightedtheorem}
\begin{theorem}[Scaling law for Nesterov]
\label{thm:nesterov scaling law}
\normalfont
Assume \(\eta\lesssim1\) and \(k\gtrsim\max(\eta^{-1}_\beta,(1-\beta)^{-1})\). Then
\begin{equation}
\label{eq:nag-loss-scaling}
\mathbb E[\cE(\btheta_k)]\eqsim
\begin{cases}
(\eta_\beta k)^{-s}+\dfrac{\sigma^2}{B}\eta_\beta, & 1-\beta\gtrsim\sqrt{\eta},\\[2mm]
\TT_N(k)+(\eta_\beta k)^{-s}+\dfrac{\sigma^2}{B}\min\{\eta_\beta,\eta_\beta^{1/q}\}, & \dfrac{\eta}{B^q}\lesssim1-\beta\lesssim\sqrt{\eta},
\end{cases}
\end{equation}
where $\TT_N$ denotes the underdamped transient and satisfies \(0\leq \TT_N(k)\lesssim\min\{1,(\eta k)^{-s}\}\beta^k\). If \(1-\beta\lesssim\eta/B^q\), the dynamics are not risk stable.
\end{theorem}
\end{highlightedtheorem}

\vspace{-.25em}
The proof is deferred to Appendix~\ref{app:nesterov scaling law}. Nesterov's advantage over Polyak arises in the mixed-damping regime, where its look-ahead mechanism induces curvature-adaptive damping; see Eq.~\eqref{eqn: curvature-adaptivity}. In particular, when \(\eta_\beta\gg1\), the noise floor is reduced from
$
\frac{\sigma^2}{B}\eta_\beta
$ 
to 
$
\frac{\sigma^2}{B}\eta_\beta^{1/q}.
$
This reduced noise amplification is the key to Nesterov's improved data-scaling exponent.

\section{Optimal data scaling across batch-size regimes}
\label{sec:data-scaling}

We now optimize the final-step risk using the scaling laws derived above, subject to the one-pass constraint $N=D/B$ and the stability constraints. We focus on the large-batch setting, allowing $B$ to grow with the data budget $D$ rather than remain fixed.
\begin{definition}[Batch-size exponent]
\label{def:batch-size-exponent}
For a sequence of training problems with \(D\to\infty\), the batch size has exponent \(b\in[0,1)\) if
$
B=D^{b+o(1)},
$
or equivalently, $
b=\lim_{D\to\infty}\frac{\log B}{\log D}.
$
\end{definition}

\begin{definition}[Batch-size-dependent data-scaling exponent]
\label{def:batch-scaling-exponent}
Fix \(b\in[0,1)\), so that \(B=D^{b+o(1)}\) and \(N=D/B\).  For an optimization method with admissible hyperparameter region \(\cA_B\), define its optimal data-scaling exponent \(r(b)\) by
\[
\inf_{(\eta,\beta)\in\cA_B}\EE[\cE(\btheta_N)]
=
D^{-r(b)+o(1)}.
\]

\end{definition}
Here, \(\beta\) is omitted for SGD. We denote the exponents for SGD, Polyak, and Nesterov by \(r_{\sgd}(b)\), \(r_{\polyak}(b)\), and \(r_{\nesterov}(b)\), respectively. A larger \(r(b)\) indicates better data efficiency.

Define the three transition exponents
\begin{equation}
b_1=\frac{s}{s+1},\qquad
b_2=\frac{2s+1/q}{2s+1+1/q},\qquad
b_3=\frac{b_1+1}{2}=\frac{2s+1}{2s+2}.
\end{equation}
Since \(q>1\) and \(s>0\), we have \(0<b_1<b_2<b_3<1\). The following theorem gives the optimal data-scaling exponents across batch-size regimes; the proof is deferred to Appendix~\ref{app:sec:data-scaling}.

\begin{highlightedtheorem}
\begin{theorem}[Batch-size-dependent data scaling]
\label{thm:batch-size-scaling}
\normalfont
For \(B=D^{b+o(1)}\) with \(b\in[0,1)\), 
\[
r_{\sgd}(b)=
\begin{cases}
\frac{s}{s+1}, & 0\leq b<b_1,\\[2mm]
s(1-b), & b_1\leq b<1,
\end{cases}
\qquad
r_{\polyak}(b)=
\begin{cases}
\frac{s}{s+1}, & 0\leq b<b_3,\\[2mm]
2s(1-b), & b_3\leq b<1,
\end{cases}
\]
and
\[
r_{\nesterov}(b)=
\begin{cases}
\frac{s}{s+1}, & 0\leq b<b_1,\\[2mm]
\frac{s(1+(q-1)b)}{sq+1}, & b_1\leq b<b_2,\\[3mm]
2s(1-b), & b_2\leq b<1.
\end{cases}
\]
\end{theorem}
\end{highlightedtheorem}

Theorem~\ref{thm:batch-size-scaling} establishes the batch-size phase diagram illustrated in Figure~\ref{fig:intro-result} (left). The three transition exponents have distinct meanings: \(b_1\) is the critical batch exponent of SGD, \(b_2\) marks the maximum data-scaling exponent of Nesterov, and \(b_3\) is the critical batch exponent of Polyak. The relative performance of the three methods is summarized in Table~\ref{tab:batch-regimes}.

\begin{table}[t]
\centering
\small
\caption{{\bf Optimal data-scaling exponents across batch-size regimes.} Here \(r_0=s/(s+1)\) denotes the optimal small-batch exponent. Nesterov attains its maximum exponent at \(b=b_2\).}
\vspace*{-.4em}
\label{tab:batch-regimes}
\renewcommand{\arraystretch}{1.2}
\begin{tabular}{@{}ll@{}}
\toprule
Batch-size regime & Optimal data-scaling exponents \\
\midrule
Small (\(b\le b_1\))
& \(r_{\nesterov}=r_{\polyak}=r_{\sgd}=r_0\) \\
Large (\(b_1<b<b_3\))
& \(r_{\nesterov}>r_{\polyak}=r_0>r_{\sgd}\) \\
Ultra-large (\(b\ge b_3\))
& \(r_{\nesterov}=r_{\polyak}>r_{\sgd}\) \\
\bottomrule
\end{tabular}
\end{table}

\paragraph{Origin of the critical batch sizes.} By Eq.~\eqref{eqn: risk-0}, SGD and Polyak achieve the optimal small-batch data-scaling rate when the per-sample ELR satisfies \(\eta_{\mathrm{sample}}^*\eqsim D^{-b_1}\). The key question is how large the batch size can be used to realize this per-sample ELR. For SGD, \(\eta_{\mathrm{sample}}=\eta/B\), giving the familiar linear scaling rule \(B\eqsim\eta D^{b_1}\). The stability ceiling \(\eta\lesssim1\) therefore limits its critical batch size to \(B_{\sgd}^{\crit}\eqsim D^{b_1}\).

For Polyak, $\eta_{\mathrm{sample}}=\frac{\eta}{B(1-\beta)}$. Hence,
the corresponding scaling rule becomes
$
B\eqsim\frac{\eta}{1-\beta}D^{b_1}.
$
Momentum thus allows larger batches without sacrificing data efficiency. However, increasing momentum also prolongs the underdamped transient. To ensure that this transient decays within the training horizon \(N=D/B\), we require \(N\gg(1-\beta)^{-1}\), up to logarithmic factors. Combining this constraint with the scaling rule gives
\[
B\eqsim\frac{\eta D^{b_1}}{1-\beta}\ll D(1-\beta) ~~\Rightarrow~~ 1-\beta \gg \sqrt{\eta} D^{-\frac{1-b_1}{2}} ~~\Rightarrow~~
B\ll\sqrt{\eta} D^{\frac{1+b_1}{2}}.
\]
Since \(\eta\lesssim1\), the largest achievable batch size is \(B_{\polyak}^{\crit}=D^{b_3+o(1)}\) with $b_3=(1+b_1)/2$.

\paragraph{Why Nesterov improves data efficiency.} Nesterov goes further by reducing the noise cost of increasing the effective learning rate \(\eta_\beta=\eta/(1-\beta)\): its noise floor scales as \(\eta_\beta^{1/q}/B\), compared with \(\eta_\beta/B\) for Polyak when \(\eta_\beta\gg1\). This allows Nesterov to achieve a higher data-scaling exponent, which peaks at \(b_2\) before the finite training horizon becomes limiting.

\vspace*{-.2em}
\section{Numerical experiments}
\label{sec:experiment}
\vspace*{-.2em}

We quantitatively validate our predictions for stability, loss dynamics, and data scaling in PLK regression, and examine the qualitative batch-size trends in ImageNet-1K training. Additional neural-network stability experiments on CIFAR-10 are provided in Appendix~\ref{app:stability-neural}.

\vspace*{-.2em}
\subsection{PLK regression}
\label{subsec:plk}
\vspace*{-.2em}

\paragraph{Risk stability and risk-dynamics scaling laws.}
We validate the stability predictions of Theorem~\ref{thm:stability}  with \((s,q)=(0.3,1.5)\). Figure~\ref{fig:numerical-summary}\,(left) fixes \(\beta=1-2^{-6}\) and varies \(B\), confirming \(\eta_\polyak^\crit\eqsim B\) and \(\eta_\nesterov^\crit\eqsim B^q\) before saturation, while $\eta_{\sgd}^{\crit}$ remains order one. Figure~\ref{fig:numerical-summary}\,(middle) fixes \(B=4\) and varies \(\beta\), confirming the predicted linear dependence of both momentum stability boundaries on \(1-\beta\). Appendix~\ref{app:stability-experiment} provides further experimental details and additional  results. 

Figure~\ref{fig:numerical-summary} (right) shows that the momentum scaling laws accurately capture the risk dynamics of both Polyak and Nesterov for  \((s,q)=(0.3,4)\). In particular, it confirms three key predictions: an underdamped transient with crossover timescale dominated by \((1-\beta)^{-1}\), a common power-law signal decay, and distinct asymptotic noise floors. Once the transient becomes negligible, Polyak and Nesterov follow the same power law, while Nesterov reaches a substantially lower noise floor. Appendix~\ref{app:scaling-law-experiments} provides additional details and results.

\begin{figure}[!ht]
    \centering
    \begin{subfigure}[t]{0.32\textwidth}
        \centering
        \includegraphics[width=\linewidth]{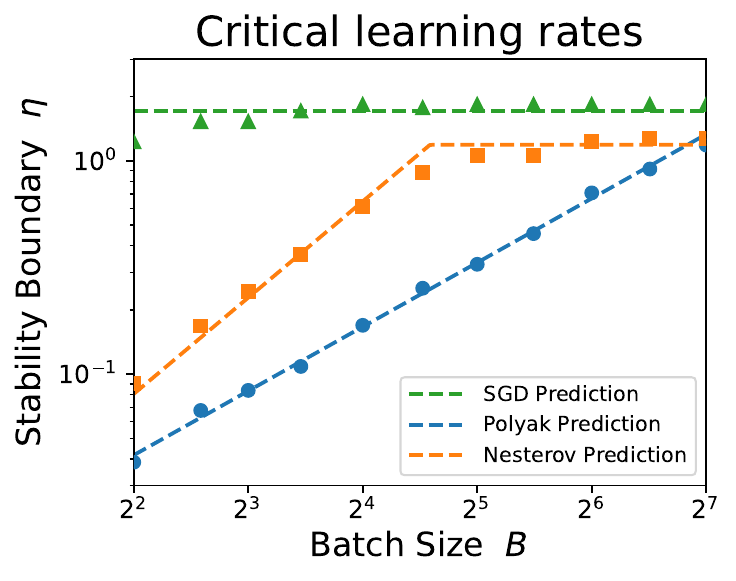}
    \end{subfigure}
    \hspace{-.2em}
    \begin{subfigure}[t]{0.32\textwidth}
        \centering
        \includegraphics[width=\linewidth]{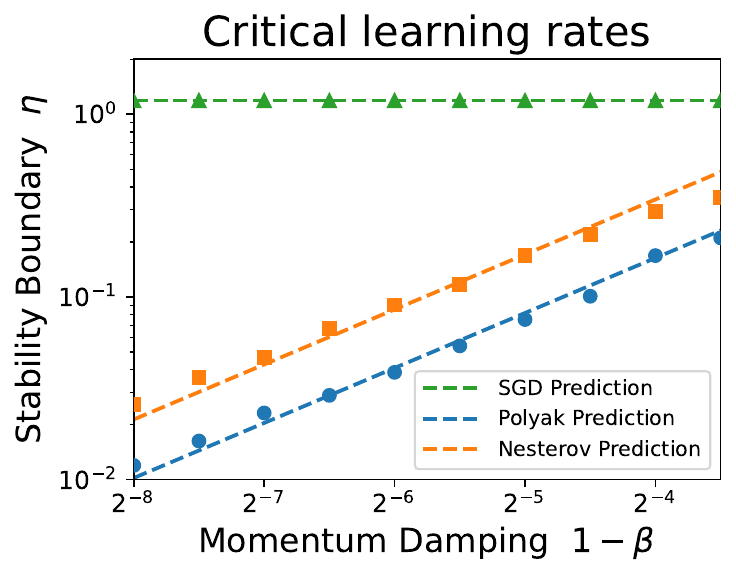}
    \end{subfigure}
    \hspace{-.2em}
    \begin{subfigure}[t]{0.32\textwidth}
        \centering
        \includegraphics[width=\linewidth]{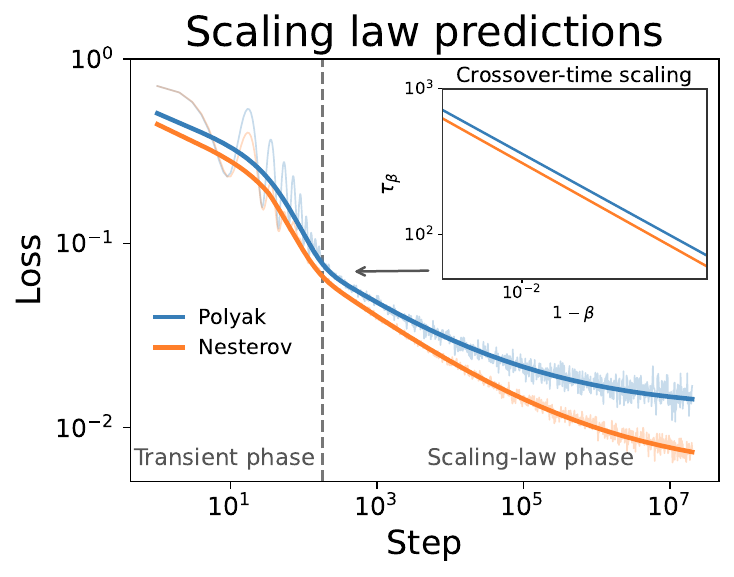}
    \end{subfigure}
    \vspace{-.7em}
 \caption{ \textbf{Validation of stability and risk-dynamics scaling laws.}
\textbf{Left:} Critical learning rate versus batch size at fixed momentum; markers denote numerical estimates and dashed lines denote theoretical predictions.
\textbf{Middle:} Critical learning rate versus \(1-\beta\) at fixed batch size.
\textbf{Right:} Polyak and Nesterov risk dynamics; thin fluctuating lines show empirical risk curves, while thick lines show theoretical predictions. The inset confirms the dominant \((1-\beta)^{-1}\) scaling of the crossover timescale $\tau_\beta$ (see Eq.~\ref{eqn: crossover}).
}
    \label{fig:numerical-summary}
    \vspace{-1em}
\end{figure}

\vspace{-.5em}
\paragraph{Batch-size phase diagram.} We first test the three-regime structure predicted by Theorem~\ref{thm:batch-size-scaling} at a finite data budget. We fix \(D=2^{14}\) and, for each batch size, tune \((\eta,\beta)\) to minimize the final excess risk. Figure~\ref{fig:intro-result} (right) exhibits a progression across the three regimes. At small batch sizes, SGD, Polyak, and Nesterov achieve nearly identical risk. As \(B\) increases, SGD deteriorates first, while Polyak maintains the small-batch performance over a wider range, reflecting its enlarged critical batch size. Nesterov goes further: in the large-batch regime, its risk continues to decrease beyond the best level attained by SGD and Polyak, before eventually increasing at ultra-large batch sizes.  Appendix~\ref{app:cbs-plkr} provides the experimental setup and additional results. 

We next directly test the data-scaling exponents predicted by Theorem~\ref{thm:batch-size-scaling}. For \((s,q)=(0.3,4)\), we choose \(B\eqsim1\), \(B\eqsim D^{0.4}\), and \(B\eqsim D^{0.8}\), representing the small-, large-, and ultra-large-batch regimes, respectively. For each data budget \(D\), we set \((\eta,\beta)\) according to the theoretically prescribed scalings. Figure~\ref{fig:three-regime-plk} shows that the empirical log--log slopes of the final excess risk closely match the predicted exponents across all three regimes. Experimental details and additional settings are provided in Appendix~\ref{app:data-scaling-experiments}.

\begin{figure}[H]
\centering
\includegraphics[width=0.32\textwidth]{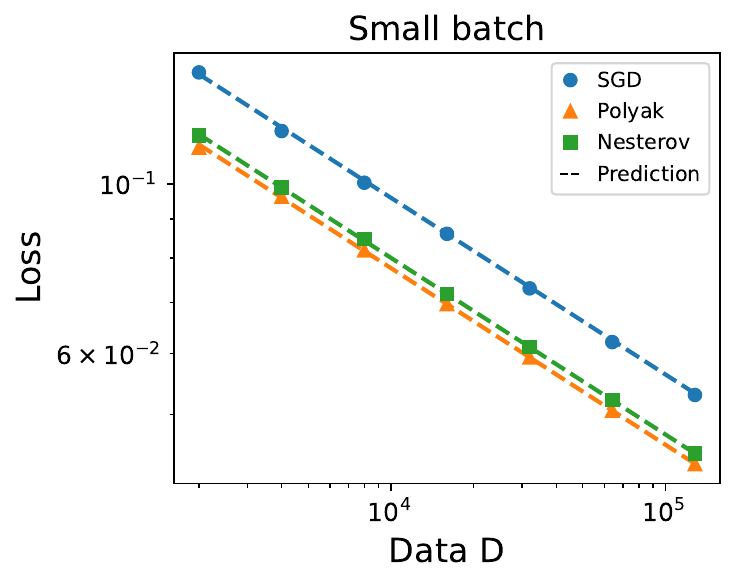}
\includegraphics[width=0.32\textwidth]{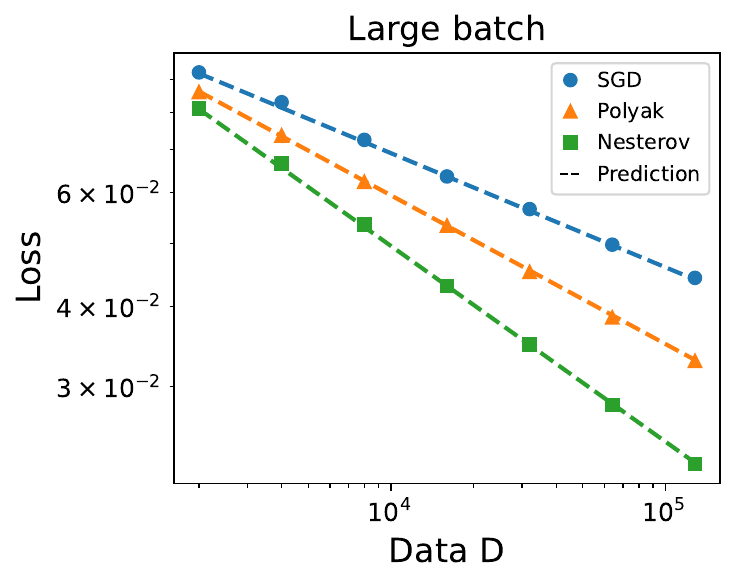}
\includegraphics[width=0.32\textwidth]{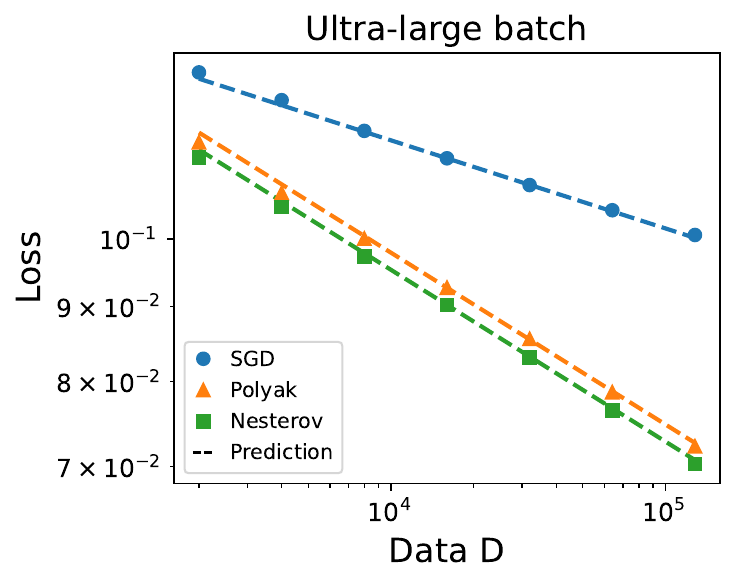}
\vspace{-.8em}
\caption{\textbf{Validation of data-scaling exponents across three batch-size regimes.} We choose representative batch-size scalings \(B\eqsim1\) \textbf{(left)}, \(B\eqsim D^{0.4}\) \textbf{(middle)}, and \(B\eqsim D^{0.8}\) \textbf{(right)}, and vary the data budget \(D\). Markers show empirical final excess risks, while dashed lines indicate the data-scaling exponents predicted by Theorem~\ref{thm:batch-size-scaling}. The empirical log--log slopes closely match the theoretical predictions across all three cases.
}

\label{fig:three-regime-plk}
\end{figure}

\begin{wrapfigure}{R}{0.4\textwidth}
\centering
\vspace{-2em}
\hspace{0em}\includegraphics[width=\linewidth]{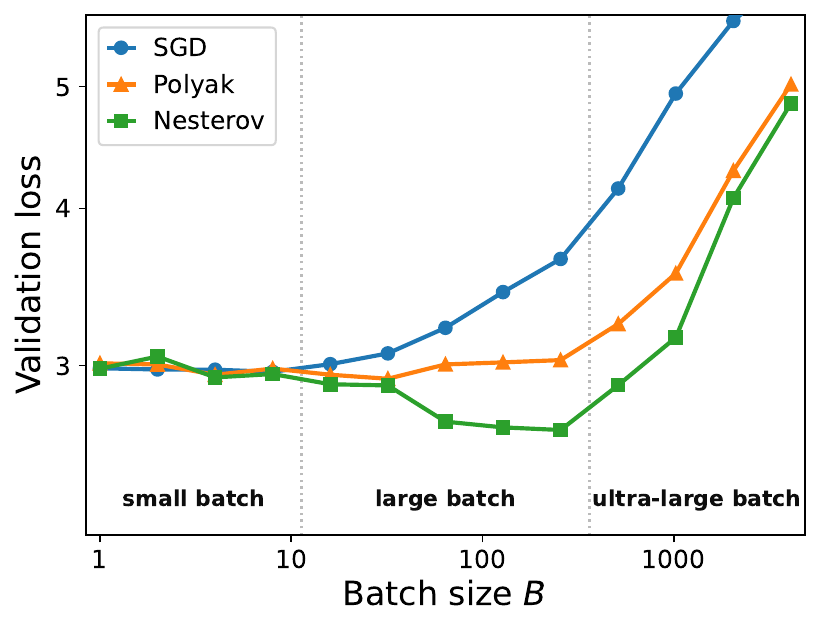}
\vspace*{-2em}
\end{wrapfigure}

\subsection{ImageNet-1K}
\label{subsec:nn}
\vspace*{-.2em}
We examine the batch-size dependence of the three optimizers when training ResNet-18 on ImageNet-1K. For each batch size, we tune the learning rate and, for Polyak and Nesterov, the momentum factor over grids, and report the best achievable risk. The results, shown in the  right figure, display the same qualitative batch-size dependence predicted by the PLK analysis. Polyak achieves a best risk comparable to that of SGD, but maintains this performance over a substantially larger batch-size range, allowing fewer sequential updates without sacrificing data efficiency. Nesterov exhibits a further advantage in the large-batch regime, where it attains lower risk than both SGD and Polyak. These results suggest that the phase structure identified in PLK regression extends qualitatively to a standard deep-learning benchmark. Additional experimental details are provided in Appendix~\ref{app:cbs-imagenet}.

\section{Conclusion and discussion}

We investigated the role of momentum in large-batch, one-pass training from two complementary perspectives: step efficiency, measured by the number of sequential updates needed to achieve a given risk, and data efficiency, measured by the risk achievable under a fixed data budget. Our analysis reveals distinct benefits of the two momentum methods. Polyak improves step efficiency by allowing substantially larger batches without sacrificing data efficiency, while Nesterov further improves data efficiency itself. These benefits depend critically on batch size and are ultimately constrained by the finite training horizon. Experiments on ImageNet-1K and CIFAR-10 provide further qualitative support for the predicted data-efficiency and stability behaviors in neural-network training.

Central to our analysis is the perspective of Functional Scaling Laws (FSL), which decomposes loss dynamics into signal learning and the accumulation and forgetting of stochastic noise. This decomposition makes it possible to disentangle how momentum affects these competing processes and to connect the resulting dynamics to the achievable step and data efficiency.

Several important directions remain open:
\begin{itemize}
    \item \textbf{Theory: Toward predictive and mechanistic FSL for momentum.} Our analysis focuses on constant learning rates and momentum. Extending FSL to momentum and learning-rate schedules is a natural next step. More fundamentally, understanding how momentum reshapes signal learning and noise forgetting could help develop a more predictive theory of optimization dynamics.

    \item \textbf{Practice: From theoretical principles to LLM pretraining.} An important challenge is to understand how the mechanisms identified here manifest in practical neural-network training, particularly large-scale LLM pretraining. Beyond testing their applicability, the ultimate goal is to translate these insights into principled strategies for jointly tuning and scheduling momentum, learning rate, and batch size~\citep{marek2026small}.
\end{itemize}

\section*{AI use statement}
\vspace*{-.5em}
We used AI tools for drafting and language editing throughout the manuscript. In Appendices B--D, they were also used to accelerate algebraic calculations and asymptotic estimates. The proof strategies and mathematical arguments were developed by the authors, who verified all AI-assisted calculations and estimates. We take responsibility for the accuracy, originality, and integrity of the final manuscript, including all claims, derivations, and conclusions.


\section*{Acknowledgement}
\vspace*{-.5em}
Lei Wu is supported in part by the National Natural Science Foundation of China (NSFC12522120, NSFC92470122, and NSFC12288101). Kairui Li is partially supported by the elite undergraduate training program of the School of Mathematical Sciences at Peking University. We thank  Sanxing Cao and Jinbo Wang for helpful discussions.

\bibliography{ref.bib}

@article{wang2023theoretical,
  title={A theoretical analysis of noise geometry in stochastic gradient descent},
  author={Wang, Mingze and Wu, Lei},
  journal={arXiv preprint arXiv:2310.00692},
  year={2023}
}

@article{polyak1964some,
  title={Some methods of speeding up the convergence of iteration methods},
  author={Polyak, Boris T},
  journal={{USSR} computational mathematics and mathematical physics},
  volume={4},
  number={5},
  pages={1--17},
  year={1964},
  publisher={Elsevier}
}

@article{marek2026small,
  title={Small batch size training for language models: When vanilla {SGD} works, and why gradient accumulation is wasteful},
  author={Marek, Martin and Lotfi, Sanae and Somasundaram, Aditya and Wilson, Andrew and Goldblum, Micah},
  journal={Advances in Neural Information Processing Systems},
  volume={38},
  pages={148837--148862},
  year={2026}
}

@book{feller1991introduction,
  title={An introduction to probability theory and its applications},
  author={Feller, William},
  volume={2},
  year={1991},
  publisher={John Wiley \& Sons}
}

@article{semenov2025benchmarking,
  title={Benchmarking optimizers for large language model pretraining},
  author={Semenov, Andrei and Pagliardini, Matteo and Jaggi, Martin},
  journal={arXiv preprint arXiv:2509.01440},
  year={2025}
}

@article{shah2025practical,
  title={Practical efficiency of {Muon} for pretraining},
  author={Shah, Ishaan and Polloreno, Anthony M and Stratos, Karl and Monk, Philip and Chaluvaraju, Adarsh and Hojel, Andrew and Ma, Andrew and Thomas, Anil and Tanwer, Ashish and Shah, Darsh J and others},
  journal={arXiv preprint arXiv:2505.02222},
  year={2025}
}

@article{morwani2026compute,
  title={Compute Efficiency and Serial Runtime Tradeoffs for Stochastic Momentum Methods},
  author={Morwani, Depen and Meterez, Alexandru and Nair, Pranav and Kakade, Sham},
  journal={arXiv preprint arXiv:2606.19179},
  year={2026}
}

@inproceedings{nesterov1983method,
  title={A method for solving the convex programming problem with convergence rate {$O(1/k^2)$}},
  author={Nesterov, Yurii},
  booktitle={Dokl {Akad Nauk SSSR}},
  volume={269},
  pages={543},
  year={1983}
}

@article{li2026functional,
  title={Functional scaling laws in kernel regression: Loss dynamics and learning rate schedules},
  author={Li, Binghui and Chen, Fengling and Huang, Zixun and Wang, Lean and Wu, Lei},
  journal={Advances in Neural Information Processing Systems},
  volume={38},
  pages={101211--101269},
  year={2025}
}

@inproceedings{gitman2019understanding,
  title={Understanding the Role of Momentum in Stochastic Gradient Methods},
  author={Gitman, Igor and Lang, Hunter and Zhang, Pengchuan and Xiao, Lin},
  booktitle={Advances in Neural Information Processing Systems},
  volume={32},
  pages={9633--9643},
  year={2019},
  eprint={1910.13962},
  archivePrefix={arXiv}
}

@inproceedings{wang2024marginal,
  title={The Marginal Value of Momentum for Small Learning Rate {SGD}},
  author={Wang, Runzhe and Malladi, Sadhika and Wang, Tianhao and Lyu, Kaifeng and Li, Zhiyuan},
  booktitle={International Conference on Learning Representations},
  year={2024},
  eprint={2307.15196},
  archivePrefix={arXiv}
}

@article{vidyasagar2025convergence,
  title={Convergence of Momentum-Based Optimization Algorithms with Time-Varying Parameters},
  author={Vidyasagar, Mathukumalli},
  journal={arXiv preprint arXiv:2506.11904},
  year={2025}
}

@inproceedings{velikanov2023view,
  title={A View of Mini-Batch {SGD} via Generating Functions: Conditions of Convergence, Phase Transitions, Benefit from Negative Momenta},
  author={Velikanov, Maksim and Kuznedelev, Denis and Yarotsky, Dmitry},
  booktitle={International Conference on Learning Representations},
  year={2023},
  eprint={2206.11124},
  archivePrefix={arXiv}
}

@article{su2016differential,
  title={A Differential Equation for Modeling {Nesterov}'s Accelerated Gradient Method: Theory and Insights},
  author={Su, Weijie and Boyd, Stephen and Cand{\`e}s, Emmanuel J.},
  journal={Journal of Machine Learning Research},
  volume={17},
  number={153},
  pages={1--43},
  year={2016}
}

@article{wibisono2016variational,
  title={A Variational Perspective on Accelerated Methods in Optimization},
  author={Wibisono, Andre and Wilson, Ashia C. and Jordan, Michael I.},
  journal={Proceedings of the National Academy of Sciences},
  volume={113},
  number={47},
  pages={E7351--E7358},
  year={2016},
  doi={10.1073/pnas.1614734113}
}

@article{shi2022understanding,
  title={Understanding the Acceleration Phenomenon via High-Resolution Differential Equations},
  author={Shi, Bin and Du, Simon S. and Jordan, Michael I. and Su, Weijie J.},
  journal={Mathematical Programming},
  volume={195},
  number={1--2},
  pages={79--148},
  year={2022},
  doi={10.1007/s10107-021-01681-8}
}

@article{wilson2021lyapunov,
  title={A {Lyapunov} Analysis of Accelerated Methods in Optimization},
  author={Wilson, Ashia C. and Recht, Ben and Jordan, Michael I.},
  journal={Journal of Machine Learning Research},
  volume={22},
  number={113},
  pages={1--34},
  year={2021}
}

@article{kovachki2021continuous,
  title={Continuous Time Analysis of Momentum Methods},
  author={Kovachki, Nikola B. and Stuart, Andrew M.},
  journal={Journal of Machine Learning Research},
  volume={22},
  number={17},
  pages={1--40},
  year={2021}
}

@article{yuan2016influence,
  title={On the Influence of Momentum Acceleration on Online Learning},
  author={Yuan, Kun and Ying, Bicheng and Sayed, Ali H.},
  journal={Journal of Machine Learning Research},
  volume={17},
  number={192},
  pages={1--66},
  year={2016}
}

@inproceedings{gupta2024nesterov,
  title={{Nesterov} Acceleration despite Very Noisy Gradients},
  author={Gupta, Kanan and Siegel, Jonathan W. and Wojtowytsch, Stephan},
  booktitle={Advances in Neural Information Processing Systems},
  volume={37},
  pages={20694--20744},
  year={2024},
  eprint={2302.05515},
  archivePrefix={arXiv}
}

@inproceedings{liu2020improved,
  title={An Improved Analysis of Stochastic Gradient Descent with Momentum},
  author={Liu, Yanli and Gao, Yuan and Yin, Wotao},
  booktitle={Advances in Neural Information Processing Systems},
  volume={33},
  year={2020},
  eprint={2007.07989},
  archivePrefix={arXiv}
}

@inproceedings{kingma2015adam,
  title={{Adam}: A Method for Stochastic Optimization},
  author={Kingma, Diederik P. and Ba, Jimmy},
  booktitle={International Conference on Learning Representations},
  year={2015},
  eprint={1412.6980},
  archivePrefix={arXiv}
}

@misc{jordan2024muon,
  title={{Muon}: An optimizer for hidden layers in neural networks},
  author={Jordan, Keller and Jin, Yuchen and Boza, Vlado and
          You, Jiacheng and Cesista, Franz and Newhouse, Laker and
          Bernstein, Jeremy},
  year={2024},
  url={https://kellerjordan.github.io/posts/muon/}
}

@article{liu2025muon,
  title={{Muon} is Scalable for {LLM} Training},
  author={Liu, Jingyuan and Su, Jianlin and Yao, Xingcheng and Jiang, Zhejun and Lai, Guokun and Du, Yulun and Qin, Yidao and Xu, Weixin and Lu, Enzhe and Yan, Junjie and Chen, Yanru and Zheng, Huabin and Liu, Yibo and Liu, Shaowei and Yin, Bohong and He, Weiran and Zhu, Han and Wang, Yuzhi and Wang, Jianzhou and Dong, Mengnan and Zhang, Zheng and Kang, Yongsheng and Zhang, Hao and Xu, Xinran and Zhang, Yutao and Wu, Yuxin and Zhou, Xinyu and Yang, Zhilin},
  journal={arXiv preprint arXiv:2502.16982},
  year={2025}
}

@inproceedings{zhu2019anisotropic,
  title={The Anisotropic Noise in Stochastic Gradient Descent: Its Behavior of Escaping from Sharp Minima and Regularization Effects},
  author={Zhu, Zhanxing and Wu, Jingfeng and Yu, Bing and Wu, Lei and Ma, Jinwen},
  booktitle={International Conference on Machine Learning},
  pages={7654--7663},
  year={2019},
  eprint={1803.00195},
  archivePrefix={arXiv}
}

@inproceedings{wang2026fast,
  title={Fast catch-up, late switching: Optimal batch size scheduling via functional scaling laws},
  author={Wang, Jinbo and Li, Binghui and Zhou, Zhanpeng and Wang, Mingze and Sun, Yuxuan and Zhang, Jiaqi and Cai, Xunliang and Wu, Lei},
  booktitle={International Conference on Learning Representations},
  volume={2026},
  pages={7310--7340},
  year={2026}
}

@article{brown2020language,
  title={Language models are few-shot learners},
  author={Brown, Tom and Mann, Benjamin and Ryder, Nick and Subbiah, Melanie and Kaplan, Jared D and Dhariwal, Prafulla and Neelakantan, Arvind and Shyam, Pranav and Sastry, Girish and Askell, Amanda and others},
  journal={Advances in Neural Information Processing Systems},
  volume={33},
  pages={1877--1901},
  year={2020}
}

@inproceedings{wu2018sgd,
  title={How {SGD} Selects the Global Minima in Over-parameterized Learning: A Dynamical Stability Perspective},
  author={Wu, Lei and Ma, Chao and E, Weinan},
  booktitle={Advances in Neural Information Processing Systems},
  volume={31},
  pages={8279--8288},
  year={2018}
}

@inproceedings{nar2018step,
  title={Step Size Matters in Deep Learning},
  author={Nar, Kamil and Sastry, S. Shankar},
  booktitle={Advances in Neural Information Processing Systems},
  volume={31},
  pages={3436--3444},
  year={2018},
  eprint={1805.08890},
  archivePrefix={arXiv}
}

@inproceedings{ma2021linear,
  title={On Linear Stability of {SGD} and Input-Smoothness of Neural Networks},
  author={Ma, Chao and Ying, Lexing},
  booktitle={Advances in Neural Information Processing Systems},
  volume={34},
  year={2021},
  eprint={2105.13462},
  archivePrefix={arXiv}
}

@inproceedings{wu2022alignment,
  title={The Alignment Property of {SGD} Noise and How It Helps Select Flat Minima: A Stability Analysis},
  author={Wu, Lei and Wang, Mingze and Su, Weijie J.},
  booktitle={Advances in Neural Information Processing Systems},
  volume={35},
  pages={4680--4693},
  year={2022},
  eprint={2207.02628},
  archivePrefix={arXiv}
}

@inproceedings{wu2023implicit,
  title={The Implicit Regularization of Dynamical Stability in Stochastic Gradient Descent},
  author={Wu, Lei and Su, Weijie J.},
  booktitle={International Conference on Machine Learning},
  series={Proceedings of Machine Learning Research},
  volume={202},
  pages={37656--37684},
  year={2023},
  eprint={2305.17490},
  archivePrefix={arXiv}
}

@article{chemnitz2025characterizing,
  title={Characterizing Dynamical Stability of Stochastic Gradient Descent in Overparameterized Learning},
  author={Chemnitz, Dennis and Engel, Maximilian},
  journal={Journal of Machine Learning Research},
  volume={26},
  number={134},
  pages={1--46},
  year={2025}
}

@inproceedings{mulayoff2021implicit,
  title={The Implicit Bias of Minima Stability: A View from Function Space},
  author={Mulayoff, Rotem and Michaeli, Tomer and Soudry, Daniel},
  booktitle={Advances in Neural Information Processing Systems},
  volume={34},
  year={2021}
}

@inproceedings{nacson2023implicit,
  title={The Implicit Bias of Minima Stability in Multivariate Shallow {ReLU} Networks},
  author={Nacson, Mor Shpigel and Mulayoff, Rotem and Ongie, Greg and Michaeli, Tomer and Soudry, Daniel},
  booktitle={International Conference on Learning Representations},
  year={2023},
  eprint={2306.17499},
  archivePrefix={arXiv}
}

@inproceedings{qiao2024stable,
  title={Stable Minima Cannot Overfit in Univariate {ReLU} Networks: Generalization by Large Step Sizes},
  author={Qiao, Dan and Zhang, Kaiqi and Singh, Esha and Soudry, Daniel and Wang, Yu-Xiang},
  booktitle={Advances in Neural Information Processing Systems},
  volume={37},
  year={2024},
  eprint={2406.06838},
  archivePrefix={arXiv}
}

@inproceedings{kaur2023maximum,
  title={On the Maximum {Hessian} Eigenvalue and Generalization},
  author={Kaur, Simran and Cohen, Jeremy and Lipton, Zachary Chase},
  booktitle={Proceedings on ``I Can't Believe It's Not Better! -- Understanding Deep Learning Through Empirical Falsification'' at NeurIPS 2022 Workshops},
  series={Proceedings of Machine Learning Research},
  volume={187},
  pages={51--65},
  year={2023},
  eprint={2206.10654},
  archivePrefix={arXiv}
}

@article{hestness2017deep,
  title={Deep Learning Scaling Is Predictable, Empirically},
  author={Hestness, Joel and Narang, Sharan and Ardalani, Newsha and Diamos, Gregory and Jun, Heewoo and Kianinejad, Hassan and Patwary, Mostofa Ali and Yang, Yang and Zhou, Yanqi},
  journal={arXiv preprint arXiv:1712.00409},
  year={2017}
}

@article{kaplan2020scaling,
  title={Scaling Laws for Neural Language Models},
  author={Kaplan, Jared and McCandlish, Sam and Henighan, Tom and Brown, Tom B. and Chess, Benjamin and Child, Rewon and Gray, Scott and Radford, Alec and Wu, Jeffrey and Amodei, Dario},
  journal={arXiv preprint arXiv:2001.08361},
  year={2020}
}

@inproceedings{hoffmann2022training,
  title={Training Compute-Optimal Large Language Models},
  author={Hoffmann, Jordan and Borgeaud, Sebastian and Mensch, Arthur and Buchatskaya, Elena and Cai, Trevor and Rutherford, Eliza and de Las Casas, Diego and Hendricks, Lisa Anne and Welbl, Johannes and Clark, Aidan and Hennigan, Tom and Noland, Eric and Millican, Katie and van den Driessche, George and Damoc, Bogdan and Guy, Aurelia and Osindero, Simon and Simonyan, Karen and Elsen, Erich and Vinyals, Oriol and Rae, Jack W. and Sifre, Laurent},
  booktitle={Advances in Neural Information Processing Systems},
  volume={35},
  pages={30016--30030},
  year={2022},
  eprint={2203.15556},
  archivePrefix={arXiv}
}

@article{grattafiori2024llama,
  title={The {LLaMA} 3 Herd of Models},
  author={Grattafiori, Aaron and Dubey, Abhimanyu and Jauhri, Abhinav and others},
  journal={arXiv preprint arXiv:2407.21783},
  year={2024}
}

@article{deepseekai2026deepseekv4,
  title={{DeepSeek-V4}: Towards Highly Efficient Million-Token Context Intelligence},
  author={{DeepSeek-AI} and Xu, Anyi and Lin, Bangcai and Xue, Bing and others},
  journal={arXiv preprint arXiv:2606.19348},
  year={2026}
}

@article{sharma2022scaling,
  title={Scaling Laws from the Data Manifold Dimension},
  author={Sharma, Utkarsh and Kaplan, Jared},
  journal={Journal of Machine Learning Research},
  volume={23},
  number={9},
  pages={1--34},
  year={2022},
  eprint={2004.10802},
  archivePrefix={arXiv}
}

@article{hutter2021learning,
  title={Learning Curve Theory},
  author={Hutter, Marcus},
  journal={arXiv preprint arXiv:2102.04074},
  year={2021}
}

@article{maloney2022solvable,
  title={A Solvable Model of Neural Scaling Laws},
  author={Maloney, Alexander and Roberts, Daniel A. and Sully, James},
  journal={arXiv preprint arXiv:2210.16859},
  year={2022}
}

@inproceedings{wei2022more,
  title={More Than a Toy: Random Matrix Models Predict How Real-World Neural Representations Generalize},
  author={Wei, Alexander and Hu, Wei and Steinhardt, Jacob},
  booktitle={International Conference on Machine Learning},
  series={Proceedings of Machine Learning Research},
  volume={162},
  pages={23549--23588},
  year={2022},
  eprint={2203.06176},
  archivePrefix={arXiv}
}

@inproceedings{jain2024scaling,
  title={Scaling Laws for Learning with Real and Surrogate Data},
  author={Jain, Ayush and Montanari, Andrea and Sasoglu, Eren},
  booktitle={Advances in Neural Information Processing Systems},
  volume={37},
  year={2024},
  eprint={2402.04376},
  archivePrefix={arXiv}
}

@inproceedings{michaud2023quantization,
  title={The Quantization Model of Neural Scaling},
  author={Michaud, Eric J. and Liu, Ziming and Girit, Uzay and Tegmark, Max},
  booktitle={Advances in Neural Information Processing Systems},
  volume={36},
  year={2023},
  eprint={2303.13506},
  archivePrefix={arXiv}
}

@inproceedings{nam2024exactly,
  title={An Exactly Solvable Model for Emergence and Scaling Laws in the Multitask Sparse Parity Problem},
  author={Nam, Yoonsoo and Fonseca, Nayara and Lee, Seok Hyeong and Louis, Ard},
  booktitle={Advances in Neural Information Processing Systems},
  volume={37},
  year={2024},
  eprint={2404.17563},
  archivePrefix={arXiv}
}

@article{atanasov2026scaling,
  title={Scaling and Renormalization in High-Dimensional Regression},
  author={Atanasov, Alexander and Zavatone-Veth, Jacob A. and Pehlevan, Cengiz},
  journal={Journal of Statistical Mechanics: Theory and Experiment},
  year={2026},
  eprint={2405.00592},
  archivePrefix={arXiv}
}

@article{bahri2024explaining,
  title={Explaining Neural Scaling Laws},
  author={Bahri, Yasaman and Dyer, Ethan and Kaplan, Jared and Lee, Jaehoon and Sharma, Utkarsh},
  journal={Proceedings of the National Academy of Sciences},
  volume={121},
  number={27},
  year={2024},
  doi={10.1073/pnas.2311878121}
}

@inproceedings{dohmatob2024tale,
  title={A Tale of Tails: Model Collapse as a Change of Scaling Laws},
  author={Dohmatob, Elvis and Feng, Yunzhen and Yang, Pu and Charton, Francois and Kempe, Julia},
  booktitle={International Conference on Machine Learning},
  year={2024},
  eprint={2402.07043},
  archivePrefix={arXiv}
}

@inproceedings{bordelon2024dynamical,
  title={A Dynamical Model of Neural Scaling Laws},
  author={Bordelon, Blake and Atanasov, Alexander and Pehlevan, Cengiz},
  booktitle={International Conference on Machine Learning},
  series={Proceedings of Machine Learning Research},
  volume={235},
  pages={4345--4382},
  year={2024},
  eprint={2402.01092},
  archivePrefix={arXiv}
}

@inproceedings{lin2024scaling,
  title={Scaling Laws in Linear Regression: Compute, Parameters, and Data},
  author={Lin, Licong and Wu, Jingfeng and Kakade, Sham M. and Bartlett, Peter L. and Lee, Jason D.},
  booktitle={Advances in Neural Information Processing Systems},
  volume={37},
  year={2024},
  eprint={2406.08466},
  archivePrefix={arXiv}
}

@article{paquette2024phases,
  title={4+3 Phases of Compute-Optimal Neural Scaling Laws},
  author={Paquette, Elliot and Paquette, Courtney and Xiao, Lechao and Pennington, Jeffrey},
  journal={arXiv preprint arXiv:2405.15074},
  year={2024}
}

@inproceedings{bordelon2025feature,
  title={How Feature Learning Can Improve Neural Scaling Laws},
  author={Bordelon, Blake and Atanasov, Alexander and Pehlevan, Cengiz},
  booktitle={International Conference on Learning Representations},
  year={2025},
  eprint={2409.17858},
  archivePrefix={arXiv}
}

@inproceedings{yan2026larger,
  title={Larger Datasets Can Be Repeated More: A Theoretical Analysis of Multi-Epoch Scaling in Linear Regression},
  author={Yan, Tingkai and Wen, Haodong and Li, Binghui and Luo, Kairong and Chen, Wenguang and Lyu, Kaifeng},
  booktitle={International Conference on Learning Representations},
  year={2026},
  eprint={2511.13421},
  archivePrefix={arXiv}
}

@inproceedings{lin2025improved,
  title={Improved Scaling Laws in Linear Regression via Data Reuse},
  author={Lin, Licong and Wu, Jingfeng and Bartlett, Peter L.},
  booktitle={Advances in Neural Information Processing Systems},
  volume={38},
  year={2025},
  eprint={2506.08415},
  archivePrefix={arXiv}
}

@article{sous2026asymmetric,
  title={Asymmetric Scaling Laws from Sparse Features},
  author={Sous, John and Winer, Michael},
  journal={arXiv preprint arXiv:2605.23591},
  year={2026}
}

@inproceedings{zhang2025critical,
  title={How Does Critical Batch Size Scale in Pre-Training?},
  author={Zhang, Hanlin and Morwani, Depen and Vyas, Nikhil and Wu, Jingfeng and Zou, Difan and Ghai, Udaya and Foster, Dean and Kakade, Sham M.},
  booktitle={International Conference on Learning Representations},
  year={2025},
  eprint={2410.21676},
  archivePrefix={arXiv}
}

@article{li2026optimal,
  title={Optimal Learning-Rate Schedules under Functional Scaling Laws: Power Decay and Warmup-Stable-Decay},
  author={Li, Binghui and Wang, Zilin and Chen, Fengling and Zhao, Shiyang and Zheng, Ruiheng and Wu, Lei},
  journal={arXiv preprint arXiv:2602.06797},
  year={2026}
}

@article{ferbach2026dimension,
  title={Dimension-adapted momentum outscales {SGD}},
  author={Ferbach, Damien and Everett, Katie and Gidel, Gauthier and Paquette, Elliot and Paquette, Courtney},
  journal={Advances in Neural Information Processing Systems},
  volume={38},
  pages={112780--112977},
  year={2026}
}

@article{liu2019acceleratingsgdmomentumoverparameterized,
  title={Accelerating {SGD} with momentum for over-parameterized learning},
  author={Liu, Chaoyue and Belkin, Mikhail},
  journal={arXiv preprint arXiv:1810.13395},
  year={2018}
}

@article{ding2025scalinglawstochasticgradient,
      title={Scaling law for stochastic gradient descent in quadratically parameterized linear regression},
      author={Ding, Shihong and Zhang, Haihan and Zhao, Hanzhen and Fang, Cong},
      journal={arXiv preprint arXiv:2502.09106},
      year={2025}
}

@article{zhang2024optimalityacceleratedsgdhighdimensional,
  title={The optimality of (accelerated) {SGD} for high-dimensional quadratic optimization},
  author={Zhang, Haihan and Liu, Yuanshi and Chen, Qianwen and Fang, Cong},
  journal={arXiv preprint arXiv:2409.09745},
  year={2024}
}

@article{caponnetto2005fast,
  title={Fast rates for regularized least-squares algorithm},
  author={Caponnetto, Andrea and De Vito, Ernesto},
  year={2005}
}

@article{caponnetto2007optimal,
  title={Optimal rates for the regularized least-squares algorithm},
  author={Caponnetto, Andrea and De Vito, Ernesto},
  journal={Foundations of Computational Mathematics},
  volume={7},
  pages={331--368},
  year={2007},
  publisher={Springer}
}

@article{spigler2020asymptotic,
  title={Asymptotic learning curves of kernel methods: empirical data versus teacher{--}student paradigm},
  author={Spigler, Stefano and Geiger, Mario and Wyart, Matthieu},
  journal={Journal of Statistical Mechanics: Theory and Experiment},
  volume={2020},
  number={12},
  pages={124001},
  year={2020},
  publisher={IOP Publishing}
}

@inproceedings{bordelon2020spectrum,
  title={Spectrum dependent learning curves in kernel regression and wide neural networks},
  author={Bordelon, Blake and Canatar, Abdulkadir and Pehlevan, Cengiz},
  booktitle={International Conference on Machine Learning},
  pages={1024--1034},
  year={2020},
  organization={PMLR}
}
\bibliographystyle{bibstyle}

\newpage

\appendix
\part{Appendix}
\parttoc

\newpage
\section{Technical preliminaries}
\subsection{Kernel representation and rotational equivariance}
\label{app:subsec:rotation invariant}

We explain the kernel interpretation of the PLK model and justify the spectral representation used throughout the analysis.

\paragraph{Kernel representation.}

Given a positive semidefinite kernel \(K:\cX\times\cX\to\RR\), kernel regression learns over its associated reproducing kernel Hilbert space (RKHS) \(\cH_K\).  Given an input distribution $\cD$, under mild conditions, Mercer's theorem ensures that the kernel admits the following eigendecomposition
\[
K(\bx,\bx')
=
\sum_{j=1}^\infty \lambda_j e_j(\bx) e_j(\bx'),
\]
where \(\{e_j\}_{j\geq1}\) are orthonormal in \(L^2(\cD)\). The standard capacity condition assumes \(\lambda_j\eqsim j^{-q}\). Target regularity is commonly described through the interpolation spaces
\[
\cH_K^r
=
\left\{
\sum_{j=1}^\infty a_j e_j:
\sum_{j=1}^\infty \frac{a_j^2}{\lambda_j^r}<\infty
\right\},
\]
with a source condition of the form \(f^*\in\cH_K^r\).

Our formulation is the corresponding feature-map representation. Under Assumption~\ref{ass:diagonal-covariance}, define \(e_j:=\phi_j/\lambda_j^{1/2}\). Then \(\{e_j\}_{j\geq1}\) is orthonormal in \(L^2(\cD)\), and
\[
K_\phi(\bx,\bx')
=
\sum_{j=1}^\infty \phi_j(\bx)\phi_j(\bx')
=
\sum_{j=1}^\infty \lambda_j e_j(\bx)e_j(\bx').
\]
Thus, Assumption~\ref{ass:power-law}, \(\lambda_j\eqsim j^{-q}\), is precisely the standard capacity condition.

For the target function,
\[
f^*
=
\sum_{j=1}^\infty \theta_j^*\phi_j
=
\sum_{j=1}^\infty a_j^*e_j,
\qquad
a_j^*:=\lambda_j^{1/2}\theta_j^*.
\]
Our task-difficulty assumption gives \(|a_j^*|^2\eqsim j^{-1}\lambda_j^s\). Hence, for any \(\delta\in(0,s)\),
\[
\sum_{j=1}^\infty
\frac{|a_j^*|^2}{\lambda_j^{s-\delta}}
\eqsim
\sum_{j=1}^\infty j^{-1}\lambda_j^\delta
\eqsim
\sum_{j=1}^\infty j^{-1-q\delta}
<\infty.
\]
Therefore, \(f^*\in\cH_{K_\phi}^{s-\delta}\) for every \(\delta\in(0,s)\). At the endpoint, \(\sum_j|a_j^*|^2/\lambda_j^s\eqsim\sum_jj^{-1}=\infty\), so \(f^*\notin\cH_{K_\phi}^{s}\). Thus, \(q\) is the capacity exponent, while \(s\) marks the boundary of the target source regularity.

When \(0<s\le1\), the relation \(|\theta_j^*|^2\eqsim j^{-1-q(s-1)}\) shows that the target sequence is not square-summable and hence does not represent an element of \(\mathcal H\). All subsequent expressions involving \(\btheta^*\) are understood coordinatewise in the eigenbasis of \(\bH\), with prediction series interpreted in \(L^2(\cD)\), which is well defined since \(\sum_j\lambda_j|\theta_j^*|^2<\infty\).

\paragraph{Rotational equivariance.}

We now verify that SGD, Polyak, and Nesterov are equivariant under rotations. Directions in \(\ker\bH\) have zero feature variance and do not affect predictions, so we restrict attention to its orthogonal complement. Since \(\bH=\EE[\bphi(\bx)\otimes\bphi(\bx)]\) is positive, self-adjoint, and trace class, it admits a spectral decomposition
\[
\bH=\bm U\bm\Lambda\bm U^*,
\qquad
\bm\Lambda=\diag(\lambda_1,\lambda_2,\ldots),
\]
for some unitary operator \(\bm U\). Define the rotated variables
\[
\btheta'=\bm U^*\btheta,\qquad
\bphi'(\bx)=\bm U^*\bphi(\bx).
\]
Unitary invariance of the inner product gives
\[
f'(\bx;\btheta')=\langle\bphi'(\bx),\btheta'\rangle
=\langle\bphi(\bx),\btheta\rangle=f(\bx;\btheta),
\]
and hence, for every mini-batch \(\cB_k\),
\[
\bg'(\btheta';\cB_k)=\bm U^*\bg(\btheta;\cB_k).
\]
It follows immediately that the SGD and Polyak recursions are preserved under \(\btheta_k'=\bm U^*\btheta_k\). For Nesterov, the look-ahead point transforms as
\[
\btheta_k^{\la\prime}
=\bm U^*\btheta_k^{\la}
=\btheta_k'+\beta(\btheta_k'-\btheta_{k-1}'),
\]
so its recursion is preserved as well. Finally,
\[
\bH'=\EE[\bphi'(\bx)\otimes\bphi'(\bx)]
=\bm U^*\bH\bm U
=\bm\Lambda.
\]
Thus, SGD, Polyak, and Nesterov are equivariant under unitary changes of basis, while their predictions and risks are unchanged. We may therefore work in the eigenbasis of \(\bH\), where the covariance operator is diagonal, without loss of generality.

The target is the same function \(f^*\) in both representations. If \(\btheta^*\in\mathcal H\), its coefficients transform as \(\btheta^{*\prime}=\bm U^*\btheta^*\). Otherwise, use the finite spectral truncations \(f_N^*=\sum_{j=1}^N\theta_j^*\phi_j\), which converge to \(f^*\) in \(L^2(\cD)\). The prediction and risk identities then pass to the limit, without applying \(\bm U^*\) to a vector outside \(\mathcal H\).

\subsection{Noise-curvature alignment}
\label{app:subsec:gaussian-fourth-moment}
\label{app:hessian alignment}

We derive the covariance identities and noise--curvature alignment property needed for the risk stability and scaling law analyses.

\paragraph{Gaussian fourth moments.}
The following lemma expresses Gaussian fourth moments in terms of the covariance operator.

\begin{lemma}[Gaussian fourth-moment identities]
\label{app:lem:gaussian-fourth-moment}
Let \(\bx\sim\mathcal N(\bm0,\bH)\) in a separable Hilbert space \(\mathcal H\), where \(\bH\) is a positive trace-class covariance operator. For any \(\bu,\bv\in\mathcal H\),
\begin{equation}
\label{app:eq:gaussian-fourth-moment-directional}
\EE\!\left[\langle\bx,\bu\rangle^2\langle\bx,\bv\rangle^2\right]
=
\langle\bu,\bH\bu\rangle\langle\bv,\bH\bv\rangle
+
2\langle\bu,\bH\bv\rangle^2.
\end{equation}
Equivalently, for any self-adjoint trace-class operator \(\bA\),
\begin{equation}
\label{app:eq:gaussian-fourth-moment}
\EE\!\left[(\bx\otimes\bx)\bA(\bx\otimes\bx)\right]
=
2\bH\bA\bH+\bH\Tr(\bH\bA).
\end{equation}
\end{lemma}

\begin{proof}
The directional identity follows directly from Isserlis' formula. For any \(\bu,\bv\in\mathcal H\), the random variables \(\langle\bx,\bu\rangle\) and \(\langle\bx,\bv\rangle\) are jointly Gaussian, and hence
\[
\EE\!\left[\langle\bx,\bu\rangle^2\langle\bx,\bv\rangle^2\right]
=
\langle\bu,\bH\bu\rangle\langle\bv,\bH\bv\rangle
+
2\langle\bu,\bH\bv\rangle^2.
\]

We next prove the operator identity. Let \(\bA=\sum_j a_j\be_j\otimes\be_j\) be the spectral decomposition of the self-adjoint trace-class operator \(\bA\). For arbitrary \(\bu,\bv\in\mathcal H\), Isserlis' formula gives
\begin{align*}
&\left\langle
\bu,
\EE\!\left[(\bx\otimes\bx)\bA(\bx\otimes\bx)\right]
\bv
\right\rangle\\
&\qquad=
\sum_j a_j
\EE\!\left[
\langle\bu,\bx\rangle
\langle\be_j,\bx\rangle^2
\langle\bx,\bv\rangle
\right]\\
&\qquad=
2\sum_j a_j
\langle\bu,\bH\be_j\rangle
\langle\be_j,\bH\bv\rangle
+
\langle\bu,\bH\bv\rangle
\sum_j a_j\langle\be_j,\bH\be_j\rangle\\
&\qquad=
2\langle\bu,\bH\bA\bH\bv\rangle
+
\langle\bu,\bH\bv\rangle\Tr(\bH\bA).
\end{align*}
Since this holds for all \(\bu,\bv\in\mathcal H\),
\[
\EE\!\left[(\bx\otimes\bx)\bA(\bx\otimes\bx)\right]
=
2\bH\bA\bH+\bH\Tr(\bH\bA).
\]
\end{proof}

\begin{corollary}[Empirical covariance fluctuation]
\label{app:cor:gaussian-empirical-covariance}
Let
\[
\widehat{\bH}:=\frac1B\sum_{b=1}^B\bx_b\otimes\bx_b,
\qquad
\bSigma:=\widehat{\bH}-\bH,
\qquad
\bx_b\overset{\mathrm{i.i.d.}}{\sim}\mathcal N(\bm0,\bH).
\]
Then, for any self-adjoint trace-class operator \(\bA\),
\begin{equation}
\label{app:eq:gaussian-empirical-covariance}
\EE[\bSigma\bA\bSigma]
=
\frac1B\left(\bH\bA\bH+\bH\Tr(\bH\bA)\right).
\end{equation}
\end{corollary}

\begin{proof}
Write \(\bSigma=B^{-1}\sum_{b=1}^B\bZ_b\), where \(\bZ_b:=\bx_b\otimes\bx_b-\bH\) are independent and centered. Hence
\[
\EE[\bSigma\bA\bSigma]
=
\frac1B\EE[\bZ\bA\bZ]
=
\frac1B\left(\EE[(\bx\otimes\bx)\bA(\bx\otimes\bx)]-\bH\bA\bH\right),
\]
and the result follows from Lemma~\ref{app:lem:gaussian-fourth-moment}.
\end{proof}

\begin{lemma}[Noise--curvature alignment]
\label{app:lem:noise-geometry}
It holds for any $\bnu$ that
$$
\mathbb E\!\left[
\left|
\left\langle\bxi_k(\btheta_k),\bnu\right\rangle
\right|^2
\,\middle|\,
\btheta_k
\right]
\eqsim
\frac{1}{B}
\bigl(\cE(\btheta_k)+\sigma^2\bigr)
\left\langle\bnu,\bH\bnu\right\rangle.
$$
\end{lemma}


\begin{proof}
We now apply these identities to the mini-batch gradient noise and prove Lemma~\ref{app:lem:noise-geometry}. For an evaluation point \(\btheta\) independent of the current mini-batch of size \(B\), write $ \bxi_k(\btheta)=\frac1B\sum_{i=1}^B\bzeta_{k,i}(\btheta), $ where \(\bzeta_{k,i}(\btheta)\) are conditionally independent, centered single-sample gradient noises. Hence, for any unit vector \(\bnu\in\mathcal H\),
\[
\EE\!\left[\left.|\langle\bxi_k(\btheta),\bnu\rangle|^2\right|\btheta\right]
=
\frac1B\EE\!\left[\left.|\langle\bzeta(\btheta),\bnu\rangle|^2\right|\btheta\right].
\]

Let \(\bdelta:=\btheta-\btheta^*\) and \(\bphi:=\bphi(\bx)\). Since \(y=\langle\bphi,\btheta^*\rangle+\epsilon\), the single-sample gradient is $ \bg_{\bx,y}(\btheta) = \bigl(\langle\bphi,\bdelta\rangle-\epsilon\bigr)\bphi, $ while \(\nabla\cR(\btheta)=\bH\bdelta\). Therefore,
\[
\bzeta(\btheta)
=
\langle\bphi,\bdelta\rangle\bphi-\bH\bdelta-\epsilon\bphi.
\]
Using the independence of \(\epsilon\) and \(\bphi\), \(\EE[\epsilon]=0\), and \(\EE[\bphi\otimes\bphi]=\bH\), we obtain
\[
\EE\!\left[\left.|\langle\bzeta(\btheta),\bnu\rangle|^2\right|\btheta\right]
=
\EE\!\left[\langle\bphi,\bdelta\rangle^2\langle\bphi,\bnu\rangle^2\right]
-\langle\bdelta,\bH\bnu\rangle^2
+\sigma^2\langle\bnu,\bH\bnu\rangle.
\]
If \(\bdelta\notin\mathcal H\), apply Lemma~\ref{app:lem:gaussian-fourth-moment} first to its finite spectral truncations. Since \(\sum_j\lambda_j\delta_j^2=2\cE(\btheta)<\infty\), the corresponding Gaussian prediction errors converge in \(L^4\), so their fourth-moment identities pass to the limit. Thus the lemma gives
\[
\EE\!\left[\langle\bphi,\bdelta\rangle^2\langle\bphi,\bnu\rangle^2\right]
=
\langle\bdelta,\bH\bdelta\rangle\langle\bnu,\bH\bnu\rangle
+
2\langle\bdelta,\bH\bnu\rangle^2.
\]
Thus,
\begin{equation}
\label{app:eq:noise-direction-exact}
\EE\!\left[\left.|\langle\bzeta(\btheta),\bnu\rangle|^2\right|\btheta\right]
=
\bigl(\sigma^2+\langle\bdelta,\bH\bdelta\rangle\bigr)\langle\bnu,\bH\bnu\rangle
+
\langle\bdelta,\bH\bnu\rangle^2.
\end{equation}
The last term is nonnegative and, by Cauchy--Schwarz in the \(\bH\)-inner product,
\[
\langle\bdelta,\bH\bnu\rangle^2
\leq
\langle\bdelta,\bH\bdelta\rangle
\langle\bnu,\bH\bnu\rangle.
\]
Since \(\cE(\btheta)=\frac12\langle\bdelta,\bH\bdelta\rangle\), Eq.~\eqref{app:eq:noise-direction-exact} implies
\[
\EE\!\left[\left.|\langle\bzeta(\btheta),\bnu\rangle|^2\right|\btheta\right]
\eqsim
\bigl(\cE(\btheta)+\sigma^2\bigr)\langle\bnu,\bH\bnu\rangle.
\]
Combining this with the \(1/B\) variance reduction from mini-batching and setting \(\btheta=\btheta_k\) yields
\[
\EE\!\left[\left.|\langle\bxi_k(\btheta_k),\bnu\rangle|^2\right|\btheta_k\right]
\eqsim
\frac1B\bigl(\cE(\btheta_k)+\sigma^2\bigr)\langle\bnu,\bH\bnu\rangle.
\]
\end{proof}

\subsection{Second-order recursions}
\label{app:second-order-toolkit}

We collect solution formulas and bounds for second-order recursions. These tools separate the evolution of the initial error from the response to gradient noise in the scaling law analysis.

Consider the homogeneous recursion
\begin{equation}
    x_{k+1}+a x_k+b x_{k-1}=0,
    \qquad k\ge0,
    \label{app:eq:toolkit-homogeneous-recursion}
\end{equation}
with prescribed initial values $x_{-1},x_0$. Its characteristic polynomial is
\begin{equation}
    r^2+a r+b=0.
    \label{app:eq:toolkit-characteristic}
\end{equation}

\begin{lemma}[Second-order iteration]
\label{app:lem:second-order-iteration}
Let $r_+$ and $r_-$ be the roots of~\eqref{app:eq:toolkit-characteristic}. If $r_+\neq r_-$, then the unique solution of~\eqref{app:eq:toolkit-homogeneous-recursion} with initial values $x_{-1},x_0$ is
\begin{equation}
    x_k=c_+ r_+^k+c_-r_-^k,
    \qquad k\ge0,
    \label{app:eq:toolkit-homogeneous-solution}
\end{equation}
where
\begin{equation}
    c_+=\frac{r_+(x_0-r_-x_{-1})}{r_+-r_-},
    \qquad
    c_- =\frac{r_-(r_+x_{-1}-x_0)}{r_+-r_-}.
    \label{app:eq:toolkit-coefficients}
\end{equation}
If $r_+=r_-=:r$, then
\begin{equation}
    x_k=(c_1+c_2k)r^k,
    \label{app:eq:toolkit-repeated-root-solution}
\end{equation}
where $c_1=x_0$ and $c_2=x_0-rx_{-1}$ when $r\neq0$. If the repeated root is \(r=0\), then \(a=b=0\) and \(x_k=0\) for \(k\ge1\), with \(x_{-1},x_0\) prescribed.
\end{lemma}

\begin{proof}
For distinct roots, substituting \eqref{app:eq:toolkit-coefficients} into \eqref{app:eq:toolkit-homogeneous-solution} gives the prescribed value at \(k=0\) and satisfies the recursion at \(k=0\) with the prescribed \(x_{-1}\). Each root then generates a solution for \(k\ge1\). The repeated-root formula follows by taking the limit as the two roots approach the same value. When \(r=0\), the recursion directly gives \(x_k=0\) for \(k\ge1\). In each case, uniqueness follows by induction.
\end{proof}

For the normalized initial values \(x_{-1}=x_0=1\), the same formula reads
\begin{equation}
    x_k=r_+^k+r_-(1-r_+)D_k(r_+,r_-),\qquad
    D_k(r_+,r_-):=\sum_{\ell=0}^{k-1}r_+^{k-1-\ell}r_-^\ell,
    \label{app:eq:divided-difference-solution}
\end{equation}
where \(D_0=0\), with the convention \(0^0=1\). For distinct roots, \(D_k=(r_+^k-r_-^k)/(r_+-r_-)\), while \(D_k(r,r)=kr^{k-1}\) for \(k\ge1\). In particular, \(|D_k|\le kR^{k-1}\) for \(k\ge1\) if \(|r_\pm|\le R\) and \(R>0\). This bound remains valid when the two roots coincide, so it controls modes near the damping threshold without dividing by a small root separation.

\begin{lemma}[Forced second-order iteration]
\label{app:lem:forced-second-order-iteration}
Consider the forced recursion
\begin{equation}
    \tilde x_{k+1}+a\tilde x_k+b\tilde x_{k-1}+z_k=0,
    \qquad k\ge0,
    \label{app:eq:toolkit-forced-recursion}
\end{equation}
with arbitrary initial values $\tilde x_{-1},\tilde x_0$. Let $x_k^{\mathrm{det}}$ be the solution of the homogeneous recursion
\begin{equation}
    x_{k+1}^{\mathrm{det}}+a x_k^{\mathrm{det}}+b x_{k-1}^{\mathrm{det}}=0,
    \qquad
    x_{-1}^{\mathrm{det}}=\tilde x_{-1},
    \qquad
    x_0^{\mathrm{det}}=\tilde x_0.
    \label{app:eq:toolkit-homogeneous-forced}
\end{equation}
Let $G_k$ be the Green function associated with~\eqref{app:eq:toolkit-homogeneous-recursion}, defined by
\begin{equation}
    G_{k+1}+aG_k+bG_{k-1}=0,
    \qquad
    G_{-1}=0,
    \qquad
    G_0=1.
    \label{app:eq:toolkit-green-recursion}
\end{equation}
Then the unique solution of~\eqref{app:eq:toolkit-forced-recursion} is
\begin{equation}
    \tilde x_k
    =x_k^{\mathrm{det}}
    -\sum_{m=0}^{k-1}G_{k-1-m}z_m,
    \qquad k\ge0.
    \label{app:eq:toolkit-forced-solution}
\end{equation}
If $r_+\neq r_-$, then
\begin{equation}
    G_k=\frac{r_+^{k+1}-r_-^{k+1}}{r_+-r_-},
    \qquad k\ge0,
    \label{app:eq:toolkit-green-closed-form}
\end{equation}
and at a repeated root $r_+=r_-=r$ this is interpreted by continuity as $G_k=(k+1)r^k$.
\end{lemma}

\begin{proof}
Set \(w_k=-\sum_{m=0}^{k-1}G_{k-1-m}z_m\), with \(w_{-1}=w_0=0\). Using the recursion for \(G\), all terms cancel in \(w_{k+1}+aw_k+bw_{k-1}\) except \(-G_0z_k=-z_k\). Thus \(x^{\mathrm{det}}+w\) has the prescribed initial values and satisfies \eqref{app:eq:toolkit-forced-recursion}. Uniqueness follows by induction. The initial values \(G_{-1}=0\), \(G_0=1\) and Lemma~\ref{app:lem:second-order-iteration} give \eqref{app:eq:toolkit-green-closed-form}. Taking the limit as the two roots approach the same value gives \(G_k=(k+1)r^k\), with \(G_0=1\) also when \(r=0\).
\end{proof}

The next lemma gives the total energy of the Green function and a bound on its tail. We use these results to estimate the forgetting kernel sums in the label noise proofs and to control multiplicative noise.

\begin{lemma}[Energy and tail bounds for the Green function]
\label{lem:generic-label-noise-energy}
Consider the recursion
\[
    G_{k+1}+aG_k+bG_{k-1}=0,\qquad G_{-1}=0,\quad G_0=1.
\]
Let \(r_\pm\) be the roots of \(r^2+ar+b=0\), and assume \(|r_\pm|<1\). Then the following statements hold.
\begin{enumerate}
\item The total energy satisfies
\begin{equation}
    \sum_{k=0}^{\infty}G_k^2
    =\frac{1+b}{(1-b)\bigl((1+b)^2-a^2\bigr)}.
    \label{eq:generic-green-energy}
\end{equation}
\item For any \(z>1\) satisfying \(\sqrt z\,|r_\pm|<1\), the weighted energy satisfies
\begin{equation}
    \sum_{\tau=0}^{\infty}z^\tau G_\tau^2
    =\frac{1+bz}{(1-bz)\bigl((1+bz)^2-a^2z\bigr)},
    \label{eq:generic-green-weighted-energy}
\end{equation}
\item For the same \(z\), the tail satisfies
\begin{equation}
    \sum_{\tau=k}^{\infty}G_\tau^2
    \le z^{-k}\sum_{\tau=0}^{\infty}z^\tau G_\tau^2.
    \label{eq:generic-green-energy-tail}
\end{equation}
\end{enumerate}
\end{lemma}

\begin{proof}
Let \(Q=\sum_{k\ge0}G_k^2\) and \(R=\sum_{k\ge0}G_kG_{k-1}\). The root condition ensures convergence. Multiplying the recursion by \(G_k\) and summing gives \((1+b)R=-aQ\). Squaring it and summing gives \(Q-1=(a^2+b^2)Q+2abR\). Eliminating \(R\) yields
\[
    1=Q\,\frac{1-b}{1+b}\bigl((1+b)^2-a^2\bigr),
\]
which is \eqref{eq:generic-green-energy}.
For the weighted identity, set \(\widehat G_k:=z^{k/2}G_k\). This sequence satisfies the recursion with coefficients \(a\sqrt z\) and \(bz\), with \(\widehat G_{-1}=0\) and \(\widehat G_0=1\). Its characteristic roots are \(\sqrt z\,r_\pm\), whose moduli are less than one by assumption. Applying \eqref{eq:generic-green-energy} to this sequence gives \eqref{eq:generic-green-weighted-energy}. Finally, \(z^\tau\ge z^k\) for \(\tau\ge k\), which proves \eqref{eq:generic-green-energy-tail}.
\end{proof}

We next bound the overdamped coordinate decay for Polyak and Nesterov. These estimates give the signal learning bounds in Appendices~\ref{app:hbm scaling law} and~\ref{app:nesterov scaling law}.

\begin{lemma}[Overdamped coordinate decay for Polyak]
\label{lem:pol-overdamped-coordinate}
Let \(0\le\beta<1\) and \(x>0\). Consider the recursion
\[
    h_{k+1}=(1+\beta-x)h_k-\beta h_{k-1},\qquad h_{-1}=h_0=1.
\]
If \(x\le c_0(1-\beta)^2\) for a sufficiently small absolute constant \(c_0>0\), then
\begin{equation}
    c e^{-Cxk/(1-\beta)}\le h_k\le C e^{-cxk/(1-\beta)},\qquad k\ge0,
    \label{eq:pol-overdamped-coordinate-decay}
\end{equation}
for absolute constants \(c,C>0\).
\end{lemma}

\begin{proof}
The characteristic polynomial is \(p(r)=r^2-(1+\beta-x)r+\beta\). To estimate the distance of its roots from \(1\), set \(r=1-u\). Then
\[
    p(1-u)=u^2-\gamma u+x,\qquad \gamma=1-\beta+x.
\]
Thus \(1-\beta\le\gamma\le1-\beta+2x\). Since \(\gamma\eqsim1-\beta\) and \(\gamma^2-4x\ge(1-4c_0)(1-\beta)^2>0\), the roots \(u_+\le u_-\) are real and positive. Set \(r_+=1-u_+\) and \(r_-=1-u_-\). Then
\[
    u_+=\frac{2x}{\gamma+\sqrt{\gamma^2-4x}}\eqsim\frac{x}{1-\beta},
    \qquad r_+-r_-=\sqrt{\gamma^2-4x}\eqsim1-\beta.
\]
Since \(u_+\lesssim c_0\), we have \(r_+>0\). The identity \(r_+r_-=\beta\ge0\) gives \(0\le r_-\le r_+<1\). By \eqref{app:eq:divided-difference-solution},
\[
    h_k=r_+^k+d(r_+^k-r_-^k),\qquad
    0\le d:=\frac{r_-(1-r_+)}{r_+-r_-}\lesssim\frac{x}{(1-\beta)^2}\le c_0.
\]
Thus \(r_+^k\le h_k\le(1+Cc_0)r_+^k\). Since \(-\log r_+\eqsim1-r_+\eqsim x/(1-\beta)\), this proves the stated bounds.
\end{proof}

\begin{lemma}[Overdamped coordinate decay for Nesterov]
\label{lem:nes-overdamped-coordinate}
Let \(0\le\beta<1\) and \(x>0\). Consider the recursion
\[
    h_{k+1}=(1+\beta)(1-x)h_k-\beta(1-x)h_{k-1},\qquad h_{-1}=h_0=1.
\]
If \(x\le c_0(1-\beta)^2\) for a sufficiently small absolute constant \(c_0>0\), then
\begin{equation}
    c e^{-Cxk/(1-\beta)}\le h_k\le C e^{-cxk/(1-\beta)},\qquad k\ge0,
    \label{eq:nes-overdamped-coordinate-decay}
\end{equation}
for absolute constants \(c,C>0\).
\end{lemma}

\begin{proof}
The characteristic polynomial is \(p(r)=r^2-(1+\beta)(1-x)r+\beta(1-x)\). To estimate the distance of its roots from \(1\), set \(r=1-u\). Then
\[
    p(1-u)=u^2-\gamma u+x,\qquad \gamma=1-\beta+(1+\beta)x.
\]
Thus \(1-\beta\le\gamma\le1-\beta+2x\). Since \(\gamma\eqsim1-\beta\) and \(\gamma^2-4x\ge(1-4c_0)(1-\beta)^2>0\), the roots \(u_+\le u_-\) are real and positive. Set \(r_+=1-u_+\) and \(r_-=1-u_-\). Then
\[
    u_+=\frac{2x}{\gamma+\sqrt{\gamma^2-4x}}\eqsim\frac{x}{1-\beta},
    \qquad r_+-r_-=\sqrt{\gamma^2-4x}\eqsim1-\beta.
\]
Since \(u_+\lesssim c_0\), we have \(r_+>0\). Since \(x\le c_0(1-\beta)^2<1\), the identity \(r_+r_-=\beta(1-x)\ge0\) gives \(0\le r_-\le r_+<1\). By \eqref{app:eq:divided-difference-solution},
\[
    h_k=r_+^k+d(r_+^k-r_-^k),\qquad
    0\le d:=\frac{r_-(1-r_+)}{r_+-r_-}\lesssim\frac{x}{(1-\beta)^2}\le c_0.
\]
Thus \(r_+^k\le h_k\le(1+Cc_0)r_+^k\). Since \(-\log r_+\eqsim1-r_+\eqsim x/(1-\beta)\), this proves the stated bounds.
\end{proof}

\subsection{Power-law spectral sums}
\label{app:preparatory-fsl-estimates}

Coordinatewise decay bounds reduce signal learning to sums of exponentials over the spectrum. The next lemma covers both the full spectrum and the overdamped tail, so the same estimate applies to both momentum methods.

\begin{lemma}[Power-law spectral sums]
\label{lem:power-law-spectral-tail}
Let \(q,s,c>0\). For every integer \(J\ge1\) and \(t\ge0\),
\begin{equation}
\label{eq:power-law-spectral-tail}
    \sum_{j\ge J}j^{-1-q s}e^{-ctj^{-q}}
    \eqsim
    (J^q+t)^{-s}.
\end{equation}
The comparison constants depend only on \(q,s,c\). In particular, the full sum is of order \((1+t)^{-s}\), and the sum at \(t=0\) is of order \(J^{-qs}\).
\end{lemma}

\begin{proof}
Write \(F_{J}(t)\) for the sum in~\eqref{eq:power-law-spectral-tail}. First suppose \(0\le t\le J^q\). The integral test gives
\[
    F_J(t)
    \le \sum_{j\ge J}j^{-1-q s}
    \le J^{-1-q s}+\int_J^\infty x^{-1-q s}\,dx
    \lesssim J^{-q s}.
\]
For \(J\le j\le2J\), we have \(tj^{-q}\le1\). Therefore,
\[
    F_J(t)
    \ge e^{-c}\sum_{j=J}^{2J}j^{-1-q s}
    \ge e^{-c}(J+1)(2J)^{-1-q s}
    \gtrsim J^{-q s}.
\]
Because \(J^q\le J^q+t\le2J^q\), these bounds give~\eqref{eq:power-law-spectral-tail} in this case.

Now suppose \(t>J^q\) and put \(L:=\lceil t^{1/q}\rceil\). Then \(J\le L\le2t^{1/q}\). For \(L\le j\le2L\), \(tj^{-q}\le1\), so
\[
    F_J(t)
    \ge e^{-c}\sum_{j=L}^{2L}j^{-1-q s}
    \ge e^{-c}(L+1)(2L)^{-1-q s}
    \gtrsim t^{-s}.
\]
For the upper bound, set \(f_t(x):=x^{-1-q s}e^{-ct x^{-q}}\). Its derivative has the sign of \(cqtx^{-q}-(1+q s)\), so \(f_t\) increases and then decreases on \((0,\infty)\). Compare each term before the maximum with the integral on the interval to its right, and each term after the maximum with the integral on the interval to its left. At most two terms remain, each bounded by the maximum. Thus
\[
    F_J(t)
    \le \sum_{j\ge1}f_t(j)
    \le \int_1^\infty f_t(x)\,dx+2\sup_{x\ge1}f_t(x).
\]
With \(u=ctx^{-q}\),
\[
    \int_1^\infty f_t(x)\,dx
    =\frac{(ct)^{-s}}{q}\int_0^{ct}u^{s-1}e^{-u}\,du
    \le\frac{\Gamma(s)}{q(ct)^s}.
\]
The maximum over \(x>0\) is \(O(t^{-s-1/q})\), which is \(O(t^{-s})\) because \(t>J^q\ge1\). Thus \(F_J(t)\lesssim t^{-s}\). Since \(t<J^q+t<2t\), the claimed comparison follows.
\end{proof}

For a truncated sum, nonnegativity and the full-sum case \(J=1\) give
\begin{equation}
    \sum_{j=1}^{J}j^{-1-qs}e^{-ctj^{-q}}
    \le\sum_{j\ge1}j^{-1-qs}e^{-ctj^{-q}}
    \lesssim(1+t)^{-s},\qquad t\ge0.
    \label{eq:power-law-spectral-head}
\end{equation}
This is the bound used for the Nesterov underdamped signal learning term in Proposition~\ref{thm:nes-underdamped-transient-unified}. No restriction on \(t\) is needed for this upper bound. A matching lower bound would require \(t\lesssim J^q\): for fixed \(J\), the truncated sum decays exponentially once \(t/J^q\) is large, whereas the full sum retains its power-law tail.

\subsection{A convolution inequality}
\label{app:convolution-inequalities}

The following convolution bound is used to control the multiplicative noise in Appendices~\ref{app:hbm scaling law} and~\ref{app:nesterov scaling law}. The index shift accommodates the different gradient evaluation points of the two methods.

\begin{lemma}[Boundedness of a convolution inequality]
\label{lem:convolution-bound}
Let \((x_k)\), \((u_k)\), \((y_k)\), and \((w_k)\) be nonnegative sequences, with \(U:=\sup_k u_k<\infty\), and let \(v\ge0\). Suppose
\[
    x_k\le u_k+\sum_{\tau=0}^{k-1}w_\tau\bigl(v+y_{k-1-\tau}\bigr),\qquad k\ge0,
\]
where the sum is empty for \(k=0\). Suppose there exist a constant \(L>0\) and a fixed index shift \(\iota\in\{0,1\}\) such that \(y_m\le Lx_{m+\iota}\) for all \(m\ge0\). If \(L\sum_{\tau\ge0}w_\tau\le1/2\), then
\[
    \sup_{k\ge0}x_k\le 2U+\frac{v}{L}.
\]

\end{lemma}
\begin{proof}
Set \(W:=\sum_{\tau\ge0}w_\tau\) and \(X_n:=\max_{0\le k\le n}x_k\). Since \(k-1-\tau+\iota\le n\) for \(1\le k\le n\), we have \(X_n\le U+vW+LWX_n\). Thus \(X_n\le2(U+vW)\le2U+v/L\), which proves the claim.
\end{proof}

\section{Risk stability analysis}
\label{app:sec:stability}

\subsection{Renewal formulation of risk stability}
\label{app:subsec:label-noise-stability}

We first recall a general boundedness criterion for renewal equations~\citep{feller1991introduction}. Consider a nonnegative sequence \((x_k)_{k\geq0}\) satisfying
\[
x_k=a_k+\sum_{j=0}^{k-1}\cK_{k-1-j}x_j,
\]
where \(a_k\geq0\) is the input and \(\cK=(\cK_k)_{k\geq0}\) is the renewal kernel. Its total mass \(\|\cK\|_1:=\sum_{k\geq0}\cK_k\) measures the overall feedback strength of the recursion.

\begin{lemma}[Boundedness of a renewal equation]
\label{app:lem:renewal-stability}
Suppose \(0<\underline a:=\inf_{k\geq0}a_k\leq\sup_{k\geq0}a_k=:\bar a<\infty\). Then
\[
\sup_{k\geq0}x_k<\infty\quad\Longleftrightarrow\quad \|\cK\|_1<1.
\]
\end{lemma}

\begin{proof}
Suppose first that \(\|\cK\|_1<1\), and let \(M_k:=\max_{0\leq j\leq k}x_j\). Then \(x_k\leq\bar a+\|\cK\|_1M_{k-1}\), which implies
\[
\sup_{k\geq0}x_k\leq\max\left\{x_0,\frac{\bar a}{1-\|\cK\|_1}\right\}<\infty.
\]
Conversely, suppose that \((x_k)\) is bounded and let \(L:=\liminf_{k\to\infty}x_k\). For every fixed \(N\), nonnegativity gives
\[
x_k\geq\underline a+\sum_{m=0}^{N}\cK_mx_{k-1-m}.
\]
Taking \(\liminf_{k\to\infty}\) yields \(L\geq\underline a+L\sum_{m=0}^{N}\cK_m\). Letting \(N\to\infty\), boundedness is impossible if \(\sum_{m\geq0}\cK_m\geq1\).
\end{proof}

For each method \(\cA\in\{\sgd,\polyak,\nesterov\}\), we will derive a nonnegative risk moment \(\ell_k^{(\cA)}\), equal to \(2\EE[\cE(\btheta_k)]\) for SGD and Polyak and comparable to it for Nesterov, satisfying
\begin{equation}
\label{app:eq:generic-risk-renewal}
\ell_k^{(\cA)}=h_k^{(\cA)}+\sum_{j=0}^{k-1}\cK_{k-1-j}^{(\cA)}\bigl(\ell_j^{(\cA)}+\sigma^2\bigr).
\end{equation}
Setting \(x_k=\ell_k^{(\cA)}+\sigma^2\) and \(a_k=h_k^{(\cA)}+\sigma^2\) reduces~\eqref{app:eq:generic-risk-renewal} to the renewal equation above. Hence, once the coordinate-wise dynamics are stable and \(h_k^{(\cA)}\) is bounded, Lemma~\ref{app:lem:renewal-stability} gives the global feedback condition
\[
\kappa_{\cA}(\eta,\beta,B):=\|\cK^{(\cA)}\|_1=\sum_{k\geq0}\cK_k^{(\cA)}<1.
\]
The kernel \(\cK^{(\cA)}\) describes how perturbations propagate through the optimizer dynamics and depends on the method and hyperparameters \((\eta,\beta,B)\), but not on the label-noise variance \(\sigma^2\). Additive label noise enters as the constant external forcing \(\sigma^2\), whereas multiplicative noise feeds the current risk back through \(\cK^{(\cA)}\). In Appendices~\ref{app:subsec:sgd-stability}--\ref{app:subsec:nesterov-stability}, we derive \(\cK^{(\cA)}\), compute \(\kappa_{\cA}\), and combine \(\kappa_{\cA}<1\) with the corresponding coordinate-wise stability condition to obtain the critical learning rates.
\subsection{Risk stability of SGD}
\label{app:subsec:sgd-stability}

The SGD error recursion, including label noise, is
\begin{equation}
\label{app:eq:sgd-stability-recursion}
\be_{k+1}
=
(\bI-\eta\widehat{\bH}_k)\be_k
+
\eta\widehat{\boldsymbol\zeta}_k.
\end{equation}

\begin{theorem}[Risk stability of SGD]
\label{app:thm:sgd-stability}
Let the Hessian be a positive trace-class operator with spectrum
\[
\bH=\diag(\lambda_1,\lambda_2,\ldots),
\qquad
\lambda_1=\lambda_{\max}\geq\lambda_2\geq\cdots>0,
\]
and suppose \(\sigma>0.\) Then SGD is risk stable if and only if
\begin{equation}
\label{app:eq:sgd-local-stability}
\eta<\frac{2B}{(B+1)\lambda_{\max}}
\end{equation}
and
\begin{equation}
\label{app:eq:sgd-global-stability}
\kappa_{\sgd}(\eta)
:=
\sum_{i\geq1}
\frac{\eta\lambda_i}{2B-(B+1)\eta\lambda_i}
<1.
\end{equation}
\end{theorem}

\begin{proof}
Let \(\{\boldsymbol{\varphi}_i\}_{i\geq1}\) denote the orthonormal eigenbasis of \(\bH\). We work with its spectral coordinates. For every finite-risk initialization, the relevant risk moment is
\[
\ell_k
:=
\sum_{i\geq1}\lambda_i
\EE\bigl[|\langle\boldsymbol{\varphi}_i,\be_k\rangle|^2\bigr]
=
2\EE[\cE(\btheta_k)].
\]
In particular, \(\ell_0<\infty\). The coordinate identities below may be justified first on finite spectral truncations and then passed to the limit by monotone convergence. Expanding \eqref{app:eq:sgd-stability-recursion}, using the vanishing cross terms, and applying Lemma~\ref{app:lem:gaussian-fourth-moment} gives, for every \(i\geq1\),
\begin{equation}
\label{app:eq:sgd-covariance}
\EE\bigl[|\langle\boldsymbol{\varphi}_i,\be_{k+1}\rangle|^2\bigr]
=
(1-\eta\lambda_i)^2
\EE\bigl[|\langle\boldsymbol{\varphi}_i,\be_k\rangle|^2\bigr]
+
\frac{\eta^2}{B}
\left(
\lambda_i^2
\EE\bigl[|\langle\boldsymbol{\varphi}_i,\be_k\rangle|^2\bigr]
+
\lambda_i\ell_k
+
\sigma^2\lambda_i
\right).
\end{equation}
Define the coordinate variance by
\[
u_{i,k}:=\EE\bigl[|\langle\boldsymbol{\varphi}_i,\be_k\rangle|^2\bigr].
\]
With this notation, \eqref{app:eq:sgd-covariance} yields
\begin{equation}
\label{app:eq:sgd-diagonal-recursion}
u_{i,k+1}
=
a_i u_{i,k}
+
\frac{\eta^2}{B}\lambda_i\bigl(\ell_k+\sigma^2\bigr),
\qquad
a_i:=(1-\eta\lambda_i)^2+\frac{\eta^2}{B}\lambda_i^2.
\end{equation}

We first determine the local condition. Since
\[
1-a_i
=
\eta\lambda_i
\left[
2-\left(1+\frac1B\right)\eta\lambda_i
\right],
\]
the coordinatewise contraction condition
\[
a_i<1
\qquad\text{for every }i
\]
is equivalent to \eqref{app:eq:sgd-local-stability}. If this condition fails, then \eqref{app:eq:sgd-diagonal-recursion} contains an eigendirection satisfying
\[
u_{i,k+1}
\geq
u_{i,k}+\frac{\eta^2\sigma^2}{B}\lambda_i,
\]
and therefore the risk is unbounded. Hence the local condition is necessary.

Assume henceforth that \eqref{app:eq:sgd-local-stability} holds. Iterating \eqref{app:eq:sgd-diagonal-recursion}, multiplying the \(i\)-th equation by \(\lambda_i\), and summing over all eigendirections gives
\begin{equation}
\label{app:eq:sgd-renewal}
\ell_k
=
h_k
+
\sum_{s=0}^{k-1}
\cK_{k-1-s}\bigl(\ell_s+\sigma^2\bigr),
\end{equation}
where
\[
h_k:=\sum_{i\geq1}\lambda_i a_i^k u_{i,0},
\qquad
\cK_m:=\frac{\eta^2}{B}\sum_{i\geq1}\lambda_i^2a_i^m.
\]
The local condition gives
\[
0\leq a_i<1,
\qquad
0\leq h_k\leq \ell_0.
\]
Moreover,
\[
\sum_{m\geq0}\cK_m
=
\frac{\eta^2}{B}
\sum_{i\geq1}\frac{\lambda_i^2}{1-a_i}
=
\sum_{i\geq1}
\frac{\eta\lambda_i}{2B-(B+1)\eta\lambda_i}
=
\kappa_{\sgd}(\eta).
\]
Lemma~\ref{app:lem:renewal-stability} therefore gives
\[
\sup_{k\geq0}\ell_k<\infty
\quad\Longleftrightarrow\quad
\kappa_{\sgd}(\eta)<1,
\]
which completes the proof.

\end{proof}

\begin{corollary}[Critical learning rate of SGD]
\label{app:cor:sgd-critical-rate}
There exists a unique critical learning rate satisfying
\[
0<\eta_{\sgd}^{\crit}<\frac{2B}{(B+1)\lambda_{\max}}
\]
and
\begin{equation}
\label{app:eq:sgd-critical-equation}
\sum_{i\geq1}
\frac{
\eta_{\sgd}^{\crit}\lambda_i
}{
2B-(B+1)\eta_{\sgd}^{\crit}\lambda_i
}
=1.
\end{equation}
The SGD iteration is risk stable if and only if
\[
0<\eta<\eta_{\sgd}^{\crit}.
\]
Moreover,
\begin{equation}
\label{app:eq:sgd-critical-bounds}
\frac{2B}{(B+1)\lambda_{\max}+\Tr(\bH)}
\leq
\eta_{\sgd}^{\crit}
\leq
\min\left\{
\frac{2B}{(B+1)\lambda_{\max}},
\frac{2B}{\Tr(\bH)}
\right\}.
\end{equation}
Consequently, under the fixed-exponent power-law assumptions \( \lambda_i\eqsim i^{-q}, q>1 \), we have
\[
\eta_{\sgd}^{\crit}\eqsim1.
\]
\end{corollary}

\begin{proof}
The renewal mass \(\kappa_{\sgd}(\eta)\) is continuous and strictly increasing on the local stability interval. It vanishes at \(\eta=0\) and diverges as \(\eta\) approaches the local boundary, because the denominator of the leading eigendirection tends to zero. Hence there is a unique solution of \(\kappa_{\sgd}(\eta)=1\), which gives the stated critical learning rate and stability interval.

At the critical point, every denominator in \eqref{app:eq:sgd-critical-equation} is at most \(2B\). Therefore \(1\geq\frac{\eta_{\sgd}^{\crit}}{2B}\Tr(\bH),\) which yields
\[
\eta_{\sgd}^{\crit}
\leq
\frac{2B}{\Tr(\bH)}.
\]
The local condition gives the other upper bound in \eqref{app:eq:sgd-critical-bounds}. Conversely, every denominator is at least
\[
2B-(B+1)\eta_{\sgd}^{\crit}\lambda_{\max}.
\]
Using the critical equation then gives
\[
1
\leq
\frac{
\eta_{\sgd}^{\crit}\Tr(\bH)
}{
2B-(B+1)\eta_{\sgd}^{\crit}\lambda_{\max}
},
\]
and rearranging yields the lower bound in \eqref{app:eq:sgd-critical-bounds}.

Finally, under \(\lambda_i\eqsim i^{-q}\) with \(q>1\), we have \(\lambda_{\max}\eqsim1\) and \(\Tr(\bH)\eqsim1\). The two-sided bounds therefore imply \(\eta_{\sgd}^{\crit}\eqsim1\).
\end{proof}

\subsection{Risk stability of Polyak}
\label{app:subsec:polyak-stability}

The Polyak error recursion, including label noise, is
\begin{equation}
\label{app:eq:polyak-stability-recursion}
\be_{k+1}
=
(\bI-\eta\widehat{\bH}_k)\be_k
+
\beta(\be_k-\be_{k-1})
+
\eta\widehat{\boldsymbol\zeta}_k.
\end{equation}
Writing \(\bD=(1+\beta)\bI-\eta\bH,\) Equation \eqref{app:eq:polyak-stability-recursion} becomes
\[
\be_{k+1}
=
\bD\be_k
-
\beta\be_{k-1}
-
\eta(\widehat{\bH}_k-\bH)\be_k
+
\eta\widehat{\boldsymbol\zeta}_k.
\]

\begin{theorem}[Risk stability of Polyak]
\label{app:thm:polyak-stability}
Let the Hessian be a positive trace-class operator and assume
\[
\bH=\diag(\lambda_1,\lambda_2,\ldots),
\qquad
\lambda_1=\lambda_{\max}\geq\lambda_2\geq\cdots>0,
\qquad
0\leq\beta<1,
\qquad
\sigma>0.
\]
Then Polyak momentum is risk stable if and only if
\begin{equation}
\label{app:eq:polyak-local-stability}
\eta
<
\frac{
2B(1-\beta^2)
}{
\lambda_{\max}\left[B(1-\beta)+(1+\beta)\right]
}
\end{equation}
and
\begin{equation}
\label{app:eq:polyak-global-stability}
\kappa_{\polyak}(\eta,\beta)
:=
\sum_{i\geq1}
\frac{
\eta\lambda_i(1+\beta)
}{
2B(1-\beta^2)
-
\eta\lambda_i\left[B(1-\beta)+(1+\beta)\right]
}
<1.
\end{equation}
\end{theorem}

\begin{proof}
Define the coordinate second moments
\[
\begin{aligned}
\bU_k&:=\EE[\be_k\otimes\be_k],
&
\bP_k&:=\EE[\be_k\otimes\be_{k-1}],
\\
\bQ_k&:=\EE[\be_{k-1}\otimes\be_k],
&
\bR_k&:=\EE[\be_{k-1}\otimes\be_{k-1}].
\end{aligned}
\]
The risk moment relevant to risk stability is
\[
\ell_k
:=
\Tr(\bH\bU_k)
=
2\EE[\cE(\btheta_k)].
\]
Expanding \eqref{app:eq:polyak-stability-recursion}, using the vanishing cross terms, and applying Lemma~\ref{app:lem:gaussian-fourth-moment} gives
\begin{equation}
\label{app:eq:polyak-covariance-recursion}
\left\{
\begin{aligned}
\bU_{k+1}
={}&
\bD\bU_k\bD
-
\beta\bD\bP_k
-
\beta\bQ_k\bD
+
\beta^2\bR_k
\\
&+
\frac{\eta^2}{B}
\left(
\bH\bU_k\bH
+
\bH\Tr(\bH\bU_k)
+
\sigma^2\bH
\right),
\\
\bP_{k+1}
={}&
\bD\bU_k-\beta\bQ_k,
\\
\bQ_{k+1}
={}&
\bU_k\bD-\beta\bP_k,
\\
\bR_{k+1}
={}&
\bU_k.
\end{aligned}
\right.
\end{equation}
In the eigenbasis of the Hessian, define
\[
u_{i,k}:=\langle\boldsymbol{\varphi}_i,\bU_k\boldsymbol{\varphi}_i\rangle,
\qquad
v_{i,k}:=\langle\boldsymbol{\varphi}_i,\bP_k\boldsymbol{\varphi}_i\rangle,
\qquad
r_{i,k}:=\langle\boldsymbol{\varphi}_i,\bR_k\boldsymbol{\varphi}_i\rangle.
\]
The diagonal entries of the two cross-covariance operators coincide because \( \bQ_k=\bP_k^*. \) Set
\[
d_i:=1+\beta-\eta\lambda_i,
\qquad
\bX_{i,k}:=(u_{i,k},v_{i,k},r_{i,k})^\top,
\qquad
\bb:=(1,0,0)^\top,
\]
and
\[
\bM_i
:=
\begin{pmatrix}
d_i^2+\dfrac{\eta^2}{B}\lambda_i^2
&
-2\beta d_i
&
\beta^2
\\
d_i
&
-\beta
&
0
\\
1
&
0
&
0
\end{pmatrix}.
\]
Taking the diagonal entries of \eqref{app:eq:polyak-covariance-recursion} yields
\begin{equation}
\label{app:eq:polyak-block-recursion}
\bX_{i,k+1}
=
\bM_i\bX_{i,k}
+
\frac{\eta^2}{B}\lambda_i\bb\bigl(\ell_k+\sigma^2\bigr).
\end{equation}

After removing the global risk-coupling term, the characteristic polynomial of the isolated block is
\[
p_i(\mu)
=
\mu^3
+
\left(
\beta-d_i^2-\frac{\eta^2}{B}\lambda_i^2
\right)\mu^2
+
\beta
\left(
d_i^2-\frac{\eta^2}{B}\lambda_i^2-\beta
\right)\mu
-
\beta^3.
\]
The Jury stability criterion shows that all roots lie strictly inside the unit disk if and only if
\[
2B(1-\beta^2)
-
\eta\lambda_i\left[B(1-\beta)+(1+\beta)\right]
>0.
\]
Requiring this inequality in every eigendirection gives \eqref{app:eq:polyak-local-stability}. If the local condition fails, the isolated second-moment system has a mode on or outside the unit circle. Because the label noise forcing enters its current-error coordinate directly and has strictly positive variance, the corresponding coordinate risk is unbounded. Thus, the local condition is necessary.

Assume henceforth that the local condition holds, and define \(g_{i,m}:=\bb^\top\bM_i^m\bb.\) This quantity is nonnegative. Iterating \eqref{app:eq:polyak-block-recursion}, taking the first component, multiplying by the corresponding eigenvalue, and summing over all eigendirections gives
\begin{equation}
\label{app:eq:polyak-renewal}
\ell_k
=
h_k
+
\sum_{s=0}^{k-1}
\cK_{k-1-s}\bigl(\ell_s+\sigma^2\bigr),
\end{equation}
where
\[
h_k
:=
\sum_{i\geq1}
\lambda_i\bb^\top\bM_i^k\bX_{i,0},
\qquad
\cK_m
:=
\frac{\eta^2}{B}
\sum_{i\geq1}\lambda_i^2g_{i,m}.
\]
The strict local condition gives the uniform power bound \(\sup_{i,k}\|\bM_i^k\|<\infty.\) Positivity of the initial coordinate covariance also gives \(|v_{i,0}|\leq\frac{u_{i,0}+r_{i,0}}{2}.\) Together, these estimates imply
\[
0\leq h_k
\leq
C\sum_{i\geq1}\lambda_i(u_{i,0}+r_{i,0})
<\infty.
\]
A direct calculation gives
\begin{equation}
\label{app:eq:polyak-resolvent}
\bb^\top(\bI-\bM_i)^{-1}
=
\frac{
B\bigl(1+\beta,-2\beta d_i,\beta^2(1+\beta)\bigr)
}{
\eta\lambda_i
\left\{
2B(1-\beta^2)
-
\eta\lambda_i\left[B(1-\beta)+(1+\beta)\right]
\right\}
}.
\end{equation}
Consequently,
\[
\sum_{m\geq0}\cK_m
=
\frac{\eta^2}{B}
\sum_{i\geq1}
\lambda_i^2\bb^\top(\bI-\bM_i)^{-1}\bb
=
\kappa_{\polyak}(\eta,\beta).
\]
Lemma~\ref{app:lem:renewal-stability} therefore gives
\[
\sup_{k\geq0}\ell_k<\infty
\quad\Longleftrightarrow\quad
\kappa_{\polyak}(\eta,\beta)<1,
\]
which completes the proof.

\end{proof}

\begin{corollary}[Critical learning rate of Polyak]
\label{app:cor:polyak-critical-rate}
There exists a unique critical learning rate satisfying
\[
0
<
\eta_{\polyak}^{\crit}
<
\frac{
2B(1-\beta^2)
}{
\lambda_{\max}\left[B(1-\beta)+(1+\beta)\right]
}
\]
and the critical equation
\begin{equation}
\label{app:eq:polyak-critical-equation}
\sum_{i\geq1}
\frac{
\eta_{\polyak}^{\crit}\lambda_i(1+\beta)
}{
2B(1-\beta^2)
-
\eta_{\polyak}^{\crit}\lambda_i
\left[B(1-\beta)+(1+\beta)\right]
}
=1.
\end{equation}
The Polyak iteration is risk stable if and only if
\[
0<\eta<\eta_{\polyak}^{\crit}.
\]
Moreover,
\begin{equation}
\label{app:eq:polyak-critical-bounds}
\begin{aligned}
\frac{
2B(1-\beta^2)
}{
\lambda_{\max}\left[B(1-\beta)+(1+\beta)\right]
+
(1+\beta)\Tr(\bH)
}
\leq
\eta_{\polyak}^{\crit},
\\
\eta_{\polyak}^{\crit}
\leq
\min\left\{
\frac{
2B(1-\beta^2)
}{
\lambda_{\max}\left[B(1-\beta)+(1+\beta)\right]
},
\frac{2B(1-\beta)}{\Tr(\bH)}
\right\}.
\end{aligned}
\end{equation}
Under the power-law assumption \(\lambda_i\eqsim i^{-q},q>1 \), when \(B\gg1,1-\beta\ll1,\) we have
\[
\eta_{\polyak}^{\crit}
\eqsim
\min\{1,B(1-\beta)\}.
\]
\end{corollary}

\begin{proof}
The renewal mass \(\kappa_{\polyak}(\eta,\beta)\) is continuous and strictly increasing throughout the local stability interval. It vanishes at the origin and diverges at the local boundary. Theorem~\ref{app:thm:polyak-stability} therefore gives a unique critical learning rate and the stated stability interval.

At the critical point, every denominator in \eqref{app:eq:polyak-critical-equation} is at most \(2B(1-\beta^2)\) and hence
\[
\eta_{\polyak}^{\crit}
\leq
\frac{2B(1-\beta)}{\Tr(\bH)},
\]
while the local condition gives the other upper bound in \eqref{app:eq:polyak-critical-bounds}. Conversely, every denominator is at least \( 2B(1-\beta^2) - \eta_{\polyak}^{\crit}\lambda_{\max} \left[B(1-\beta)+(1+\beta)\right]. \) Substitution into the critical equation gives
\[
1
\leq
\frac{
\eta_{\polyak}^{\crit}(1+\beta)\Tr(\bH)
}{
2B(1-\beta^2)
-
\eta_{\polyak}^{\crit}\lambda_{\max}
\left[B(1-\beta)+(1+\beta)\right]
}.
\]
Rearranging this inequality gives the lower bound in \eqref{app:eq:polyak-critical-bounds}.

Finally, the power-law spectrum and the high-momentum regime give
\[
\lambda_{\max}\eqsim1,
\qquad
\Tr(\bH)\eqsim1,
\qquad
1-\beta^2\eqsim1-\beta.
\]
Consequently, both sides of \eqref{app:eq:polyak-critical-bounds} scale as
\[
\frac{B(1-\beta)}{B(1-\beta)+1}
\eqsim
\min\{1,B(1-\beta)\}.
\]
\end{proof}

\subsection{Risk stability of Nesterov}
\label{app:subsec:nesterov-stability}

Define the look-ahead error
\[
\be_k^{\la}
:=
(1+\beta)\be_k-\beta\be_{k-1}.
\]
The Nesterov error recursion, including label noise, is
\begin{equation}
\label{app:eq:nesterov-stability-recursion}
\be_{k+1}
=
(\bI-\eta\widehat{\bH}_k)\be_k^{\la}
+
\eta\widehat{\boldsymbol\zeta}_k.
\end{equation}
With the deterministic transition \(\bD=\bI-\eta\bH,\) the recursion becomes
\[
\be_{k+1}
=
\bD\be_k^{\la}
-
\eta(\widehat{\bH}_k-\bH)\be_k^{\la}
+
\eta\widehat{\boldsymbol\zeta}_k.
\]

\begin{theorem}[Risk stability of Nesterov]
\label{app:thm:nesterov-stability}
Let the Hessian be a positive trace-class operator and assume
\[
\bH=\diag(\lambda_1,\lambda_2,\ldots),
\qquad
\lambda_1=\lambda_{\max}\geq\lambda_2\geq\cdots>0,
\qquad
0\leq\beta<1,
\qquad
\sigma>0.
\]
Define
\[
\Delta_{\nesterov}(x)
:=
2B(1-\beta^2)
+
\left[4B\beta^2+(B-1)\beta-(B+1)\right]x
-
(B+1)\beta(1+2\beta)x^2.
\]
Then Nesterov momentum is risk stable if and only if
\begin{equation}
\label{app:eq:nesterov-local-stability}
\Delta_{\nesterov}(\eta\lambda_{\max})>0
\end{equation}
and
\begin{equation}
\label{app:eq:nesterov-global-stability}
\kappa_{\nesterov}(\eta,\beta)
:=
\sum_{i\geq1}
\frac{
\eta\lambda_i
\left[1+\beta+\beta(1+2\beta)\eta\lambda_i\right]
}{
\Delta_{\nesterov}(\eta\lambda_i)
}
<1.
\end{equation}
\end{theorem}

\begin{proof}
Define
\[
\bU_k:=\EE[\be_k\otimes\be_k],
\qquad
\bV_k:=\EE[\be_k^{\la}\otimes\be_k^{\la}],
\qquad
\bP_k:=\EE[\be_k\otimes\be_{k-1}],
\]
\[
\bQ_k:=\EE[\be_{k-1}\otimes\be_k],
\qquad
\bR_k:=\EE[\be_{k-1}\otimes\be_{k-1}].
\]
Expanding \eqref{app:eq:nesterov-stability-recursion}, using the vanishing cross terms, and applying Lemma~\ref{app:lem:gaussian-fourth-moment} gives
\begin{equation}
\label{app:eq:nesterov-covariance-recursion}
\left\{
\begin{aligned}
\bU_{k+1}
={}&
\bD\bV_k\bD
+
\frac{\eta^2}{B}
\left(
\bH\bV_k\bH
+
\bH\Tr(\bH\bV_k)
+
\sigma^2\bH
\right),
\\
\bV_k
={}&
(1+\beta)^2\bU_k
-
\beta(1+\beta)(\bP_k+\bQ_k)
+
\beta^2\bR_k,
\\
\bP_{k+1}
={}&
\bD\bigl((1+\beta)\bU_k-\beta\bQ_k\bigr),
\\
\bQ_{k+1}
={}&
\bigl((1+\beta)\bU_k-\beta\bP_k\bigr)\bD,
\\
\bR_{k+1}
={}&
\bU_k.
\end{aligned}
\right.
\end{equation}

The renewal equation below naturally controls the look-ahead loss moment
\[
\ell_k^{\la}
:=
\Tr(\bH\bV_k)
=
\EE\|\be_k^{\la}\|_{\bH}^2.
\]
This is equivalent to controlling the risk at the actual iterate. Define its loss moment by
\[
\ell_k
:=
\Tr(\bH\bU_k)
=
2\EE[\cE(\btheta_k)].
\]
Then
\begin{equation}
\label{app:eq:nesterov-risk-equivalence-forward}
\ell_k^{\la}
\leq
2(1+\beta)^2\ell_k+2\beta^2\ell_{k-1}.
\end{equation}
Conversely, since \(\bD\) commutes with \(\bH\) and \(\bH\) is trace class, the first equation in \eqref{app:eq:nesterov-covariance-recursion} gives
\begin{equation}
\label{app:eq:nesterov-risk-equivalence-backward}
\ell_{k+1}
\leq
C_{\eta,\beta,\bH,B}
\bigl(\ell_k^{\la}+\sigma^2\bigr).
\end{equation}
Therefore
\[
\sup_k\ell_k^{\la}<\infty
\quad\Longleftrightarrow\quad
\sup_k\EE[\cE(\btheta_k)]<\infty.
\]

Let
\[
u_{i,k}:=\langle\boldsymbol{\varphi}_i,\bU_k\boldsymbol{\varphi}_i\rangle,
\qquad
p_{i,k}:=\langle\boldsymbol{\varphi}_i,\bP_k\boldsymbol{\varphi}_i\rangle,
\qquad
r_{i,k}:=\langle\boldsymbol{\varphi}_i,\bR_k\boldsymbol{\varphi}_i\rangle,
\]
and
\[
v_{i,k}:=\langle\boldsymbol{\varphi}_i,\bV_k\boldsymbol{\varphi}_i\rangle.
\]
The cross-covariance operators satisfy
\[
\bQ_k=\bP_k^*.
\]
Therefore,
\[
v_{i,k}
=
(1+\beta)^2u_{i,k}
-
2\beta(1+\beta)p_{i,k}
+
\beta^2r_{i,k}.
\]
Set
\[
d_i:=1-\eta\lambda_i,
\qquad
a_i:=d_i^2+\frac{\eta^2}{B}\lambda_i^2,
\qquad
\bX_{i,k}:=(u_{i,k},p_{i,k},r_{i,k})^\top,
\]
\[
\bb:=(1,0,0)^\top,
\qquad
\bc:=\bigl((1+\beta)^2,-2\beta(1+\beta),\beta^2\bigr)^\top,
\]
and
\[
\bM_i
:=
\begin{pmatrix}
a_i(1+\beta)^2
&
-2a_i\beta(1+\beta)
&
a_i\beta^2
\\
d_i(1+\beta)
&
-d_i\beta
&
0
\\
1
&
0
&
0
\end{pmatrix}.
\]
Then
\[
v_{i,k}=\bc^\top\bX_{i,k},
\qquad
\ell_k^{\la}=\sum_{i\geq1}\lambda_i\bc^\top\bX_{i,k},
\]
and the diagonal part of \eqref{app:eq:nesterov-covariance-recursion} becomes
\begin{equation}
\label{app:eq:nesterov-block-recursion}
\bX_{i,k+1}
=
\bM_i\bX_{i,k}
+
\frac{\eta^2}{B}\lambda_i\bb\bigl(\ell_k^{\la}+\sigma^2\bigr).
\end{equation}

The characteristic polynomial of the isolated block is
\[
p_i(\mu)
=
\mu^3
-
\left[a_i(1+\beta)^2-d_i\beta\right]\mu^2
+
a_i\beta\left[d_i(1+\beta)^2-\beta\right]\mu
-
a_id_i\beta^3.
\]
The Jury stability criterion shows that all roots lie strictly inside the unit disk if and only if
\[
\Delta_{\nesterov}(\eta\lambda_i)>0.
\]
The polynomial is a concave quadratic satisfying
\[
\Delta_{\nesterov}(0)=2B(1-\beta^2)>0,
\]
and it has exactly one positive root. Imposing the coordinate condition in every eigendirection is therefore equivalent to \eqref{app:eq:nesterov-local-stability}. If it fails, the nondegenerate label-noise forcing excites a nonstable isolated mode, and the corresponding coordinate risk is unbounded.

Assume henceforth that the local condition holds. Iterating \eqref{app:eq:nesterov-block-recursion}, multiplying the \(i\)-th coordinate contribution by \(\lambda_i\), and summing over all eigendirections gives
\begin{equation}
\label{app:eq:nesterov-renewal}
\ell_k^{\la}
=
h_k
+
\sum_{s=0}^{k-1}
\cK_{k-1-s}\bigl(\ell_s^{\la}+\sigma^2\bigr),
\end{equation}
where
\[
h_k
:=
\sum_{i\geq1}
\lambda_i\bc^\top\bM_i^k\bX_{i,0},
\qquad
\cK_m
:=
\frac{\eta^2}{B}
\sum_{i\geq1}
\lambda_i^2\bc^\top\bM_i^m\bb.
\]
Both sequences are nonnegative: their summands are look-ahead variances generated by isolated coordinate systems. The local blocks are uniformly power bounded, and positivity of the initial coordinate covariance gives
\[
0\leq h_k
\leq
C\sum_{i\geq1}\lambda_i(u_{i,0}+r_{i,0})
<\infty.
\]
A direct calculation gives
\begin{equation}
\label{app:eq:nesterov-resolvent}
\bc^\top(\bI-\bM_i)^{-1}
=
\frac{
B\bigl(
1+\beta+\beta(1+2\beta)\eta\lambda_i,
-2\beta(1+\beta),
\beta^2[1+\beta(1-\eta\lambda_i)]
\bigr)
}{
\eta\lambda_i\Delta_{\nesterov}(\eta\lambda_i)
}.
\end{equation}
Its first entry yields
\[
\sum_{m\geq0}\cK_m
=
\frac{\eta^2}{B}
\sum_{i\geq1}
\lambda_i^2\bc^\top(\bI-\bM_i)^{-1}\bb
=
\kappa_{\nesterov}(\eta,\beta).
\]
Lemma~\ref{app:lem:renewal-stability}, together with \eqref{app:eq:nesterov-risk-equivalence-forward}-- \eqref{app:eq:nesterov-risk-equivalence-backward}, gives
\[
\sup_{k\geq0}\EE[\cE(\btheta_k)]<\infty
\quad\Longleftrightarrow\quad
\kappa_{\nesterov}(\eta,\beta)<1,
\]
which completes the proof.

\end{proof}

\begin{corollary}[Critical learning rate of Nesterov]
\label{app:cor:nesterov-critical-rate}
There exists a unique critical learning rate satisfying
\[
\kappa_{\nesterov}
\bigl(\eta_{\nesterov}^{\crit},\beta\bigr)
=1
\]
before the local stability boundary. The Nesterov iteration is risk stable if and only if
\[
0<\eta<\eta_{\nesterov}^{\crit}.
\]
Under the power-law and high-momentum assumptions \( \lambda_i\eqsim i^{-q}, q>1, B\gg1, 1-\beta\ll1, \)
\begin{equation}
\label{app:eq:nesterov-critical-scaling}
\eta_{\nesterov}^{\crit}
\eqsim
\min\{1,B^q(1-\beta)\}.
\end{equation}
\end{corollary}

\begin{proof}
First consider the local condition. The polynomial is concave and satisfies
\[
\Delta_{\nesterov}(0)>0.
\]
It has a unique positive root, denoted by \(x_+(B,\beta)\), and hence the local condition is equivalent to
\[
\eta\lambda_{\max}<x_+(B,\beta).
\]
In the high-momentum regime,
\[
B\gg1,
\qquad
\beta\to1^-,
\]
the root remains of constant order. Indeed,
\[
x_+(B,\beta)\eqsim1,
\qquad
x_+(B,1)
=
\frac{4B-2}{3(B+1)}
\longrightarrow
\frac43.
\]
Together with the power-law normalization \(\lambda_{\max}\eqsim1,\) this gives the order-one local ceiling
\[
\eta\lesssim1.
\]

We next consider the global renewal mass \(\kappa_{\nesterov}(\eta,\beta)\). Each summand is continuous and strictly increasing up to the local boundary, and the leading eigendirection diverges at that boundary. Hence Theorem~\ref{app:thm:nesterov-stability} gives a unique solution of the critical equation and the stated stability interval.

On every fixed subinterval strictly inside the local stability region,
\[
\kappa_{\nesterov}(\eta,\beta)
\eqsim
\frac1B
\sum_{i\geq1}
\frac{\eta\lambda_i}{1-\beta+\eta\lambda_i}.
\]
For the power-law spectrum, the spectral sum satisfies
\[
\sum_{i\geq1}
\frac{\eta\lambda_i}{1-\beta+\eta\lambda_i}
\eqsim
\begin{cases}
\dfrac{\eta}{1-\beta},
&
\eta\lesssim1-\beta,
\\[3mm]
\left(\dfrac{\eta}{1-\beta}\right)^{1/q},
&
\eta\gtrsim1-\beta.
\end{cases}
\]
At a global critical point, the first regime would require \( \frac{\eta}{1-\beta}\eqsim B, \) contradicting \( \frac{\eta}{1-\beta}\lesssim1 \) when the batch size is large. Therefore the global critical point lies in the second regime, where
\[
\left(\frac{\eta}{1-\beta}\right)^{1/q}
\eqsim
B.
\]
Thus the global condition permits learning rates up to
\[
\eta\eqsim B^q(1-\beta).
\]
Combining this scale with the order-one local ceiling yields \eqref{app:eq:nesterov-critical-scaling}.
\end{proof}

\section{Functional scaling laws for risk dynamics}
\label{app:sec:scaling-laws}

We follow the functional scaling law framework~\citep{li2026functional} and show that the risk dynamics of a general second-order iteration can likewise be characterized by two macroscopic functions: signal learning \(\cS\) and noise forgetting \(K\). We then specialize the framework to Polyak and Nesterov in Appendices~\ref{app:hbm scaling law} and~\ref{app:nesterov scaling law}. For each method, we bound the signal term \(\cS\), characterize the label-noise contribution \(N\), and control the multiplicative-noise contribution \(M\); combining the three yields the corresponding risk scaling law. The SGD result is recalled in Appendix~\ref{app:sgd_scaling_law}.

Throughout, \(\btheta_{-1}=\btheta_0=\bm0\) for momentum methods, and \(\btheta_0=\bm0\) for SGD.

\subsection{Functional scaling law for second-order iterations}
\label{app:subsec:functional-scaling-decomposition}

Consider a second-order stochastic iteration with constant learning rate \(\eta\) and fixed coefficients \(\beta,\gamma\):
\begin{equation}
    \begin{aligned}
    \bm v_k&=\btheta_k+\gamma(\btheta_k-\btheta_{k-1}),\\
    \btheta_{k+1}&=\btheta_k+\beta(\btheta_k-\btheta_{k-1})-\eta\bg(\bm v_k;\cB_k).
    \end{aligned}
    \label{eq:general-second-order-update}
\end{equation}
Here \(\bm v_k\) is the gradient evaluation point, and \(\bg(\bm v_k;\cB_k)=\bH(\bm v_k-\btheta^*)+\bxi_k(\bm v_k)\), where \(\bxi_k(\bm v_k)\) is the mini-batch gradient noise. In the eigenbasis of \(\bH\), set \(e_k(j)=\theta_k(j)-\theta_j^*\) and \(\zeta_k(j)=[\bxi_k(\bm v_k)]_j\). Subtracting \(\btheta^*\) from the update and collecting the two error terms gives
\begin{equation}
    e_{k+1}(j)+a_je_k(j)+b_je_{k-1}(j)=-\eta\zeta_k(j),
    \label{eq:unified-coordinate-recursion}
\end{equation}
where \(a_j=\eta\lambda_j(1+\gamma)-(1+\beta)\) and \(b_j=\beta-\eta\lambda_j\gamma\). Thus \(\zeta_k(j)\) is exactly the gradient noise in coordinate \(j\) of the stochastic update. Let \(\mathscr F_k\) be the history before mini-batch \(k\) is sampled. We assume that this noise satisfies
\begin{equation}
    \EE[\zeta_k(j)\mid\mathscr F_k]=0,\qquad
    \EE[\zeta_k(j)^2\mid\mathscr F_k]
    \eqsim\frac{\lambda_j}{B}(\mathcal J_k+\sigma^2),
    \label{eq:unified-coordinate-noise-property}
\end{equation}
where \(\mathcal J_k=\cE(\bm v_k)\) is the excess risk at the gradient evaluation point, which depends on the algorithm: Polyak uses \(\gamma=0\), so \(\mathcal J_k=\cE(\btheta_k)\), whereas Nesterov uses \(\gamma=\beta\), so \(\mathcal J_k=\cE(\btheta_k^{\la})\). Lemma~\ref{app:lem:noise-geometry} provides these conditional noise bounds. To derive the scaling law, we need to estimate the expected excess risk \(R_k:=\EE[\cE(\btheta_k)]=\frac12\sum_j\lambda_j\EE[e_k(j)^2]\).

Let \(e_k^{\mathrm{det}}(j)\) be the error sequence obtained by setting the gradient noise $\zeta_k(j)$ to zero in \eqref{eq:unified-coordinate-recursion}, so that 
\[
    e_{k+1}^{\mathrm{det}}(j)+a_je_k^{\mathrm{det}}(j)+b_je_{k-1}^{\mathrm{det}}(j)=0,
\]
with \(e_{-1}^{\mathrm{det}}(j)=e_{-1}(j)\) and \(e_0^{\mathrm{det}}(j)=e_0(j)\) (under our zero initialization, both initial errors equal \(-\theta_j^*\)). Let \(G_k(j)\) satisfy the same noise-free recursion, but with \(G_{-1}(j)=0\) and \(G_0(j)=1\). Lemma~\ref{app:lem:forced-second-order-iteration} gives
\begin{equation}
    e_k(j)=e_k^{\mathrm{det}}(j)
       -\eta\sum_{\tau=0}^{k-1}G_\tau(j)\zeta_{k-1-\tau}(j).
    \label{eq:unified-impulse-response}
\end{equation}
Following the framework of functional scaling laws~\citep{li2026functional}, we introduce  the signal-learning function and the noise-forgetting kernel by
\begin{equation}
    \cS(k):=\frac12\sum_{j\ge1}\lambda_j(e_k^{\mathrm{det}}(j))^2,
    \qquad  K(\tau):=\sum_{j\ge1}\lambda_j^2G_\tau(j)^2.
    \label{eq:unified-signal-kernel}
\end{equation}
\begin{theorem}[Functional scaling law for second-order iterations]
\label{thm:unified-risk-recursion}
Under \eqref{eq:unified-coordinate-noise-property}, the expected excess risk satisfies
\begin{equation}
    R_k\eqsim\cS(k)+\frac{\eta^2}{B}\sum_{\tau=0}^{k-1}
         K(\tau)\bigl(\sigma^2+\EE[\mathcal J_{k-1-\tau}]\bigr).
    \label{eq:unified-risk-recursion-expanded}
\end{equation}
\end{theorem}
\begin{proof}
By \eqref{eq:unified-coordinate-noise-property}, \(\EE[\zeta_m(j)]=0\) and, for \(m<n\), \(\EE[\zeta_m(j)\zeta_n(j)]=\EE[\zeta_m(j)\EE[\zeta_n(j)\mid\mathscr F_n]]=0\). Squaring \eqref{eq:unified-impulse-response} and taking expectations gives
\begin{equation}
    \EE[e_k(j)^2]=(e_k^{\mathrm{det}}(j))^2
       +\eta^2\sum_{\tau=0}^{k-1}G_\tau(j)^2\EE[\zeta_{k-1-\tau}(j)^2].
    \label{eq:unified-second-moment-coordinate}
\end{equation}
We multiply both sides by \(\lambda_j/2\) and sum over \(j\). The noise bound \eqref{eq:unified-coordinate-noise-property} then yields the result.
\end{proof}

The two noise contributions have different sources. Label noise is of scale $\sigma^2$, which is independent of the current error, whereas multiplicative noise depends on the risk at the gradient evaluation point. Write
\begin{equation}
    N(k):=\frac{\eta^2\sigma^2}{B}\sum_{\tau=0}^{k-1} K(\tau),
    \qquad
    M(k):=\frac{\eta^2}{B}\sum_{\tau=0}^{k-1}
         K(\tau)\EE[\mathcal J_{k-1-\tau}].
    \label{eq:unified-noise-terms}
\end{equation}
The scaling law then takes the form
\begin{equation}
    \EE[\cE(\btheta_k)]\eqsim\cS(k)+N(k)+M(k).
    \label{eq:unified-risk-recursion}
\end{equation}
\subsection{\texorpdfstring{Scaling law for Polyak (Proof of Theorem~\ref{thm:hbm scaling law})}{Scaling law for Polyak (Proof of Theorem)}}
\label{app:hbm scaling law}
\subsubsection{Polyak update}
The Polyak update~\eqref{eq:hb-update} is a special case of \eqref{eq:general-second-order-update} with \(\gamma=0\), so \(\bm v_k=\btheta_k\). The coefficients in Appendix~\ref{app:subsec:functional-scaling-decomposition} are therefore \(a_j=-(1+\beta-\eta\lambda_j)\) and \(b_j=\beta\), and the error recursion becomes
\begin{equation}
    e_{k+1}(j)=(1+\beta-\eta\lambda_j)e_k(j)-\beta e_{k-1}(j)-\eta\xi_k(j).
    \label{eq:pol_coordinate_recursion_unified}
\end{equation}
Here \(\xi_k(j)=[\bxi_k(\btheta_k)]_j\) is the mini-batch gradient noise at \(\btheta_k\), so \(\zeta_k(j)=\xi_k(j)\). By Lemma~\ref{app:lem:noise-geometry}, the noise satisfies \eqref{eq:unified-coordinate-noise-property} with \(\mathcal J_k=\cE(\btheta_k)\). Theorem~\ref{thm:unified-risk-recursion} gives
\begin{equation}
    R_k\eqsim\cS_{\polyak}(k)+N_{\polyak}(k)+M_{\polyak}(k).
    \label{eq:pol-fsl-application}
\end{equation}
Here
\begin{equation}
    \begin{aligned}
    \cS_{\polyak}(k)&:=\frac12\sum_{j\ge1}\lambda_j(e_k^{\mathrm{det}}(j))^2,\qquad
     K_{\polyak}(\tau):=\sum_{j\ge1}\lambda_j^2G_\tau(j)^2,\\
    N_{\polyak}(k)&:=\frac{\eta^2\sigma^2}{B}\sum_{\tau=0}^{k-1} K_{\polyak}(\tau),\\
    M_{\polyak}(k)&:=\frac{\eta^2}{B}\sum_{\tau=0}^{k-1} K_{\polyak}(\tau)R_{k-1-\tau},
    \end{aligned}
    \label{eq:pol-signal-kernel-noise}
\end{equation}
where \(e_k^{\mathrm{det}}(j)\) and \(G_k(j)\) are the noise-free error and Green function associated with \eqref{eq:pol_coordinate_recursion_unified}, with the initial conditions specified in Appendix~\ref{app:subsec:functional-scaling-decomposition}. In the following we estimate the signal learning \(\cS_{\polyak}\), label noise \(N_{\polyak}\), and multiplicative noise \(M_{\polyak}\).

\subsubsection{Classification of damping regimes}
These estimates depend on \(e_k^{\mathrm{det}}(j)\) and \(G_k(j)\), which satisfy the same second-order recursion obtained from \eqref{eq:pol_coordinate_recursion_unified} by setting the noise to zero. We use the results in Appendix~\ref{app:second-order-toolkit} to bound these sequences through the roots of the characteristic equation. We first determine the damping threshold. The characteristic polynomial is \(r^2-(1+\beta-\eta\lambda_j)r+\beta\), and its discriminant factors as
\[
    \Delta_j=(\eta\lambda_j-(1-\sqrt\beta)^2)(\eta\lambda_j-(1+\sqrt\beta)^2).
\]
Suppose \(\eta\lambda_1<1\). The damping threshold is \((1-\sqrt\beta)^2\eqsim(1-\beta)^2\). Let \(J_{\max}\) be the largest index satisfying \(\eta\lambda_j\ge (1-\sqrt\beta)^2\), i.e.,
\[
    J_{\max}:=\max\{j\ge1:\eta\lambda_j\ge (1-\sqrt\beta)^2\}.
\]
If no such index exists, we set \(J_{\max}=1\). Since \(\lambda_j\eqsim j^{-q}\),
\begin{equation}
    J_{\max}\eqsim\max\left\{1,\left(\frac{\eta}{(1-\beta)^2}\right)^{1/q}\right\}.
    \label{eq:Jmax-unified}
\end{equation}
The transition between the two regimes occurs at \(\eta\eqsim(1-\beta)^2\):
\begin{itemize}
    \item \textbf{Fully overdamped regime:} \(\eta\lesssim(1-\beta)^2\). Every coordinate is overdamped.
    \item \textbf{Mixed-damping regime:} \(\eta\gtrsim(1-\beta)^2\). The leading coordinates \(j\le J_{\max}\) are underdamped, while the tail \(j>J_{\max}\) is overdamped.
\end{itemize}
Here \(\eta\lesssim(1-\beta)^2\) means \(\eta\lambda_1\le c_0(1-\beta)^2\) for a sufficiently small constant \(c_0>0\), so that Lemma~\ref{lem:pol-overdamped-coordinate} applies to every coordinate.

\subsubsection{\texorpdfstring{Estimates of signal learning \(\cS_{\polyak}(k)\)}{Estimates of signal learning S (polyak)}}
\label{app:signal-learning-momentum}

We estimate \(\cS_{\polyak}(k)\) by summing the weighted squared noise-free coordinate errors. We first treat the fully overdamped regime, then separate the overdamped and underdamped contributions in the mixed-damping regime.

\paragraph{Fully overdamped regime.}
\label{app:signal-fully-overdamped-momentum}
In this regime, all coordinates are overdamped, and the effective learning rate is \(\eta/(1-\beta)\). The following proposition shows that \(\cS_{\polyak}(k)\) stays of constant order until \(k\eqsim(1-\beta)/\eta\), then follows a power-law decay.

\begin{proposition}[Signal learning for Polyak in the fully overdamped regime]
\label{thm:deterministic-risk-fully-overdamped}

Under Assumptions~\ref{ass:diagonal-covariance} and~\ref{ass:power-law}, if \(\eta\lesssim(1-\beta)^2\), then
\begin{equation}
\label{eq:det-risk-fully-overdamped-unified}
    \cS_{\polyak}(k)
    \eqsim
    \begin{cases}
    1,
    & k\lesssim\dfrac{1-\beta}{\eta},
    \\[3mm]
    \left(\dfrac{1-\beta}{\eta k}\right)^s,
    & k\gtrsim\dfrac{1-\beta}{\eta}.
    \end{cases}
\end{equation}
\end{proposition}

\begin{proof}
Since the noise-free recursion is linear and \(e_{-1}^{\mathrm{det}}(j)=e_0^{\mathrm{det}}(j)=-\theta_j^*\), we write \(e_k^{\mathrm{det}}(j)=-h_k(j)\theta_j^*\), where \(h_k(j)\) describes how the initial error evolves in coordinate \(j\). By \eqref{eq:pol_coordinate_recursion_unified}, it satisfies
\[
    h_{k+1}(j)=(1+\beta-\eta\lambda_j)h_k(j)-\beta h_{k-1}(j),\qquad h_{-1}(j)=h_0(j)=1.
\]
Applying Lemma~\ref{lem:pol-overdamped-coordinate} with \(x=\eta\lambda_j\) gives
\[
    c e^{-C\eta\lambda_j k/(1-\beta)}\le h_k(j)\le C e^{-c\eta\lambda_j k/(1-\beta)},\qquad j\ge1,\quad k\ge0.
\]
Since \(\cS_{\polyak}(k)=\frac12\sum_j\lambda_j(\theta_j^*)^2h_k(j)^2\) and \(\lambda_j(\theta_j^*)^2\eqsim j^{-1-qs}\), applying Lemma~\ref{lem:power-law-spectral-tail} with \(J=1\) to both exponential bounds yields
\[
    \cS_{\polyak}(k)\eqsim\left(1+\frac{\eta k}{1-\beta}\right)^{-s}.
\]
This proves the proposition.
\end{proof}

\paragraph{Mixed-damping regime.}
\label{app:signal-mixed-damping-momentum}

When \(\eta\gtrsim(1-\beta)^2\), we enter the mixed-damping regime. In this regime, the crossover index \(J_{\max}\) separates the underdamped leading coordinates \(j\le J_{\max}\) from the overdamped tail \(j>J_{\max}\). Accordingly, we write \(\cS_{\polyak}(k)=\cS_{\rm over}(k)+\cS_{\rm under}(k)\), where
\[
\begin{aligned}
    \cS_{\rm over}(k)&:=\frac12\sum_{j>J_{\max}}\lambda_j\bigl(e_k^{\mathrm{det}}(j)\bigr)^2,\\
    \cS_{\rm under}(k)&:=\frac12\sum_{1\le j\le J_{\max}}\lambda_j\bigl(e_k^{\mathrm{det}}(j)\bigr)^2.
\end{aligned}
\]
The next two propositions estimate these contributions separately.

\begin{proposition}[Overdamped signal learning for Polyak in the mixed-damping regime]
\label{thm:deterministic-bias-mixed-unified}

Under Assumptions~\ref{ass:diagonal-covariance} and~\ref{ass:power-law}, assume \(\eta\gtrsim(1-\beta)^2\). Then the overdamped contribution satisfies
\begin{equation}
\label{eq:mixed-over-tail-unified}
    \cS_{\rm over}(k)
    \eqsim
    \begin{cases}
    \left(\dfrac{(1-\beta)^2}{\eta}\right)^s,
    & k\lesssim(1-\beta)^{-1},
    \\[3mm]
    \left(\dfrac{1-\beta}{\eta k}\right)^s,
    & k\gtrsim(1-\beta)^{-1}.
    \end{cases}
\end{equation}
\end{proposition}

\begin{proof}
Write \(e_k^{\mathrm{det}}(j)=-h_k(j)\theta_j^*\). Lemma~\ref{lem:pol-overdamped-coordinate} applies when \(\eta\lambda_j\le c_0(1-\beta)^2\). To separate these coordinates from those near the damping threshold, set
\[
    J_0:=\min\{j>J_{\max}:\eta\lambda_j\le c_0(1-\beta)^2\}.
\]
Then
\[
    J_0\eqsim J_{\max}\eqsim\left(\frac{\eta}{(1-\beta)^2}\right)^{1/q}.
\]
The coordinate bound and Lemma~\ref{lem:power-law-spectral-tail} give
\[
    \frac12\sum_{j\ge J_0}\lambda_j(\theta_j^*)^2h_k(j)^2
    \eqsim\left(J_0^q+\frac{\eta k}{(1-\beta)}\right)^{-s}
    \eqsim J_{\max}^{-qs}(1+(1-\beta) k)^{-s}.
\]
For \(J_{\max}<j<J_0\), we have \(\eta\lambda_j\eqsim(1-\beta)^2\). The positive gaps \(1-r_{j,\pm}\) have sum \((1-\beta)+\eta\lambda_j\eqsim(1-\beta)\) and product \(\eta\lambda_j\eqsim(1-\beta)^2\), so both are of order \((1-\beta)\). The roots have positive sum and nonnegative product \(\beta\), hence \(0\le r_{j,-}\le r_{j,+}\le1-c(1-\beta)\). Equation~\eqref{app:eq:divided-difference-solution} therefore gives
\[
    |h_k(j)|\lesssim(1+(1-\beta) k)e^{-c(1-\beta)k}.
\]
Consequently,
\[
    \frac12\sum_{J_{\max}<j<J_0}\lambda_j(\theta_j^*)^2h_k(j)^2
    \lesssim J_{\max}^{-qs}(1+(1-\beta) k)^2e^{-2c(1-\beta) k}
    \lesssim J_{\max}^{-qs}(1+(1-\beta) k)^{-s}.
\]
Combining the two parts and substituting the order of \(J_{\max}\) yields
\[
    \cS_{\rm over}(k)\eqsim
    \left(\frac{(1-\beta)^2}{\eta}\right)^s(1+(1-\beta) k)^{-s}.
\]
The two time ranges follow by comparing \(k\) with \((1-\beta)^{-1}\). This proves the proposition.
\end{proof}

We next estimate the underdamped contribution. The following proposition gives an exponentially decaying upper bound for this contribution.

\begin{proposition}[Underdamped signal learning for Polyak in the mixed-damping regime]
\label{thm:pol-underdamped-transient-unified}
Under Assumptions~\ref{ass:diagonal-covariance} and~\ref{ass:power-law}, assume \(\eta\gtrsim(1-\beta)^2\) and \(\eta\lambda_1\le c_{\rm st}\) for sufficiently small \(c_{\rm st}>0\). Then there exists a nonnegative separated-underdamped contribution \(\TT_P(k)\) such that
\begin{equation}
\label{eq:pol-underdamped-bound-unified}
    \TT_P(k)
    \le \cS_{\rm under}(k)
    \lesssim
    \TT_P(k)
    +
    \left(\frac{(1-\beta)^2}{\eta}\right)^s
    \bigl(1+(1-\beta)k\bigr)^{-s},
    \qquad k\ge0,
\end{equation}
and
\begin{equation}
\label{eq:pol-separated-underdamped-bound}
    0\le \TT_P(k)\lesssim\beta^k,
    \qquad k\ge0.
\end{equation}
\end{proposition}

\begin{proof}
Recall that \(e_k^{\mathrm{det}}(j)=-h_k(j)\theta_j^*\), with \(h_{-1}(j)=h_0(j)=1\). For each underdamped coordinate, the characteristic roots are
\[
    r_{j,\pm}=\sqrt\beta e^{\pm i\omega_j},\qquad
    \cos\omega_j=\frac{1+\beta-\eta\lambda_j}{2\sqrt\beta}.
\]
The initial values give the common representation
\[
    h_k(j)=\beta^{k/2}\bigl[\cos(k\omega_j)+A_j\sin(k\omega_j)\bigr],
    \qquad
    A_j=\frac{1-\beta-\eta\lambda_j}{2\sqrt\beta\sin\omega_j}.
\]
At the damping threshold \(\eta\lambda_j=(1-\sqrt\beta)^2\), the expression is understood by continuity. Away from this threshold, \(A_j\) is uniformly bounded, so we can bound the sine and cosine terms separately. Near the threshold, \(\sin\omega_j\) can approach zero and \(A_j\) is not uniformly bounded. We therefore estimate \(A_j\sin(k\omega_j)\) together, using the root formula \eqref{app:eq:divided-difference-solution}, which remains valid when the roots coincide. This distinction motivates the following partition. Fix a constant \(C_{\rm mom}>2\) and set
\[
    \begin{aligned}
    \mathcal I_{\rm cross}&:=\{j\le J_{\max}:\eta\lambda_j\le C_{\rm mom}(1-\beta)^2\},\\
    \mathcal I_{\rm sep}&:=\{j\le J_{\max}:\eta\lambda_j>C_{\rm mom}(1-\beta)^2\}.
    \end{aligned}
\]
Accordingly,
\[
    \cS_{\rm under}(k)=\cS_{\rm cross}(k)+\TT_P(k),
\]
where
\[
    \cS_{\rm cross}(k):=\frac12\sum_{j\in\mathcal I_{\rm cross}}\lambda_j(\theta_j^*)^2h_k(j)^2,
    \qquad
    \TT_P(k):=\frac12\sum_{j\in\mathcal I_{\rm sep}}\lambda_j(\theta_j^*)^2h_k(j)^2.
\]
On \(\mathcal I_{\rm cross}\), \(|r_{j,-}(1-r_{j,+})|^2=\beta\eta\lambda_j\). The divided-difference bound in \eqref{app:eq:divided-difference-solution} gives
\[
    |h_k(j)|
    \le\bigl(1+k\sqrt{\eta\lambda_j}\bigr)\beta^{k/2}
    \lesssim\bigl(1+(1-\beta)k\bigr)\beta^{k/2},
\]
including at a repeated root. To sum this coordinate bound, note that \(\eta\lambda_j\le C_{\rm mom}(1-\beta)^2\) and \(\lambda_j\eqsim j^{-q}\) imply
\[
    j\gtrsim\left(\frac{\eta}{(1-\beta)^2}\right)^{1/q}
\]
on \(\mathcal I_{\rm cross}\). Thus the power-law weights give
\[
    \sum_{j\in\mathcal I_{\rm cross}}\lambda_j(\theta_j^*)^2
    \lesssim
    \sum_{j\gtrsim(\eta/(1-\beta)^2)^{1/q}}j^{-1-qs}
    \lesssim
    \left(\frac{(1-\beta)^2}{\eta}\right)^s.
\]
Substituting both bounds into the definition of \(\cS_{\rm cross}\) yields
\[
    \cS_{\rm cross}(k)
    \lesssim
    \left(\frac{(1-\beta)^2}{\eta}\right)^s
    \bigl(1+(1-\beta)k\bigr)^2\beta^k.
\]
Since \(\beta^k\le e^{-(1-\beta)k}\) and
\[
    \sup_{t\ge0}(1+t)^{s+2}e^{-t}<\infty,
\]
we obtain
\[
    \bigl(1+(1-\beta)k\bigr)^2\beta^k
    \lesssim
    \bigl(1+(1-\beta)k\bigr)^{-s}.
\]
Therefore,
\begin{equation}
\label{eq:pol-cross-absorbed-by-over}
    \cS_{\rm cross}(k)
    \lesssim
    \left(\frac{(1-\beta)^2}{\eta}\right)^s
    \bigl(1+(1-\beta)k\bigr)^{-s}
    \eqsim
    \cS_{\rm over}(k),
\end{equation}
where the last equivalence follows from Proposition~\ref{thm:deterministic-bias-mixed-unified}.

It remains to control the separated underdamped modes. For \(j\in\mathcal I_{\rm sep}\),
\[
    4\beta\sin^2\omega_j
    =\bigl(\eta\lambda_j-(1-\sqrt\beta)^2\bigr)
      \bigl((1+\sqrt\beta)^2-\eta\lambda_j\bigr)
    \gtrsim\eta\lambda_j,
\]
and
\[
    (1-\beta-\eta\lambda_j)^2\lesssim\eta\lambda_j.
\]
Hence \(A_j^2\lesssim1\), and therefore \(h_k(j)^2\lesssim\beta^k\) on \(\mathcal I_{\rm sep}\). Since \(\sum_{j\ge1}\lambda_j(\theta_j^*)^2\lesssim1\),
\[
    0\le \TT_P(k)
    \lesssim
    \beta^k\sum_{j\in\mathcal I_{\rm sep}}\lambda_j(\theta_j^*)^2
    \lesssim\beta^k,
\]
which proves \eqref{eq:pol-separated-underdamped-bound}. Since \(\cS_{\rm under}(k)=\cS_{\rm cross}(k)+\TT_P(k)\), the bound \eqref{eq:pol-cross-absorbed-by-over} gives \eqref{eq:pol-underdamped-bound-unified}.
\end{proof}

\begin{corollary}[Total signal learning for Polyak in the mixed-damping regime]
\label{cor:pol-mixed-signal-absorbed}
Under Assumptions~\ref{ass:diagonal-covariance} and~\ref{ass:power-law}, assume \(\eta\gtrsim(1-\beta)^2\) and \(\eta\lambda_1\le c_{\rm st}\) for sufficiently small \(c_{\rm st}>0\). Then
\begin{equation}
\label{eq:pol-mixed-signal-absorbed}
    \cS_{\polyak}(k)
    =\cS_{\rm over}(k)+\cS_{\rm under}(k)
    \eqsim
    \TT_P(k)
    +
    \left(\frac{(1-\beta)^2}{\eta}\right)^s
    \bigl(1+(1-\beta)k\bigr)^{-s},
    \qquad k\ge0.
\end{equation}
\end{corollary}

\begin{proof}
By Proposition~\ref{thm:pol-underdamped-transient-unified},
\[
    \cS_{\rm under}(k)=\TT_P(k)+\cS_{\rm cross}(k),
    \qquad
    \cS_{\rm cross}(k)\lesssim\cS_{\rm over}(k).
\]
Since \(\cS_{\rm over}(k)\), \(\cS_{\rm cross}(k)\), and \(\TT_P(k)\) are nonnegative,
\[
    \cS_{\rm over}(k)+\TT_P(k)
    \le
    \cS_{\polyak}(k)
    =
    \cS_{\rm over}(k)+\cS_{\rm cross}(k)+\TT_P(k)
    \lesssim
    \cS_{\rm over}(k)+\TT_P(k).
\]
Proposition~\ref{thm:deterministic-bias-mixed-unified} gives
\[
    \cS_{\rm over}(k)
    \eqsim
    \left(\frac{(1-\beta)^2}{\eta}\right)^s
    \bigl(1+(1-\beta)k\bigr)^{-s},
\]
which proves the corollary.
\end{proof}

\subsubsection{Estimates of label noise}
\label{app:forgetting-kernel-properties}
\label{app:pol-forgetting-kernel-unified}

Recall from \eqref{eq:pol-signal-kernel-noise} that
\[
    N_{\polyak}(k)=\frac{\eta^2\sigma^2}{B}\sum_{\tau=0}^{k-1} K_{\polyak}(\tau),\qquad
     K_{\polyak}(\tau)=\sum_{j\ge1}\lambda_j^2\bigl(G_\tau^{\polyak}(j)\bigr)^2.
\]
Here \(G_\tau^{\polyak}(j)\) is the Green function associated with \eqref{eq:pol_coordinate_recursion_unified}. The following proposition establishes a property of the forgetting kernel that we use to estimate \(N_{\polyak}(k)\).

\begin{proposition}[A property of the forgetting kernel for Polyak]
\label{lem:pol-forgetting-prefix-unified}
Assume \(\eta\lambda_1\le c_{\rm st}<1\).  There is a universal constant \(C>0\) such that, whenever
\[
    k\ge
    C\max\left\{
        \frac1{1-\beta},
        \frac{1-\beta}{\eta}
    \right\},
\]
one has
\begin{equation}
    \sum_{\tau=0}^{k-1} K_{\polyak}(\tau)
    \eqsim \frac1{\eta(1-\beta)}.
    \label{eq:pol-forgetting-prefix-unified}
\end{equation}
\end{proposition}

\begin{proof}
The learning-rate condition ensures that both characteristic roots have modulus less than one. Applying Lemma~\ref{lem:generic-label-noise-energy} with \(a_j=-(1+\beta-\eta\lambda_j)\) and \(b_j=\beta\) from \eqref{eq:pol_coordinate_recursion_unified} gives
\begin{equation}
    \sum_{\tau=0}^{\infty}
    \left(G_\tau^{\polyak}(j)\right)^2
    =
    \frac{1+\beta}
    {(1-\beta)\eta\lambda_j\left(2(1+\beta)-\eta\lambda_j\right)}.
    \label{eq:pol_green_energy_unified}
\end{equation}
Consequently,
\begin{align}
    \sum_{\tau=0}^{\infty} K_{\polyak}(\tau)
    =
    \frac{1}{\eta}
    \sum_{j\ge1}
    \frac{(1+\beta)\lambda_j}
    {(1-\beta)\left(2(1+\beta)-\eta\lambda_j\right)}
    \eqsim
    \frac{1}{\eta(1-\beta)}.
    \label{eq:pol_forgetting_kernel_mass}
\end{align}
The last equivalence uses the small-step bound \(\eta\lambda_1\le c_{\rm st}<1\) and \(\sum_{j\ge1}\lambda_j\eqsim1\).
This gives the upper bound in \eqref{eq:pol-forgetting-prefix-unified}.  For the lower bound, choose a fixed \(J_*\) such that
\[
    \sum_{j=1}^{J_*}\lambda_j
    \ge\frac12\sum_{j\ge1}\lambda_j.
\]
For \(j\le J_*\), set \(\nu_j:=\min\{1-\beta,\eta\lambda_j/(1-\beta)\}\) and \(z_j:=1+\varepsilon\nu_j\), where \(\varepsilon>0\) is a sufficiently small absolute constant. Since \(\nu_j\le1-\beta\) and \((1-\beta)\nu_j\le \eta\lambda_j\), we have \(\sqrt{z_j}\,|r_{j,\pm}|<1\), \(1-\beta z_j\eqsim1-\beta\), and \((1+\beta z_j)^2-a_j^2z_j\eqsim \eta\lambda_j\). The weighted energy in Lemma~\ref{lem:generic-label-noise-energy} is therefore comparable to the total energy in \eqref{eq:pol_green_energy_unified}. Since \(\log z_j\eqsim\nu_j\), the tail bound in that lemma gives constants \(C_G,c_G>0\) such that
\begin{equation}
    \sum_{\tau=k}^{\infty}
    \left(G_\tau^{\polyak}(j)\right)^2
    \le
    C_G
    \exp\!\left\{
        -c_Gk
        \min\left(
            1-\beta,
            \frac{\eta\lambda_j}{1-\beta}
        \right)
    \right\}
    \sum_{\tau=0}^{\infty}
    \left(G_\tau^{\polyak}(j)\right)^2.
    \label{eq:pol-green-tail-unified}
\end{equation}
For \(j\le J_*\),
\[
    \min\left(
        1-\beta,
        \frac{\eta\lambda_j}{1-\beta}
    \right)
    \gtrsim
    \min\left(
        1-\beta,
        \frac{\eta}{1-\beta}
    \right).
\]
The time bound makes the tail in~\eqref{eq:pol-green-tail-unified} at most a fixed fraction of the full energy. Sum over \(j\le J_*\) and use~\eqref{eq:pol_green_energy_unified} to obtain \eqref{eq:pol-forgetting-prefix-unified}.
\end{proof}

Multiplying the kernel sum by \(\eta^2\sigma^2/B\) gives the label-noise contribution.
\begin{corollary}[Label noise for Polyak]
\label{lem:pol-label-noise-unified}
Assume \(\eta\lambda_1\le c_{\rm st}<1\). There is a constant \(C>0\) such that, if
\[
    k\ge C\max\left\{\frac1{1-\beta},\frac{1-\beta}{\eta}\right\},
\]
then
\begin{equation}
    N_{\polyak}(k)\eqsim \frac{\sigma^2\eta}{B(1-\beta)}.
    \label{eq:pol_label_noise_finite_unified}
\end{equation}
\end{corollary}

\begin{proof}
This follows immediately from the definition of \(N_{\polyak}(k)\) and Proposition~\ref{lem:pol-forgetting-prefix-unified}.
\end{proof}

\subsubsection{Control of multiplicative noise}
Recall from \eqref{eq:pol-signal-kernel-noise} that
\[
    M_{\polyak}(k)=\frac{\eta^2}{B}\sum_{\tau=0}^{k-1} K_{\polyak}(\tau)R_{k-1-\tau}.
\]
In this appendix, we control this term and show that
\[
    M_{\polyak}(k)\lesssim N_{\polyak}(k).
\] We first establish a uniform risk bound using Lemma~\ref{lem:convolution-bound} and the kernel estimate established above. The following proposition gives the resulting control of multiplicative noise.

\begin{proposition}[Control of multiplicative noise for Polyak]
\label{thm:uniform-risk-boundedness-momentum}
\label{thm:momentum-noise-absorption}
Suppose \(\sigma>0\) and \(\eta\le c\min\{1,B(1-\beta)\}\) for sufficiently small \(c>0\). There is a constant \(M<\infty\), independent of \(k,\eta,\beta,B\), such that
\[
    \sup_{k\ge0}\EE[\cE(\btheta_k)]\le M.
\]
Moreover,
\[
    M_{\polyak}(k)\lesssim N_{\polyak}(k).
\]
\end{proposition}
\begin{proof}
The signal learning bounds above give \(\sup_k\cS_{\polyak}(k)\le S_*\), uniformly over the admissible parameters. By \eqref{eq:pol_forgetting_kernel_mass},
\[
    \kappa:=\frac{\eta^2}{B}\sum_{\tau\ge0} K_{\polyak}(\tau)
    \lesssim\frac{\eta}{B(1-\beta)}\le c.
\]
Let \(A\ge1\) satisfy
\[
    R_k\le A\bigl(\cS_{\polyak}(k)+N_{\polyak}(k)+M_{\polyak}(k)\bigr).
\]
 Set \(x_k=y_k=R_k\), \(u_k=A\cS_{\polyak}(k)\), and \(w_\tau=A\eta^2 K_{\polyak}(\tau)/B\). Take \(v=\sigma^2\), \(L=1\), and \(\iota=0\). The risk inequality above has the form required by Lemma~\ref{lem:convolution-bound}, with \(y_m\le Lx_{m+\iota}\) and \(U:=\sup_k u_k\le AS_*\). By these definitions,
\[
    L\sum_{\tau\ge0}w_\tau=A\frac{\eta^2}{B}\sum_{\tau\ge0} K_{\polyak}(\tau)=A\kappa.
\]
Since \(\kappa\lesssim c\), choosing \(c\) sufficiently small ensures \(A\kappa\le1/2\), as required by the lemma. Applying Lemma~\ref{lem:convolution-bound} now gives
\[
    \sup_k R_k=\sup_k x_k\le2U+\frac{v}{L}\le 2AS_*+\sigma^2=:M.
\]
 Hence
\[
    M_{\polyak}(k)\le\frac{\eta^2M}{B}\sum_{\tau=0}^{k-1} K_{\polyak}(\tau)=\frac{M}{\sigma^2}N_{\polyak}(k).
\]
\end{proof}

The multiplicative-noise bound gives a simplified form of the scaling law.
\begin{corollary}[Scaling law of simplified form for Polyak]
\label{cor:pol-simplified-scaling-law}
Under the assumptions of Proposition~\ref{thm:momentum-noise-absorption},
\begin{equation}
    \EE[\cE(\btheta_k)]\eqsim\cS_{\polyak}(k)+N_{\polyak}(k).
    \label{eq:unified-risk-after-absorption}
\end{equation}
\end{corollary}
\begin{proof}
By Proposition~\ref{thm:momentum-noise-absorption}, \(M_{\polyak}(k)\lesssim N_{\polyak}(k)\). Substituting this bound into \eqref{eq:pol-fsl-application} and using nonnegativity proves the claim.
\end{proof}

\subsubsection{Completion of the proof}
We now combine the signal learning and noise estimates to establish the following scaling laws in the two damping regimes. The constant \(c>0\) is sufficiently small to ensure \(\eta\lambda_1\le c_{\rm st}\), and \(C>0\) is sufficiently large to satisfy the time conditions of the preceding estimates.

\paragraph{Fully overdamped regime.}

Here the overdamped signal learning estimate combines directly with the label noise bound, since the multiplicative noise is controlled by the latter.

\begin{theorem}[Scaling law for Polyak in the fully overdamped regime]
\label{thm:pol_scaling_sgd}
Under Assumptions~\ref{ass:diagonal-covariance} and~\ref{ass:power-law}, suppose \(\sigma>0\), assume the Polyak stability condition \(\eta\le c\min\{1,B(1-\beta)\}\), and consider the fully overdamped regime \(\eta\lesssim(1-\beta)^2\). For \(k\) satisfying
\[
    k\ge C\frac{1-\beta}{\eta},
\]
the expected excess risk satisfies
\[
    \EE[\cE(\btheta_k)]
    \eqsim
    \left(\frac{1-\beta}{\eta k}\right)^s
    +
    \frac{\sigma^2\eta}{B(1-\beta)}.
\]
\end{theorem}

\begin{proof}
By Corollary~\ref{cor:pol-simplified-scaling-law}, \(R_k\eqsim\cS_{\polyak}(k)+N_{\polyak}(k)\). Proposition~\ref{thm:deterministic-risk-fully-overdamped} gives \(\cS_{\polyak}(k)\eqsim(1+\eta k/(1-\beta))^{-s}\). Since \(k\ge C(1-\beta)/\eta\), this is comparable to \(((1-\beta)/(\eta k))^s\). The stated learning-rate and time conditions also allow us to apply Corollary~\ref{lem:pol-label-noise-unified}. Thus \(N_{\polyak}(k)\eqsim\sigma^2\eta/[B(1-\beta)]\). Substituting these estimates proves the theorem.
\end{proof}

\paragraph{Mixed-damping regime.}

Here we retain the underdamped transient and combine the overdamped tail estimate with the label noise bound.

\begin{theorem}[Scaling law for Polyak in the mixed-damping regime]
\label{thm:pol_two_sided_scaling}
Under Assumptions~\ref{ass:diagonal-covariance} and~\ref{ass:power-law}, suppose \(\sigma>0\), assume the Polyak stability condition \(\eta\le c\min\{1,B(1-\beta)\}\), and consider the mixed-damping regime \(\eta\gtrsim(1-\beta)^2\). For \(k\) satisfying
\[
    k\ge\frac{C}{1-\beta},
\]
the expected excess risk satisfies
\begin{equation}
\label{eq:pol_mixed_scaling_law}
    \EE[\cE(\btheta_k)]
    \eqsim
    \TT_P(k)
    +
    \left(\frac{1-\beta}{\eta k}\right)^s
    +
    \frac{\sigma^2\eta}{B(1-\beta)}.
\end{equation}
Here \(\TT_P(k)\) is the separated underdamped transient defined in Proposition~\ref{thm:pol-underdamped-transient-unified} and satisfies
\[
    0\le \TT_P(k)\lesssim\beta^k.
\]
\end{theorem}

\begin{proof}
By Corollary~\ref{cor:pol-simplified-scaling-law}, \(R_k\eqsim\cS_{\polyak}(k)+N_{\polyak}(k)\). Corollary~\ref{cor:pol-mixed-signal-absorbed} gives
\[
    \cS_{\polyak}(k)
    \eqsim
    \TT_P(k)
    +
    \left(\frac{(1-\beta)^2}{\eta}\right)^s
    \bigl(1+(1-\beta)k\bigr)^{-s}.
\]
Since \(k\ge C/(1-\beta)\), the second term is comparable to \(((1-\beta)/(\eta k))^s\). The learning-rate and time conditions also imply the assumptions of Corollary~\ref{lem:pol-label-noise-unified}, which supplies the stated label noise term. Combining these estimates proves the theorem.
\end{proof}

\subsection{\texorpdfstring{Scaling law for Nesterov (Proof of Theorem~\ref{thm:nesterov scaling law})}{Scaling law for Nesterov (Proof of Theorem)}}
\label{app:nesterov scaling law}
\subsubsection{Nesterov update}
The Nesterov update is a special case of \eqref{eq:general-second-order-update} with \(\gamma=\beta\), so \(\bm v_k=\btheta_k^{\la}=(1+\beta)\btheta_k-\beta\btheta_{k-1}\). The coefficients in Appendix~\ref{app:subsec:functional-scaling-decomposition} are therefore \(a_j=-(1+\beta)(1-\eta\lambda_j)\) and \(b_j=\beta(1-\eta\lambda_j)\), and the error recursion becomes
\begin{equation}
    e_{k+1}(j)=(1+\beta)(1-\eta\lambda_j)e_k(j)
       -\beta(1-\eta\lambda_j)e_{k-1}(j)-\eta\xi_k^{\la}(j).
    \label{eq:nes_coordinate_recursion_unified}
\end{equation}
Here \(\xi_k^{\la}(j)=[\bxi_k(\btheta_k^{\la})]_j\) is the mini-batch gradient noise at \(\btheta_k^{\la}\), so \(\zeta_k(j)=\xi_k^{\la}(j)\). Let \(\mathcal L_k:=\cE(\btheta_k^{\la})\) be the excess risk at that point. By Lemma~\ref{app:lem:noise-geometry}, the noise satisfies \eqref{eq:unified-coordinate-noise-property} with \(\mathcal J_k=\mathcal L_k\). Theorem~\ref{thm:unified-risk-recursion} gives
\begin{equation}
    R_k\eqsim\cS_{\nesterov}(k)+N_{\nesterov}(k)+M_{\nesterov}(k).
    \label{eq:nes-fsl-application}
\end{equation}
Here
\begin{equation}
    \begin{aligned}
    \cS_{\nesterov}(k)&:=\frac12\sum_{j\ge1}\lambda_j(e_k^{\mathrm{det}}(j))^2,\qquad
     K_{\nesterov}(\tau):=\sum_{j\ge1}\lambda_j^2G_\tau(j)^2,\\
    N_{\nesterov}(k)&:=\frac{\eta^2\sigma^2}{B}\sum_{\tau=0}^{k-1} K_{\nesterov}(\tau),\\
    M_{\nesterov}(k)&:=\frac{\eta^2}{B}\sum_{\tau=0}^{k-1} K_{\nesterov}(\tau)\EE[\mathcal L_{k-1-\tau}],
    \end{aligned}
    \label{eq:nes-signal-kernel-noise}
\end{equation}
where \(e_k^{\mathrm{det}}(j)\) and \(G_k(j)\) are the noise-free error and Green function associated with \eqref{eq:nes_coordinate_recursion_unified}, with the initial conditions specified in Appendix~\ref{app:subsec:functional-scaling-decomposition}. In the following we estimate the signal learning \(\cS_{\nesterov}\), label noise \(N_{\nesterov}\), and multiplicative noise \(M_{\nesterov}\).

\subsubsection{Classification of damping regimes}
These estimates depend on \(e_k^{\mathrm{det}}(j)\) and \(G_k(j)\), which satisfy the same second-order recursion obtained from \eqref{eq:nes_coordinate_recursion_unified} by setting the noise to zero. We use the results in Appendix~\ref{app:second-order-toolkit} to bound these sequences through the roots of the characteristic equation. We first determine the damping threshold. The characteristic polynomial is \(r^2-(1+\beta)(1-\eta\lambda_j)r+\beta(1-\eta\lambda_j)\), with discriminant
\[
    \Delta_j=(1-\eta\lambda_j)\bigl((1+\beta)^2(1-\eta\lambda_j)-4\beta\bigr).
\]
Suppose \(\eta\lambda_1<1\). The damping threshold is \((1-\beta)^2/(1+\beta)^2\eqsim(1-\beta)^2\). Let \(J_{\max}\) be the largest index satisfying \(\eta\lambda_j\ge (1-\beta)^2/(1+\beta)^2\), i.e.,
\[
    J_{\max}:=\max\{j\ge1:\eta\lambda_j\ge (1-\beta)^2/(1+\beta)^2\}.
\]
If no such index exists, we set \(J_{\max}=1\). Since \(\lambda_j\eqsim j^{-q}\),
\begin{equation}
    J_{\max}\eqsim\max\left\{1,\left(\frac{\eta}{(1-\beta)^2}\right)^{1/q}\right\}.
    \label{eq:nes-Jmax-regimes}
\end{equation}
The transition between the two regimes occurs at \(\eta\eqsim(1-\beta)^2\):
\begin{itemize}
    \item \textbf{Fully overdamped regime:} \(\eta\lesssim(1-\beta)^2\). Every coordinate is overdamped.
    \item \textbf{Mixed-damping regime:} \(\eta\gtrsim(1-\beta)^2\). The leading coordinates \(j\le J_{\max}\) are underdamped, while the tail \(j>J_{\max}\) is overdamped.
\end{itemize}
Here \(\eta\lesssim(1-\beta)^2\) means \(\eta\lambda_1\le c_0(1-\beta)^2\) for a sufficiently small constant \(c_0>0\), so that Lemma~\ref{lem:nes-overdamped-coordinate} applies to every coordinate.

\subsubsection{\texorpdfstring{Estimates of signal learning \(\cS_{\nesterov}(k)\)}{Estimates of signal learning S (nesterov)}}
\label{app:nes-signal-learning}

We estimate \(\cS_{\nesterov}(k)\) by summing the weighted squared noise-free coordinate errors. We first treat the fully overdamped regime, then separate the overdamped and underdamped contributions in the mixed-damping regime.

\paragraph{Fully overdamped regime.}
In this regime, all coordinates are overdamped, and the effective learning rate is \(\eta/(1-\beta)\). The following proposition shows that \(\cS_{\nesterov}(k)\) stays of constant order until \(k\eqsim(1-\beta)/\eta\), then follows a power-law decay.

\begin{proposition}[Signal learning for Nesterov in the fully overdamped regime]
\label{thm:nes-deterministic-risk-fully-overdamped}
Under Assumptions~\ref{ass:diagonal-covariance} and~\ref{ass:power-law}, if \(\eta\lesssim(1-\beta)^2\), then
\begin{equation}
\label{eq:nes-det-risk-fully-overdamped}
    \cS_{\nesterov}(k)
    \eqsim
    \begin{cases}
    1,
    & k\lesssim\dfrac{1-\beta}{\eta},
    \\[3mm]
    \left(\dfrac{1-\beta}{\eta k}\right)^s,
    & k\gtrsim\dfrac{1-\beta}{\eta}.
    \end{cases}
\end{equation}
\end{proposition}

\begin{proof}
Since the noise-free recursion is linear and \(e_{-1}^{\mathrm{det}}(j)=e_0^{\mathrm{det}}(j)=-\theta_j^*\), we write \(e_k^{\mathrm{det}}(j)=-h_k(j)\theta_j^*\), where \(h_k(j)\) describes how the initial error evolves in coordinate \(j\). By \eqref{eq:nes_coordinate_recursion_unified}, it satisfies
\[
    h_{k+1}(j)=(1+\beta)(1-\eta\lambda_j)h_k(j)-\beta(1-\eta\lambda_j)h_{k-1}(j),\qquad h_{-1}(j)=h_0(j)=1.
\]
Applying Lemma~\ref{lem:nes-overdamped-coordinate} with \(x=\eta\lambda_j\) gives
\[
    c e^{-C\eta\lambda_j k/(1-\beta)}\le h_k(j)\le C e^{-c\eta\lambda_j k/(1-\beta)},\qquad j\ge1,\quad k\ge0.
\]
Since \(\cS_{\nesterov}(k)=\frac12\sum_j\lambda_j(\theta_j^*)^2h_k(j)^2\) and \(\lambda_j(\theta_j^*)^2\eqsim j^{-1-qs}\), applying Lemma~\ref{lem:power-law-spectral-tail} with \(J=1\) to both exponential bounds yields
\[
    \cS_{\nesterov}(k)\eqsim\left(1+\frac{\eta k}{1-\beta}\right)^{-s}.
\]
This proves the proposition.
\end{proof}

\paragraph{Mixed-damping regime.}

When \(\eta\gtrsim(1-\beta)^2\), we enter the mixed-damping regime. In this regime, the crossover index \(J_{\max}\) separates the underdamped leading coordinates \(j\le J_{\max}\) from the overdamped tail \(j>J_{\max}\). Accordingly, we write \(\cS_{\nesterov}(k)=\cS_{\rm over}(k)+\cS_{\rm under}(k)\), where
\[
\begin{aligned}
    \cS_{\rm over}(k)&:=\frac12\sum_{j>J_{\max}}\lambda_j\bigl(e_k^{\mathrm{det}}(j)\bigr)^2,\\
    \cS_{\rm under}(k)&:=\frac12\sum_{1\le j\le J_{\max}}\lambda_j\bigl(e_k^{\mathrm{det}}(j)\bigr)^2.
\end{aligned}
\]
The next two propositions estimate these contributions separately.

\begin{proposition}[Overdamped signal learning for Nesterov in the mixed-damping regime]
\label{thm:nes-deterministic-bias-mixed}
Under Assumptions~\ref{ass:diagonal-covariance} and~\ref{ass:power-law}, assume \(\eta\gtrsim(1-\beta)^2\). Then the overdamped contribution satisfies
\begin{equation}
\label{eq:nes-mixed-over-tail}
    \cS_{\rm over}(k)
    \eqsim
    \begin{cases}
    \left(\dfrac{(1-\beta)^2}{\eta}\right)^s,
    & k\lesssim(1-\beta)^{-1},
    \\[3mm]
    \left(\dfrac{1-\beta}{\eta k}\right)^s,
    & k\gtrsim(1-\beta)^{-1}.
    \end{cases}
\end{equation}
\end{proposition}

\begin{proof}
Write \(e_k^{\mathrm{det}}(j)=-h_k(j)\theta_j^*\). Lemma~\ref{lem:nes-overdamped-coordinate} applies when \(\eta\lambda_j\le c_0(1-\beta)^2\). To separate these coordinates from those near the damping threshold, set
\[
    J_0:=\min\{j>J_{\max}:\eta\lambda_j\le c_0(1-\beta)^2\}.
\]
Then
\[
    J_0\eqsim J_{\max}\eqsim\left(\frac{\eta}{(1-\beta)^2}\right)^{1/q}.
\]
The coordinate bound and Lemma~\ref{lem:power-law-spectral-tail} give
\[
    \frac12\sum_{j\ge J_0}\lambda_j(\theta_j^*)^2h_k(j)^2
    \eqsim\left(J_0^q+\frac{\eta k}{(1-\beta)}\right)^{-s}
    \eqsim J_{\max}^{-qs}(1+(1-\beta) k)^{-s}.
\]
For \(J_{\max}<j<J_0\), we have \(\eta\lambda_j\eqsim(1-\beta)^2\). The positive gaps \(1-r_{j,\pm}\) have sum \((1-\beta)+(1+\beta)\eta\lambda_j\eqsim(1-\beta)\) and product \(\eta\lambda_j\eqsim(1-\beta)^2\), so both are of order \((1-\beta)\). The roots have positive sum and nonnegative product \(\beta(1-\eta\lambda_j)\), hence \(0\le r_{j,-}\le r_{j,+}\le1-c(1-\beta)\). Equation~\eqref{app:eq:divided-difference-solution} therefore gives
\[
    |h_k(j)|\lesssim(1+(1-\beta) k)e^{-c(1-\beta)k}.
\]
Consequently,
\[
    \frac12\sum_{J_{\max}<j<J_0}\lambda_j(\theta_j^*)^2h_k(j)^2
    \lesssim J_{\max}^{-qs}(1+(1-\beta) k)^2e^{-2c(1-\beta) k}
    \lesssim J_{\max}^{-qs}(1+(1-\beta) k)^{-s}.
\]
Combining the two parts and substituting the order of \(J_{\max}\) yields
\[
    \cS_{\rm over}(k)\eqsim
    \left(\frac{(1-\beta)^2}{\eta}\right)^s(1+(1-\beta) k)^{-s}.
\]
The two time ranges follow by comparing \(k\) with \((1-\beta)^{-1}\). This proves the proposition.
\end{proof}

We next estimate the underdamped contribution. The following proposition gives an exponentially decaying upper bound for this contribution.

\begin{proposition}[Underdamped signal learning for Nesterov in the mixed-damping regime]
\label{thm:nes-underdamped-transient-unified}
Under Assumptions~\ref{ass:diagonal-covariance} and~\ref{ass:power-law}, assume \(\eta\gtrsim(1-\beta)^2\) and \(\eta\lambda_1\le c_{\rm st}\) for sufficiently small \(c_{\rm st}>0\). Then there exists a nonnegative separated-underdamped contribution \(\TT_N(k)\) such that
\begin{equation}
\label{eq:nes-underdamped-bound-unified}
    \TT_N(k)
    \le \cS_{\rm under}(k)
    \lesssim
    \TT_N(k)
    +
    \left(\frac{(1-\beta)^2}{\eta}\right)^s
    \bigl(1+(1-\beta)k\bigr)^{-s},
    \qquad k\ge0,
\end{equation}
and
\begin{equation}
\label{eq:nes-separated-underdamped-bound}
    0\le \TT_N(k)\lesssim\beta^k(1+\eta k)^{-s},
    \qquad k\ge0.
\end{equation}
In particular, for \(k\ge1\),
\begin{equation}
\label{eq:nes-separated-underdamped-envelope}
    0\le \TT_N(k)
    \lesssim
    \min\{1,(\eta k)^{-s}\}\beta^k.
\end{equation}
\end{proposition}

\begin{proof}
Let
\[
    u_c:=\frac{(1-\beta)^2}{(1+\beta)^2},
    \qquad
    p_j:=\beta(1-\eta\lambda_j).
\]
Recall that \(e_k^{\mathrm{det}}(j)=-h_k(j)\theta_j^*\), with \(h_{-1}(j)=h_0(j)=1\). For each underdamped coordinate, the characteristic roots are
\[
    r_{j,\pm}=\sqrt{p_j}e^{\pm i\omega_j},\qquad
    \cos\omega_j=\frac{(1+\beta)\sqrt{1-\eta\lambda_j}}{2\sqrt\beta}.
\]
The initial values give the common representation
\[
    h_k(j)=p_j^{k/2}\bigl[\cos(k\omega_j)+A_j\sin(k\omega_j)\bigr],
    \qquad
    A_j=\frac{(1-\beta)\sqrt{1-\eta\lambda_j}}{2\sqrt\beta\sin\omega_j}.
\]
At the damping threshold \(\eta\lambda_j=u_c\), the expression is understood by continuity. Away from this threshold, \(A_j\) is uniformly bounded, so we can bound the sine and cosine terms separately. Near the threshold, \(\sin\omega_j\) can approach zero and \(A_j\) is not uniformly bounded. We therefore estimate \(A_j\sin(k\omega_j)\) together, using the root formula \eqref{app:eq:divided-difference-solution}, which remains valid when the roots coincide. This distinction motivates the following partition. Fix a constant \(C_{\rm mom}>2\) and set
\[
    \begin{aligned}
    \mathcal I_{\rm cross}&:=\{j\le J_{\max}:\eta\lambda_j\le C_{\rm mom}(1-\beta)^2\},\\
    \mathcal I_{\rm sep}&:=\{j\le J_{\max}:\eta\lambda_j>C_{\rm mom}(1-\beta)^2\}.
    \end{aligned}
\]
Accordingly,
\[
    \cS_{\rm under}(k)=\cS_{\rm cross}(k)+\TT_N(k),
\]
where
\[
    \cS_{\rm cross}(k):=\frac12\sum_{j\in\mathcal I_{\rm cross}}\lambda_j(\theta_j^*)^2h_k(j)^2,
    \qquad
    \TT_N(k):=\frac12\sum_{j\in\mathcal I_{\rm sep}}\lambda_j(\theta_j^*)^2h_k(j)^2.
\]
On \(\mathcal I_{\rm cross}\), \(|r_{j,-}(1-r_{j,+})|^2=p_j\eta\lambda_j\), so \eqref{app:eq:divided-difference-solution} gives
\[
    |h_k(j)|
    \le\bigl(1+k\sqrt{\eta\lambda_j}\bigr)p_j^{k/2}
    \lesssim\bigl(1+(1-\beta)k\bigr)p_j^{k/2}.
\]
Since \(\eta\lambda_j\ge u_c\),
\[
    p_j\le p_{\nesterov},
    \qquad
    p_{\nesterov}:=\beta(1-u_c)=\frac{4\beta^2}{(1+\beta)^2}.
\]
To sum the coordinate bound, note that \(\eta\lambda_j\le C_{\rm mom}(1-\beta)^2\) and \(\lambda_j\eqsim j^{-q}\) imply
\[
    j\gtrsim\left(\frac{\eta}{(1-\beta)^2}\right)^{1/q}
\]
on \(\mathcal I_{\rm cross}\). Thus
\[
    \sum_{j\in\mathcal I_{\rm cross}}\lambda_j(\theta_j^*)^2
    \lesssim
    \sum_{j\gtrsim(\eta/(1-\beta)^2)^{1/q}}j^{-1-qs}
    \lesssim
    \left(\frac{(1-\beta)^2}{\eta}\right)^s.
\]
Substituting the coordinate and weight bounds into the definition of \(\cS_{\rm cross}\) gives
\[
    \cS_{\rm cross}(k)
    \lesssim
    \left(\frac{(1-\beta)^2}{\eta}\right)^s
    \bigl(1+(1-\beta)k\bigr)^2p_{\nesterov}^{\,k}.
\]
Moreover,
\[
    1-p_{\nesterov}
    =\frac{(1-\beta)(1+3\beta)}{(1+\beta)^2}
    \eqsim1-\beta.
\]
Hence \(p_{\nesterov}^{\,k}\le e^{-c(1-\beta)k}\) for an absolute constant \(c>0\). Since
\[
    \sup_{t\ge0}(1+t)^{s+2}e^{-ct}<\infty,
\]
we obtain
\[
    \bigl(1+(1-\beta)k\bigr)^2p_{\nesterov}^{\,k}
    \lesssim
    \bigl(1+(1-\beta)k\bigr)^{-s}.
\]
Therefore,
\begin{equation}
\label{eq:nes-cross-absorbed-by-over}
    \cS_{\rm cross}(k)
    \lesssim
    \left(\frac{(1-\beta)^2}{\eta}\right)^s
    \bigl(1+(1-\beta)k\bigr)^{-s}
    \eqsim
    \cS_{\rm over}(k),
\end{equation}
where the last equivalence follows from Proposition~\ref{thm:nes-deterministic-bias-mixed}.

It remains to control the separated underdamped modes. For \(j\in\mathcal I_{\rm sep}\),
\[
    4\beta\sin^2\omega_j
    =(1+\beta)^2(\eta\lambda_j-u_c),
    \qquad
    A_j^2\lesssim\frac{(1-\beta)^2}{\eta\lambda_j}\lesssim1.
\]
Thus \(h_k(j)^2\lesssim p_j^k=\beta^k(1-\eta\lambda_j)^k\) on \(\mathcal I_{\rm sep}\). Substituting this bound into \(\TT_N(k)\) and using \((1-\eta\lambda_j)^k\le e^{-\eta k\lambda_j}\), we obtain
\[
    0\le \TT_N(k)
    \lesssim
    \beta^k\sum_{j\in\mathcal I_{\rm sep}}j^{-1-qs}e^{-c\eta k j^{-q}}
    \lesssim
    \beta^k(1+\eta k)^{-s},
    \qquad k\ge0,
\]
where the last step is the all-time upper bound \eqref{eq:power-law-spectral-head}. This proves \eqref{eq:nes-separated-underdamped-bound}; for \(k\ge1\), the equivalence
\[
    (1+\eta k)^{-s}\eqsim\min\{1,(\eta k)^{-s}\}
\]
gives \eqref{eq:nes-separated-underdamped-envelope}. Since \(\cS_{\rm under}(k)=\cS_{\rm cross}(k)+\TT_N(k)\), the bound \eqref{eq:nes-cross-absorbed-by-over} gives \eqref{eq:nes-underdamped-bound-unified}.
\end{proof}

\begin{corollary}[Total signal learning for Nesterov in the mixed-damping regime]
\label{cor:nes-mixed-signal-absorbed}
Under Assumptions~\ref{ass:diagonal-covariance} and~\ref{ass:power-law}, assume \(\eta\gtrsim(1-\beta)^2\) and \(\eta\lambda_1\le c_{\rm st}\) for sufficiently small \(c_{\rm st}>0\). Then
\begin{equation}
\label{eq:nes-mixed-signal-absorbed}
    \cS_{\nesterov}(k)
    =\cS_{\rm over}(k)+\cS_{\rm under}(k)
    \eqsim
    \TT_N(k)
    +
    \left(\frac{(1-\beta)^2}{\eta}\right)^s
    \bigl(1+(1-\beta)k\bigr)^{-s},
    \qquad k\ge0.
\end{equation}
\end{corollary}

\begin{proof}
By Proposition~\ref{thm:nes-underdamped-transient-unified},
\[
    \cS_{\rm under}(k)=\TT_N(k)+\cS_{\rm cross}(k),
    \qquad
    \cS_{\rm cross}(k)\lesssim\cS_{\rm over}(k).
\]
Since \(\cS_{\rm over}(k)\), \(\cS_{\rm cross}(k)\), and \(\TT_N(k)\) are nonnegative,
\[
    \cS_{\rm over}(k)+\TT_N(k)
    \le
    \cS_{\nesterov}(k)
    =
    \cS_{\rm over}(k)+\cS_{\rm cross}(k)+\TT_N(k)
    \lesssim
    \cS_{\rm over}(k)+\TT_N(k).
\]
Proposition~\ref{thm:nes-deterministic-bias-mixed} gives
\[
    \cS_{\rm over}(k)
    \eqsim
    \left(\frac{(1-\beta)^2}{\eta}\right)^s
    \bigl(1+(1-\beta)k\bigr)^{-s},
\]
which proves the corollary.
\end{proof}

\subsubsection{Estimates of label noise}
\label{app:nes-forgetting-kernel-properties}
\label{app:nes-forgetting-kernel-unified}

Recall from \eqref{eq:nes-signal-kernel-noise} that
\[
    N_{\nesterov}(k)=\frac{\eta^2\sigma^2}{B}\sum_{\tau=0}^{k-1} K_{\nesterov}(\tau),\qquad
     K_{\nesterov}(\tau)=\sum_{j\ge1}\lambda_j^2\bigl(G_\tau^{\nesterov}(j)\bigr)^2.
\]
Here \(G_\tau^{\nesterov}(j)\) is the Green function associated with \eqref{eq:nes_coordinate_recursion_unified}. The following proposition establishes a property of the forgetting kernel that we use to estimate \(N_{\nesterov}(k)\).

\begin{proposition}[A property of the forgetting kernel for Nesterov]
\label{lem:nes-forgetting-prefix-unified}
Assume \(\eta\lambda_1\le c_{\rm st}\) for sufficiently small universal \(c_{\rm st}>0\). There is a universal constant \(C>0\) such that, whenever
\[
    k\ge C\max\left\{\frac1{1-\beta},\frac{1-\beta}{\eta}\right\},
\]
one has
\begin{equation}
    \sum_{\tau=0}^{k-1} K_{\nesterov}(\tau)
    \eqsim
    \frac{1}{\eta^2}\min\left\{
        \left(\frac{\eta}{1-\beta}\right)^{1/q},
        \frac{\eta}{1-\beta}
    \right\}.
    \label{eq:nes-forgetting-prefix-unified}
\end{equation}
\end{proposition}

\begin{proof}
The learning-rate condition ensures that both characteristic roots have modulus less than one. Applying Lemma~\ref{lem:generic-label-noise-energy} with \(a_j=-(1+\beta)(1-\eta\lambda_j)\) and \(b_j=\beta(1-\eta\lambda_j)\) from \eqref{eq:nes_coordinate_recursion_unified} gives
\begin{equation}
\sum_{\tau=0}^{\infty}
\left(G_\tau^{\nesterov}(j)\right)^2
=
\frac{1+\beta(1-\eta\lambda_j)}
{\left(1-\beta(1-\eta\lambda_j)\right)
\eta\lambda_j
\left[1+(1+2\beta)(1-\eta\lambda_j)\right]}.
\label{eq:nes_green_energy_exact_unified}
\end{equation}
If \(\eta\lambda_1\le c_{\rm st}\) for sufficiently small universal \(c_{\rm st}\), the numerator and final bracket are bounded above and below by positive constants. Also, \(1-\beta(1-\eta\lambda_j)\eqsim1-\beta+\eta\lambda_j\). Hence
\begin{equation}
    \sum_{\tau=0}^{\infty}
    \left(G_\tau^{\nesterov}(j)\right)^2
    \eqsim
    \frac{1}
    {\eta\lambda_j(1-\beta+\eta\lambda_j)}.
    \label{eq:nes_green_energy_unified}
\end{equation}
Write \(\eta_{\beta}=\eta/(1-\beta)\). Then
\begin{align}
    \eta^2\sum_{\tau=0}^{\infty} K_{\nesterov}(\tau)
    &\eqsim
    \eta
    \sum_{j\ge1}
    \frac{\lambda_j}{1-\beta+\eta\lambda_j}
    =
    \sum_{j\ge1}
    \frac{\lambda_j}
    {\lambda_j+\eta_{\beta}^{-1}}.
\label{eq:nes_forgetting_kernel_reduction}
\end{align}
If \(\eta_{\beta}\le1\), then \(\lambda_j\le O(1)\) and \(\frac{\lambda_j}{\lambda_j+\eta_{\beta}^{-1}}\eqsim \eta_{\beta}\lambda_j\), so
\[
    \sum_{j\ge1}
    \frac{\lambda_j}{\lambda_j+\eta_{\beta}^{-1}}
    \eqsim \eta_{\beta}\sum_{j\ge1}\lambda_j
    \eqsim \eta_{\beta}.
\]
If \(\eta_{\beta}\ge1\), take \(J_*:=\lceil\eta_{\beta}^{1/q}\rceil\). The spectral decay gives
\[
    \lambda_{J_*}\eqsim\eta_{\beta}^{-1},\qquad J_*\eqsim\eta_{\beta}^{1/q}.
\]
 For \(j\le J_*\), the summand is bounded above and below by constants, contributing \(\eqsim J_*\). For \(j>J_*\), it is \(\eqsim \eta_{\beta}\lambda_j\), yielding
\[
    \eta_{\beta}\sum_{j>J_*}\lambda_j
    \eqsim
    \eta_{\beta} J_*^{1-q}
    \eqsim
    \eta_{\beta}^{1/q}.
\]
Hence
\begin{equation}
    \eta^2\sum_{\tau=0}^{\infty} K_{\nesterov}(\tau)
    \eqsim
    \min\left\{
        \eta_{\beta}^{1/q},
        \eta_{\beta}
    \right\}.
\label{eq:nes_forgetting_kernel_mass}
\end{equation}
This gives the upper bound in \eqref{eq:nes-forgetting-prefix-unified}. Put \(E_j:=\sum_{\tau\ge0}(G_\tau^{\nesterov}(j))^2\) and \(\nu_j:=\min\{1-\beta,\eta_{\beta}\lambda_j\}\). Write \(d_j:=1-\beta(1-\eta\lambda_j)\), and set \(z_j:=1+\varepsilon\nu_j\) for a sufficiently small absolute constant \(\varepsilon>0\). Since \(\nu_j\le d_j\) and \(d_j\nu_j\le2\eta\lambda_j\), we have \(\sqrt{z_j}\,|r_{j,\pm}|<1\), \(1-b_jz_j\eqsim d_j\), and \((1+b_jz_j)^2-a_j^2z_j\eqsim \eta\lambda_j\). The weighted energy in Lemma~\ref{lem:generic-label-noise-energy} is therefore comparable to \(E_j\). Since \(\log z_j\eqsim\nu_j\), the tail bound in that lemma gives constants \(C,c>0\) such that
\begin{equation}
    \sum_{\tau=k}^{\infty}
    \left(G_\tau^{\nesterov}(j)\right)^2
    \le
    C\exp\!\left[-ck\min\left\{1-\beta,\eta_{\beta}\lambda_j\right\}\right]
    \sum_{\tau=0}^{\infty}
    \left(G_\tau^{\nesterov}(j)\right)^2.
    \label{eq:nes-green-tail-unified}
\end{equation}
The polynomial factor at a repeated root is absorbed by decreasing \(c\).

If \(\eta_\beta\le1\), fix \(J_*\) with \(\sum_{j\le J_*}\lambda_j\ge\frac12\sum_{j\ge1}\lambda_j\). By~\eqref{eq:nes_green_energy_unified}, the block \(A=\{1,\ldots,J_*\}\) carries a fixed fraction of \(\sum_j\lambda_j^2E_j\), and \(\nu_j\gtrsim\min\{1-\beta,\eta_\beta\}\) on \(A\).

If \(\eta_\beta\ge1\), take \(A=\{j:\lambda_j\ge c_*\eta_\beta^{-1}\}\), with fixed sufficiently small \(c_*>0\). Then \(|A|\eqsim\eta_\beta^{1/q}\), \(\lambda_j^2E_j\eqsim\eta^{-2}\) on \(A\), and \(\nu_j\gtrsim1-\beta\). By~\eqref{eq:nes_forgetting_kernel_mass}, this block again carries a fixed fraction of the total mass.

Thus the stated lower bound on \(k\), with a sufficiently large constant, makes the tail in~\eqref{eq:nes-green-tail-unified} at most \(E_j/2\) for \(j\in A\). Consequently,
\[
    \sum_{\tau=0}^{k-1} K_{\nesterov}(\tau)
    \ge\frac12\sum_{j\in A}\lambda_j^2E_j
    \gtrsim\sum_{\tau=0}^{\infty} K_{\nesterov}(\tau),
\]
and~\eqref{eq:nes_forgetting_kernel_mass} gives \eqref{eq:nes-forgetting-prefix-unified}.
\end{proof}

Multiplying the kernel sum by \(\eta^2\sigma^2/B\) gives the label-noise contribution.
\begin{corollary}[Label noise for Nesterov]
\label{lem:nes-label-noise-unified}
Assume \(\eta\lambda_1\le c_{\rm st}\) for sufficiently small universal \(c_{\rm st}\). There is a constant \(C>0\) such that, if
\[
    k\ge C\max\left\{\frac1{1-\beta},\frac{1-\beta}{\eta}\right\},
\]
then
\begin{equation}
    N_{\nesterov}(k)\eqsim \frac{\sigma^2}{B}\min\left\{\left(\frac{\eta}{1-\beta}\right)^{1/q},\frac{\eta}{1-\beta}\right\}.
    \label{eq:nes_label_noise_finite_unified}
\end{equation}
\end{corollary}

\begin{proof}
This follows immediately from the definition of \(N_{\nesterov}(k)\) and Proposition~\ref{lem:nes-forgetting-prefix-unified}.
\end{proof}

\subsubsection{Control of multiplicative noise}
Recall from \eqref{eq:nes-signal-kernel-noise} that
\[
    M_{\nesterov}(k)=\frac{\eta^2}{B}\sum_{\tau=0}^{k-1} K_{\nesterov}(\tau)\EE[\mathcal L_{k-1-\tau}],
\]
where \(\mathcal L_m=\cE(\btheta_m^{\la})\) is the risk at the gradient evaluation point. In this appendix, we control this term and show that
\[
    M_{\nesterov}(k)\lesssim N_{\nesterov}(k).
\] We first use the following lemma to compare \(\EE[\mathcal L_{k-1-\tau}]\) with \(R_{k-\tau}\).
\begin{lemma}[Evaluation risk for Nesterov]
\label{lem:nes_loss_comparison}
Under the small-step condition \(\eta\lambda_1\le c_{\mathrm{st}}<1\), there exist constants \(c_{\la},C_{\la}>0\), independent of \(k,\eta,\beta,B\), such that
\begin{equation}
\label{eq:nes-lookahead-upper}
    c_{\la}\EE[\cE(\btheta_{k+1})]
    -C_{\la}\frac{\eta^2\sigma^2}{B}
    \le
    \EE[\mathcal L_k]
    \le C_{\la}\EE[\cE(\btheta_{k+1})].
\end{equation}
\end{lemma}

\begin{proof}
The stochastic Nesterov recursion at the look-ahead point is
\[
    \be_{k+1}=(\bI-\eta\bH)\be_k^{\la}
    -\eta\bxi_k(\btheta_k^{\la}).
\]
Conditioning on \(\mathscr F_k\), the cross term vanishes, so
\begin{align}
\EE[\cE(\btheta_{k+1})]
={}&
\frac12\EE\!\left[\|(\bI-\eta\bH)\be_k^{\la}\|_{\bH}^2\right]
+\frac{\eta^2}{2}\EE\|\bxi_k(\btheta_k^{\la})\|_{\bH}^2.
\end{align}
Because \(1-c_{\mathrm{st}}\le1-\eta\lambda_j\le1\), the first term is comparable to \(\EE[\mathcal L_k]\).  Moreover, Lemma~\ref{app:lem:noise-geometry} gives
\[
    \EE\|\bxi_k(\btheta_k^{\la})\|_{\bH}^2
    \lesssim
    \frac{\Tr(\bH^2)}{B}
    \left(\EE[\mathcal L_k]+\sigma^2\right).
\]
The assumptions give \(\eta^2\Tr(\bH^2)/B\le c_{\mathrm{st}}^2\Tr(\bH^2)/\lambda_1^2\), so this coefficient is uniformly bounded. Hence there are constants \(c_{\rm cmp},C_0,C_1>0\) such that
\[
    c_{\rm cmp}\EE[\mathcal L_k]
    \le \EE[\cE(\btheta_{k+1})]
    \le C_0\EE[\mathcal L_k]
       +C_1\frac{\eta^2\sigma^2}{B}.
\]
Rearranging these two inequalities gives~\eqref{eq:nes-lookahead-upper}.
\end{proof}

This comparison allows us to apply Lemma~\ref{lem:convolution-bound} to the risk recursion. Together with the kernel estimate established above, it yields a uniform risk bound and hence the following control of multiplicative noise.

\begin{proposition}[Control of multiplicative noise for Nesterov]
\label{thm:nes-uniform-risk-boundedness}
\label{thm:nes-noise-absorption}
Suppose \(\sigma>0\) and \(\eta\le c\min\{1,B^q(1-\beta)\}\) for sufficiently small \(c>0\). There is a constant \(M<\infty\), independent of \(k,\eta,\beta,B\), such that
\[
    \sup_{k\ge0}\EE[\cE(\btheta_k)]\le M.
\]
Moreover,
\[
    M_{\nesterov}(k)\lesssim N_{\nesterov}(k).
\]
\end{proposition}
\begin{proof}
The signal learning bounds give \(\sup_k\cS_{\nesterov}(k)\le S_*\). By \eqref{eq:nes_forgetting_kernel_mass},
\[
    \kappa:=\frac{\eta^2}{B}\sum_{\tau\ge0} K_{\nesterov}(\tau)\lesssim\frac1B\min\left\{
       \left(\frac{\eta}{1-\beta}\right)^{1/q},\frac{\eta}{1-\beta}\right\}
    \le\frac1B\left(\frac{\eta}{1-\beta}\right)^{1/q}\le c^{1/q}.
\]
Lemma~\ref{lem:nes_loss_comparison} gives \(\EE[\mathcal L_m]\le C_{\la}R_{m+1}\). Let \(A\ge1\) satisfy
\[
    R_k\le A\bigl(\cS_{\nesterov}(k)+N_{\nesterov}(k)+M_{\nesterov}(k)\bigr).
\]
 Set \(x_k=R_k\), \(y_k=\EE[\mathcal L_k]\), \(u_k=A\cS_{\nesterov}(k)\), and \(w_\tau=A\eta^2 K_{\nesterov}(\tau)/B\). Take \(v=\sigma^2\), \(L=C_{\la}\), and \(\iota=1\). The risk inequality above has the form required by Lemma~\ref{lem:convolution-bound}, with \(y_m\le Lx_{m+\iota}\) and \(U:=\sup_k u_k\le AS_*\). By these definitions,
\[
    L\sum_{\tau\ge0}w_\tau=AC_{\la}\frac{\eta^2}{B}\sum_{\tau\ge0} K_{\nesterov}(\tau)=AC_{\la}\kappa.
\]
Since \(\kappa\lesssim c^{1/q}\), choosing \(c\) sufficiently small ensures \(AC_{\la}\kappa\le1/2\), as required by the lemma. Applying Lemma~\ref{lem:convolution-bound} now gives
\[
    \sup_k R_k=\sup_k x_k\le2U+\frac{v}{L}\le 2AS_*+\frac{\sigma^2}{C_{\la}}=:M.
\]
 Therefore,
\[
    M_{\nesterov}(k)\le\frac{\eta^2C_{\la}M}{B}\sum_{\tau=0}^{k-1} K_{\nesterov}(\tau)=\frac{C_{\la}M}{\sigma^2}N_{\nesterov}(k).
\]
\end{proof}

The multiplicative-noise bound gives a simplified form of the scaling law.
\begin{corollary}[Scaling law of simplified form for Nesterov]
\label{cor:nes-simplified-scaling-law}
Under the assumptions of Proposition~\ref{thm:nes-noise-absorption},
\begin{equation}
    \EE[\cE(\btheta_k)]\eqsim\cS_{\nesterov}(k)+N_{\nesterov}(k).
    \label{eq:nes-risk-after-absorption}
\end{equation}
\end{corollary}
\begin{proof}
By Proposition~\ref{thm:nes-noise-absorption}, \(M_{\nesterov}(k)\lesssim N_{\nesterov}(k)\). Substituting this bound into \eqref{eq:nes-fsl-application} and using nonnegativity proves the claim.
\end{proof}

\subsubsection{Completion of the proof}
We now combine the signal learning and noise estimates to establish the following scaling laws in the two damping regimes. The constant \(c>0\) is sufficiently small to ensure \(\eta\lambda_1\le c_{\rm st}\), and \(C>0\) is sufficiently large to satisfy the time conditions of the preceding estimates.

\paragraph{Fully overdamped regime.}

Here the overdamped signal learning estimate combines directly with the label noise bound, since the multiplicative noise is controlled by the latter.

\begin{theorem}[Scaling law for Nesterov in the fully overdamped regime]
\label{thm:nes_scaling_sgd}
Under Assumptions~\ref{ass:diagonal-covariance} and~\ref{ass:power-law}, suppose \(\sigma>0\), assume the Nesterov stability condition \(\eta\le c\min\{1,B^q(1-\beta)\}\), and consider the fully overdamped regime \(\eta\lesssim(1-\beta)^2\). For \(k\) satisfying
\[
    k\ge C\frac{1-\beta}{\eta},
\]
the expected excess risk satisfies
\[
    \EE[\cE(\btheta_k)]
    \eqsim
    \left(\frac{1-\beta}{\eta k}\right)^s
    +
    \frac{\sigma^2\eta}{B(1-\beta)}.
\]
\end{theorem}

\begin{proof}
By Corollary~\ref{cor:nes-simplified-scaling-law}, \(R_k\eqsim\cS_{\nesterov}(k)+N_{\nesterov}(k)\). Proposition~\ref{thm:nes-deterministic-risk-fully-overdamped} gives \(\cS_{\nesterov}(k)\eqsim(1+\eta k/(1-\beta))^{-s}\). Since \(k\ge C(1-\beta)/\eta\), this is comparable to \(((1-\beta)/(\eta k))^s\). The stated learning-rate and time conditions also allow us to apply Corollary~\ref{lem:nes-label-noise-unified}. The fully overdamped condition and \(\lambda_1\eqsim1\) imply \(\eta/(1-\beta)\lesssim1\), so the minimum in that corollary is comparable to \(\eta/(1-\beta)\). Thus \(N_{\nesterov}(k)\eqsim\sigma^2\eta/[B(1-\beta)]\). Substituting these estimates proves the theorem.
\end{proof}

\paragraph{Mixed-damping regime.}

Here we retain the underdamped transient and combine the overdamped tail estimate with the label noise bound.

\begin{theorem}[Scaling law for Nesterov in the mixed-damping regime]
\label{thm:nes_two_sided_scaling}
Under Assumptions~\ref{ass:diagonal-covariance} and~\ref{ass:power-law}, suppose \(\sigma>0\), assume the Nesterov stability condition \(\eta\le c\min\{1,B^q(1-\beta)\}\), and consider the mixed-damping regime \(\eta\gtrsim(1-\beta)^2\). For \(k\) satisfying
\[
    k\ge\frac{C}{1-\beta},
\]
the expected excess risk satisfies
\begin{equation}
\label{eq:nes_mixed_scaling_law}
\begin{aligned}
    \EE[\cE(\btheta_k)]
    \eqsim{}
    \TT_N(k)
    +
    \left(\frac{1-\beta}{\eta k}\right)^s
    +
    \frac{\sigma^2}{B}
    \min\left\{
        \left(\frac{\eta}{1-\beta}\right)^{1/q},
        \frac{\eta}{1-\beta}
    \right\}.
\end{aligned}
\end{equation}
Here \(\TT_N(k)\) is the separated underdamped transient defined in Proposition~\ref{thm:nes-underdamped-transient-unified} and satisfies
\[
    0\le \TT_N(k)\lesssim\min\{1,(\eta k)^{-s}\}\beta^k.
\]
\end{theorem}

\begin{proof}
By Corollary~\ref{cor:nes-simplified-scaling-law}, \(R_k\eqsim\cS_{\nesterov}(k)+N_{\nesterov}(k)\). Corollary~\ref{cor:nes-mixed-signal-absorbed} gives
\[
    \cS_{\nesterov}(k)
    \eqsim
    \TT_N(k)
    +
    \left(\frac{(1-\beta)^2}{\eta}\right)^s
    \bigl(1+(1-\beta)k\bigr)^{-s}.
\]
Since \(k\ge C/(1-\beta)\), the second term is comparable to \(((1-\beta)/(\eta k))^s\). The learning-rate and time conditions also imply the assumptions of Corollary~\ref{lem:nes-label-noise-unified}, which supplies the stated label noise term. Combining these estimates proves the theorem.
\end{proof}

\subsection{Scaling law for SGD}
\label{app:sgd_scaling_law}

We recall the SGD scaling law, which is the special case \(M=\infty\) of \citet[Theorem 5.2]{li2026functional}. Below, we give an alternative proof by specializing the Polyak result to \(\beta=0\). It provides the baseline for comparing the signal learning rate and noise floor of the two momentum methods.

\begin{theorem}[SGD scaling law \citep{li2026functional}]
\label{thm:sgd-scaling-law-unified}
Under Assumptions~\ref{ass:diagonal-covariance} and~\ref{ass:power-law}, if the stability condition $\eta\lesssim1$ holds, then for $k\gtrsim1/\eta$,
\begin{equation}
\label{eq:sgd-loss-unified}
    \EE[\cE(\btheta_k)]
    \eqsim
    (\eta k)^{-s}+\frac{\eta\sigma^2}{B}.
\end{equation}
\end{theorem}

\begin{proof}
Since Polyak momentum reduces to SGD when \(\beta=0\), the result follows by setting \(\beta=0\) in Theorem~\ref{thm:pol_scaling_sgd}.
\end{proof}

\section{Optimal data scaling across batch-size regimes}
\label{app:sec:data-scaling}
This appendix proves the batch-size-dependent scaling laws stated in Theorem~\ref{thm:batch-size-scaling}. Throughout this appendix, we work in the noisy setting $\sigma>0$ considered in Section~\ref{sec:data-scaling} and impose the data-budget constraint
\[
B\eqsim D^b,\qquad N=\frac{D}{B}\eqsim D^{1-b},\qquad 0\le b<1.
\]
Thus, the batch exponent $b$ is prescribed, and only the remaining hyperparameters are optimized. We use the transition exponents defined in Section~\ref{sec:data-scaling}:
\[
b_1=\frac{s}{s+1},\qquad b_2=\frac{2s+1/q}{2s+1+1/q}=\frac{2sq+1}{2sq+q+1},\qquad b_3=\frac{2s+1}{2s+2}.
\]
For the momentum methods, we parameterize the momentum gap and learning rate by $1-\beta\eqsim D^{-m}$ and $\eta\eqsim D^{-c}$, where $m,c\ge0$.

\subsection{Data scaling for SGD}
For SGD, write $\eta\eqsim D^{-c}$, where $c\ge0$. By the SGD risk scaling law in Section~\ref{sec:scaling law},
\[
\EE[\cE(\btheta_N)]\eqsim(\eta N)^{-s}+\frac{\sigma^2\eta}{B}\eqsim D^{-s(1-b-c)}+D^{-(b+c)}.
\]
Therefore, for a fixed batch exponent $b$, the achievable decay exponent is $\min\{s(1-b-c),b+c\}$. If $0\le b<b_1$, the two terms can be balanced by taking $c=b_1-b\ge0$. Since $b_1=s/(s+1)$, this gives
\[
s(1-b-c)=b+c=b_1=\frac{s}{s+1}.
\]
Hence $r_{\sgd}(b)=s/(s+1)$ for $0\le b<b_1$.

If $b\ge b_1$, the unconstrained balancing choice would require $c=b_1-b\le0$. Because the learning rate is not allowed to grow polynomially with $D$, we must have $c\ge0$. Moreover, $b+c\ge b\ge b_1$, so the signal learning term is the slower-decaying term. It is optimized by taking $c=0$, which yields $r_{\sgd}(b)=s(1-b)$ for $b_1\le b<1$.

Consequently,
\[
r_{\sgd}(b)=
\begin{cases}
\dfrac{s}{s+1},&0\le b<b_1,\\[4pt]
s(1-b),&b_1\le b<1.
\end{cases}
\]
Thus, $b_1$ is precisely the largest batch exponent for which SGD preserves the data-optimal rate $s/(s+1)$.

\subsection{Data scaling for Polyak}
We next prove the Polyak part of Theorem~\ref{thm:batch-size-scaling}. We retain $N\eqsim D^{1-b}$, $1-\beta\eqsim D^{-m}$, and $\eta\eqsim D^{-c}$, where $m,c\ge0$. We require $b+m<1$. Proposition~\ref{thm:pol-underdamped-transient-unified} gives
\[
\TT_P(N)
\lesssim
\beta^N
\leq
\exp\bigl(-(1-\beta)N\bigr)
=
\exp\bigl(-\Theta(D^{1-b-m})\bigr),
\]
which is smaller than any polynomial power of $D$. The stability condition $\eta\lesssim B(1-\beta)$ becomes $b+c-m\ge0$. Moreover,
\[
\eta_{\beta} N=\frac{\eta N}{1-\beta}\eqsim D^{1-b-c+m},\qquad \frac{\eta_{\beta}}{B}=\frac{\eta}{B(1-\beta)}\eqsim D^{-(b+c-m)}.
\]
The damping threshold $\eta\eqsim(1-\beta)^2$ corresponds to $c=2m$. We consider the fully overdamped and mixed-damping regimes separately.

\paragraph{Fully overdamped regime.} If $c\ge2m$, Theorem~\ref{thm:hbm scaling law} gives
\[
\EE[\cE(\btheta_N)]\eqsim(\eta_{\beta} N)^{-s}+\frac{\sigma^2\eta_{\beta}}{B}\eqsim D^{-s(1-b-c+m)}+D^{-(b+c-m)}.
\]
Therefore, the polynomial decay exponent in the fully overdamped regime is $\min\{s(1-b-c+m),b+c-m\}$.

\paragraph{Mixed-damping regime.} If $c\le2m$, Theorem~\ref{thm:hbm scaling law} gives
\[
\EE[\cE(\btheta_N)]\eqsim \TT_P(N)+(\eta_{\beta} N)^{-s}+\frac{\sigma^2\eta_{\beta}}{B}.
\]
Since $b+m<1$, Proposition~\ref{thm:pol-underdamped-transient-unified} implies
\[
\TT_P(N)
\lesssim
\beta^N
\leq
\exp\bigl(-(1-\beta)N\bigr)
=
\exp\bigl(-\Theta(D^{1-b-m})\bigr),
\]
which is smaller than any polynomial power of $D$. Hence
\[
\EE[\cE(\btheta_N)]\eqsim D^{-s(1-b-c+m)}+D^{-(b+c-m)},
\]
and the polynomial decay exponent is again $\min\{s(1-b-c+m),b+c-m\}$. Thus, although the two regimes have different transient dynamics, they lead to the same polynomial optimization problem. The signal learning and additive-noise powers are balanced when
\[
s(1-b-c+m)=b+c-m
\quad\Longleftrightarrow\quad
b+c-m=\frac{s}{s+1}=b_1.
\]
If $0\le b<b_1$, this balance is attained by taking $m=0$ and $c=b_1-b$. Since $c\ge2m$, this choice lies in the fully overdamped regime and yields $r_{\polyak}(b)=s/(s+1)$.

If $b_1\le b<b_3$, the same balance is attained by taking $c=0$ and $m=b-b_1$. For $b>b_1$, this choice satisfies $c\le2m$ and therefore lies in the mixed-damping regime. Moreover, $b+c-m=b_1$, while the transient condition becomes $b+m=2b-b_1<1$. Since $(1+b_1)/2=(2s+1)/(2s+2)=b_3$, this condition holds precisely when $b<b_3$. Hence $r_{\polyak}(b)=s/(s+1)$ for $0\le b<b_3$.

Now suppose that $b\ge b_3$. Since $c\ge0$ and $m<1-b$, every admissible choice satisfies $b+c-m>2b-1\ge b_1$, so the signal learning term is the slower-decaying term and $s(1-b-c+m)<2s(1-b)$. This upper bound is approached by taking $c=0$ and $m\uparrow1-b$, which lies in the mixed-damping regime and gives
\[
s(1-b-c+m)\uparrow2s(1-b),\qquad b+c-m\downarrow2b-1.
\]
Moreover, $2b-1\ge2s(1-b)$ holds exactly when $b\ge b_3$, so the signal learning term is indeed dominant. Consequently, $r_{\polyak}(b)=2s(1-b)$ for $b_3\le b<1$, where the endpoint value is understood in the supremal sense represented by the $o(1)$ term in Definition~\ref{def:batch-scaling-exponent}. Combining the two regimes gives
\[
r_{\polyak}(b)=
\begin{cases}
\dfrac{s}{s+1},&0\le b<b_3,\\[4pt]
2s(1-b),&b_3\le b<1.
\end{cases}
\]
Thus, Polyak preserves the optimal data-scaling exponent $s/(s+1)$ up to the critical batch exponent $b_3$. Compared with SGD, it extends the critical batch size from $D^{b_1}$ to $D^{b_3}$, while leaving the optimal small-batch exponent unchanged.

\subsection{Data scaling for Nesterov}
We finally prove the Nesterov part of Theorem~\ref{thm:batch-size-scaling}. We retain $B\eqsim D^b$, $N\eqsim D^{1-b}$, $1-\beta\eqsim D^{-m}$, and $\eta\eqsim D^{-c}$, where $m,c\ge0$. We require $b+m<1$. Proposition~\ref{thm:nes-underdamped-transient-unified} gives
\[
\TT_N(N)
\lesssim
\min\{1,(\eta N)^{-s}\}\beta^N
\leq
\beta^N
\leq
\exp\bigl(-(1-\beta)N\bigr)
=
\exp\bigl(-\Theta(D^{1-b-m})\bigr),
\]
which is smaller than any polynomial power of $D$. The stability condition $\eta\lesssim B^q(1-\beta)$ becomes $q b+c-m\ge0$, while the damping threshold $\eta\eqsim(1-\beta)^2$ corresponds to $c=2m$. We optimize the fully overdamped and mixed-damping regimes separately and then compare their optimal decay exponents.

\paragraph{Fully overdamped regime.} If $c\ge2m$, Theorem~\ref{thm:nesterov scaling law} gives
\[
\EE[\cE(\btheta_N)]\eqsim(\eta_{\beta} N)^{-s}+\frac{\sigma^2\eta_{\beta}}{B}\eqsim D^{-s(1-b-c+m)}+D^{-(b+c-m)}.
\]
Thus, the decay exponent is $\min\{s(1-b-c+m),b+c-m\}$. Since $c\ge2m$ implies $c-m\ge0$, we have $b+c-m\ge b$. If $0\le b<b_1$, the signal learning and additive-noise powers can be balanced by taking $m=0$ and $c=b_1-b$, which gives $b+c-m=b_1$ and decay exponent $s/(s+1)$. If $b\ge b_1$, then $b+c-m\ge b\ge b_1$, so the signal learning term is limiting. Its exponent is maximized by taking $m=c=0$, which gives $s(1-b)$. Therefore, the optimal exponent in the fully overdamped regime is
\[
r_{\nesterov}^{\rm full}(b)=
\begin{cases}
\dfrac{s}{s+1},&0\le b<b_1,\\[4pt]
s(1-b),&b_1\le b<1.
\end{cases}
\]
In particular, the fully overdamped Nesterov dynamics have the same batch-size-dependent scaling exponent as SGD.

\paragraph{Mixed-damping regime.} If $c\le2m$, Theorem~\ref{thm:nesterov scaling law} gives
\[
\EE[\cE(\btheta_N)]
\eqsim
\TT_N(N)
+(\eta_{\beta} N)^{-s}
+\frac{\sigma^2}{B}\min\left\{\eta_{\beta},\eta_{\beta}^{1/q}\right\}.
\]
Since $b+m<1$, Proposition~\ref{thm:nes-underdamped-transient-unified} implies
\[
\TT_N(N)
\lesssim
\min\{1,(\eta N)^{-s}\}\beta^N
\leq
\beta^N
\leq
\exp\bigl(-(1-\beta)N\bigr)
=
\exp\bigl(-\Theta(D^{1-b-m})\bigr),
\]
which is smaller than any polynomial power of $D$. Hence
\[
\EE[\cE(\btheta_N)]\eqsim D^{-s(1-b-c+m)}+\frac{\sigma^2}{B}\min\left\{\frac{\eta}{1-\beta},\left(\frac{\eta}{1-\beta}\right)^{1/q}\right\}.
\]
The signal learning term has decay power $s(1-b-c+m)$. Since $\eta/(1-\beta)\eqsim D^{m-c}$, the additive-noise decay power is
\[
\begin{cases}
b+c-m,&c\ge m,\\[4pt]
b+\dfrac{c-m}{q},&c\le m.
\end{cases}
\]
We optimize these two branches separately.

\paragraph{The branch $c\ge m$.} In this branch, the decay exponent is
\[
\min\{s(1-b-c+m),b+c-m\}.
\]
Because $c-m\ge0$, the same argument as in the fully overdamped regime applies. For $0\le b<b_1$, the balance $b+c-m=b_1$ is feasible in the mixed-damping regime. For example, one may take $m=b_1-b$ and $c=2(b_1-b)$, for which $b+m=b_1<1$. For $b\ge b_1$, the optimal choice satisfies $c-m=0$ and gives exponent $s(1-b)$. Therefore,
\[
r_{\nesterov}^{\rm mixed,\,c\ge m}(b)=
\begin{cases}
\dfrac{s}{s+1},&0\le b<b_1,\\[4pt]
s(1-b),&b_1\le b<1.
\end{cases}
\]

\paragraph{The branch $c\le m$.} In this branch, the decay exponent is
\[
\min\left\{s(1-b-c+m),\,b+\frac{c-m}{q}\right\}.
\]
If $0\le b<b_1$, then at $m-c=0$ we have $b<s(1-b)$. Increasing $m-c$ increases the signal learning power but decreases the additive-noise power, so the optimal choice is $m-c=0$, which gives decay exponent $b$. If $b\ge b_1$, balancing the signal learning and additive-noise powers gives
\[
s(1-b-c+m)=b+\frac{c-m}{q},
\qquad
m-c=\frac{q\bigl((s+1)b-s\bigr)}{sq+1}.
\]
For a fixed value of $m-c$, the least restrictive choice for the transient condition is $c=0$, so we take
\[
c=0,\qquad m=\frac{q\bigl((s+1)b-s\bigr)}{sq+1}.
\]
This choice satisfies the stability condition because
\[
q b+c-m=\frac{q s\bigl(1+(q-1)b\bigr)}{sq+1}>0.
\]
The signal learning and additive-noise powers are then equal to
\[
s(1-b+m)=b-\frac{m}{q}=\frac{s\bigl(1+(q-1)b\bigr)}{sq+1}.
\]
The transient condition holds precisely when
\[
b+m<1
\quad\Longleftrightarrow\quad
b<\frac{2sq+1}{2sq+q+1}
=\frac{2s+1/q}{2s+1+1/q}
=b_2.
\]
Therefore, for $b_1\le b<b_2$, the optimal exponent on the branch $c\le m$ is
\[
\frac{s\bigl(1+(q-1)b\bigr)}{sq+1}.
\]
If $b\ge b_2$, the preceding balance is no longer compatible with $b+m<1$. Since $c\ge0$ and $m<1-b$, every admissible choice satisfies
\[
s(1-b-c+m)<2s(1-b).
\]
This upper bound is approached by taking $c=0$ and $m\uparrow1-b$. Under this choice,
\[
s(1-b+m)\uparrow2s(1-b),\qquad b-\frac{m}{q}\to\frac{(q+1)b-1}{q}.
\]
At $b=b_2$, the two limiting powers agree, while for $b>b_2$,
\[
\frac{(q+1)b-1}{q}>2s(1-b).
\]
Thus, the signal learning term is limiting and the supremal decay exponent is $2s(1-b)$. Consequently,
\[
r_{\nesterov}^{\rm mixed,\,c\le m}(b)=
\begin{cases}
b,&0\le b<b_1,\\[4pt]
\dfrac{s\bigl(1+(q-1)b\bigr)}{sq+1},&b_1\le b<b_2,\\[8pt]
2s(1-b),&b_2\le b<1.
\end{cases}
\]
Taking the larger exponent between the two additive-noise branches gives the optimal exponent in the mixed-damping regime:
\[
r_{\nesterov}^{\rm mixed}(b)=
\begin{cases}
\dfrac{s}{s+1},&0\le b<b_1,\\[6pt]
\dfrac{s\bigl(1+(q-1)b\bigr)}{sq+1},&b_1\le b<b_2,\\[8pt]
2s(1-b),&b_2\le b<1.
\end{cases}
\]
Indeed, the branch $c\ge m$ is optimal for $b<b_1$, whereas the branch $c\le m$ is optimal for $b>b_1$. Finally, optimizing over the fully overdamped and mixed-damping regimes gives
\[
r_{\nesterov}(b)=\max\left\{r_{\nesterov}^{\rm full}(b),r_{\nesterov}^{\rm mixed}(b)\right\}.
\]
For $0\le b<b_1$, the two regimes attain the same exponent $s/(s+1)$. For $b_1<b<b_2$, the mixed-damping regime exponent $s(1+(q-1)b)/(sq+1)$ is strictly larger than the fully overdamped exponent $s(1-b)$. For $b\ge b_2$, the mixed-damping regime exponent $2s(1-b)$ is also larger than the fully overdamped exponent $s(1-b)$. We therefore obtain
\[
r_{\nesterov}(b)=
\begin{cases}
\dfrac{s}{s+1},&0\le b<b_1,\\[6pt]
\dfrac{s\bigl(1+(q-1)b\bigr)}{sq+1},&b_1\le b<b_2,\\[8pt]
2s(1-b),&b_2\le b<1.
\end{cases}
\]
Thus, Nesterov agrees with SGD and Polyak in the small-batch regime, improves its data-scaling exponent once $b$ exceeds $b_1$, and reaches its maximal exponent at $b=b_2$. For $b\ge b_2$, its exponent becomes $2s(1-b)$, which coincides with the Polyak exponent in the ultra-large-batch regime.

\section{Experimental details and additional results}
\label{app:experiment}

This appendix provides experimental details and additional results for Section~\ref{sec:experiment}. In Section~\ref{app:plk-experiments}, we first present the PLK experiments corresponding to
Section~\ref{subsec:plk}, including the stability boundaries, risk dynamics, and
batch-size scaling results. We then turn to neural-network experiments: we provide the full setup for the ImageNet-1K experiment in
Section~\ref{app:neural-experiments} and report an additional CIFAR-10 experiment that examines whether the predicted stability behavior also appears in neural-network
training.
\subsection{Power-law kernel regression}
\label{app:plk-experiments}

Our theoretical analysis is formulated in the infinite-dimensional PLK setting. For numerical experiments, we truncate the system to the first \(d=4096\) eigendirections, so that, for example, \(\bH=\diag(\lambda_1,\ldots,\lambda_d)\in\RR^{d\times d}\). Unless otherwise stated, all experiments use \(d=4096\).

\subsubsection{Risk-stability boundaries}
\label{app:stability-experiment}
\label{app:stability-kernel}
\paragraph{Simulation setup.} We run with \(\btheta_{-1}=\btheta_0=0\) for \(N=2000\) iterations. A run is classified as stable if its average risk over the final \(20\) iterations is below \(1\). It is classified as unstable if the risk blows up or exceeds \(10^{12}\). For each setting and random seed, the empirical stability boundary is the largest stable learning rate on the search grid. We repeat each scan over five independent seeds and report the mean boundary.

\begin{itemize}
\item \textbf{Stability versus batch size.} For Figure~\ref{fig:numerical-summary} (left), we use \((s,q)=(0.3,1.5)\), fix \(1-\beta=2^{-6}\) for the momentum methods, and scan \(B\in \text{round}\{2^a\},a\in \{2,2.5,\ldots,7\}\) and \(50\) logarithmically spaced learning rates between \(2^{-10}\) and \(8\). The dashed curves are post-hoc fits to the predicted scalings: a constant for SGD, \(B(1-\beta)\) for Polyak, and \(B^q(1-\beta)\) followed by a constant plateau for Nesterov. The dotted line marks the split used for the Nesterov fit.

\item \textbf{Stability versus momentum damping.} For Figure~\ref{fig:numerical-summary} (middle), we use \((s,q)=(0.3,1.5)\), \(\sigma=0\), and \(B=4\), vary \(1-\beta\in \{2^{-8},2^{-7.5},2^{-7},\ldots,2^{-3.5}\}\), and scan \(50\) logarithmically spaced learning rates between \(2^{-10}\) and \(8\). Since SGD is independent of \(\beta\), its boundary is constant across the damping grid. The dashed curves are post-hoc fits to the predicted form \(\eta^\crit=\min\{L,c(1-\beta)\}\) for Polyak and Nesterov, and to a constant for SGD.
\end{itemize}

\paragraph{Additional results.} Figure~\ref{fig:additional-stability-boundaries} shows additional results for different choices of \((s,q)\).

\begin{figure}[H]
    \centering
    \captionsetup[subfigure]{skip=0pt, font=small}
    \begin{subfigure}[t]{0.325\textwidth}
        \centering
        \includegraphics[width=\linewidth]{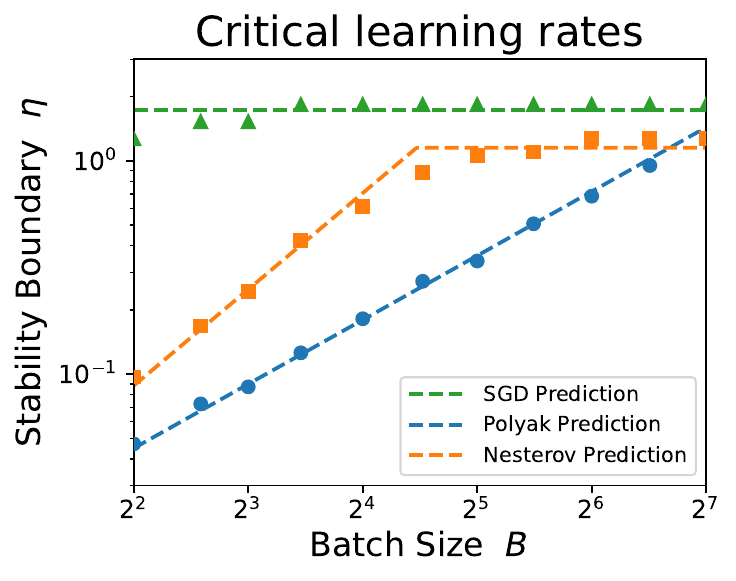}
        \caption{\(s=0.5,q=1.5\)}
    \end{subfigure}
    \begin{subfigure}[t]{0.325\textwidth}
        \centering
        \includegraphics[width=\linewidth]{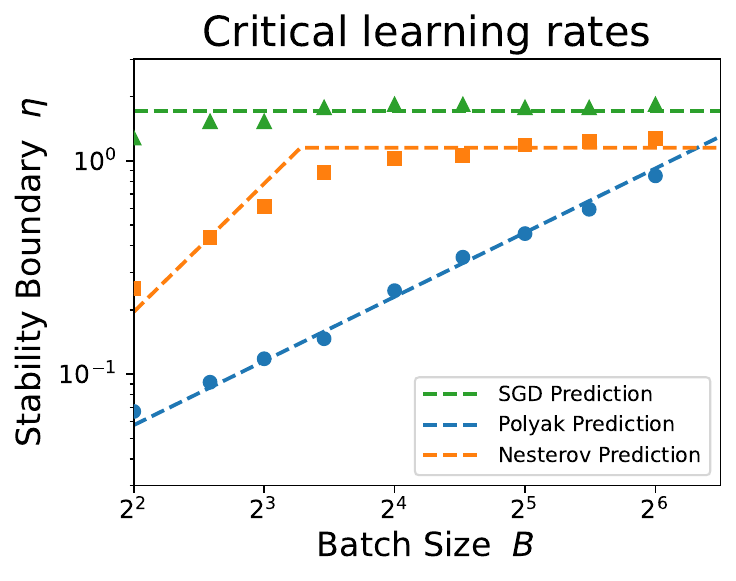}
        \caption{\(s=0.2,q=2\)}
    \end{subfigure}
    \begin{subfigure}[t]{0.32\textwidth}
        \centering
        \includegraphics[width=\linewidth]{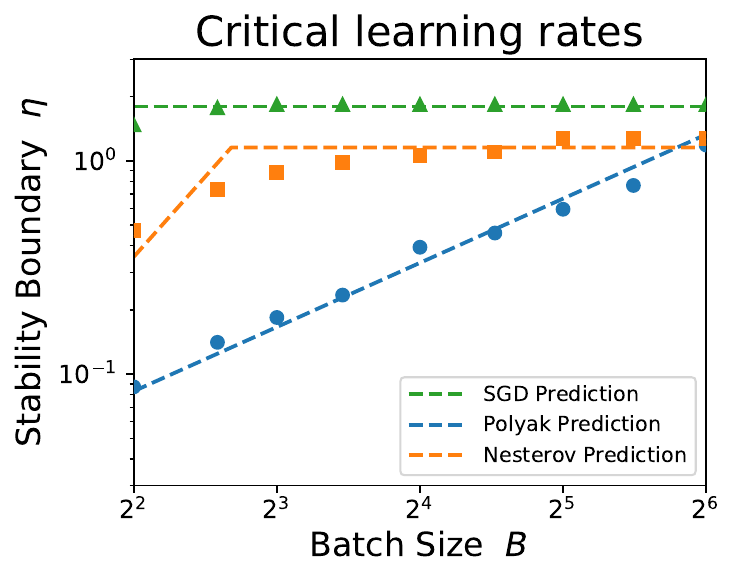}
        \caption{\(s=0.5,q=2.5\)}
    \end{subfigure}

\caption{Additional stability boundaries across \((s,q)\).}
    \label{fig:additional-stability-boundaries}
\end{figure}

\subsubsection{Risk-dynamics scaling laws}
\label{app:scaling-law-experiments}

\paragraph{Simulation setup.} We simulate PLK regression initialized with \(\btheta_{-1}=\btheta_0=0\). For Figure~\ref{fig:numerical-summary} (right), we set \((s,q)=(0.3,4)\), \(\sigma=1\), \(B=32\), \(\eta=0.03\), and \(\beta=0.98\). Each method is run for \(2\times10^7\) iterations. Let \(\overline{\cE}(k)\) denote the population excess risk averaged over the \(100\) runs.

\paragraph{Scaling-law fits.} Theorems~\ref{thm:hbm scaling law} and~\ref{thm:nesterov scaling law} determine the functional forms of the risk dynamics up to multiplicative constants. We therefore introduce \(\Theta=(C_1,C_2,C_3)\in\RR_+^3\) and fit the following ansatz separately for each method. For Polyak,
\[
F_{\polyak}(k;\Theta)=C_1\beta^k+C_2(\eta_{\beta} k)^{-s}+C_3\frac{\sigma^2}{B}\eta_{\beta},\qquad \eta_{\beta}:=\frac{\eta}{1-\beta},
\]
while for Nesterov,
\[
F_{\nesterov}(k;\Theta)=C_1\min\{1,(\eta k)^{-s}\}\beta^k+C_2(\eta_{\beta} k)^{-s}+C_3\frac{\sigma^2}{B}\min\{\eta_{\beta},\eta_{\beta}^{1/q}\}.
\]
The first terms use the transient envelopes from Theorems~\ref{thm:hbm scaling law} and~\ref{thm:nesterov scaling law}, so the fits test the predicted scaling forms rather than their multiplicative constants. For each method, the coefficients are fitted by minimizing the mean squared relative error:
\[
\min_{\Theta\in\RR_+^3}\frac{1}{|\mathcal I_{\rm fit}|}\sum_{k\in \mathcal I_{\rm fit}}\left(\frac{F(k;\Theta)-\overline{\cE}(k)}{\overline{\cE}(k)}\right)^2,\qquad \mathcal I_{\rm fit}:=\{k\geq1:\eta_{\beta} k\geq1\},
\]
where \(F=F_{\polyak}\) or \(F_{\nesterov}\) for the corresponding method.

\paragraph{Additional results.} We repeat the same experiment for \((s,q)\in\{(0.3,6),(0.5,4),(0.5,6)\}\), keeping \(\sigma=1\), \(B=32\), \(\eta=0.03\), and \(\beta=0.98\). Each run uses \(5\times10^5\) iterations, and the reported risk is averaged over \(100\) independent runs. Figure~\ref{fig:additional-scaling-beta2} shows that the predicted scaling forms consistently track the empirical risk dynamics across these choices of \((s,q)\).

\begin{figure}[H]
    \centering
    \captionsetup[subfigure]{skip=0pt, font=small}
    \begin{subfigure}[t]{0.32\textwidth}
        \centering
        \includegraphics[width=\linewidth]{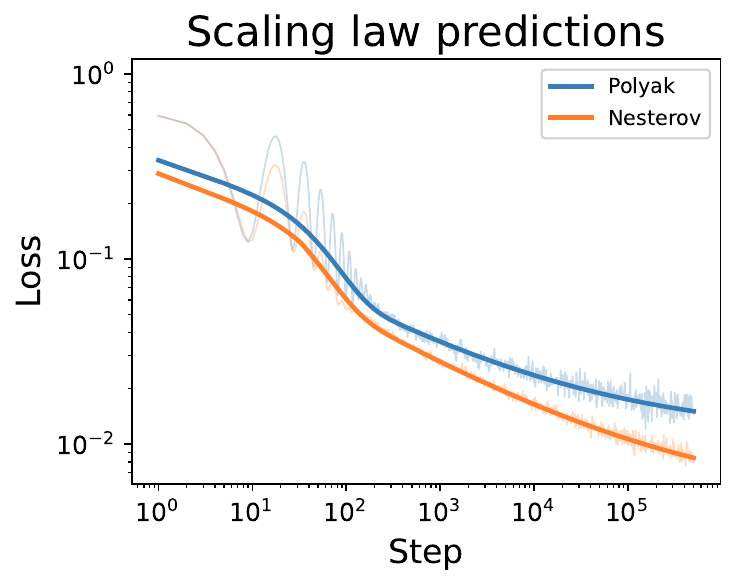}
        \caption{\((s,q)=(0.3,6)\)}
    \end{subfigure}
    \hfill
    \begin{subfigure}[t]{0.32\textwidth}
        \centering
        \includegraphics[width=\linewidth]{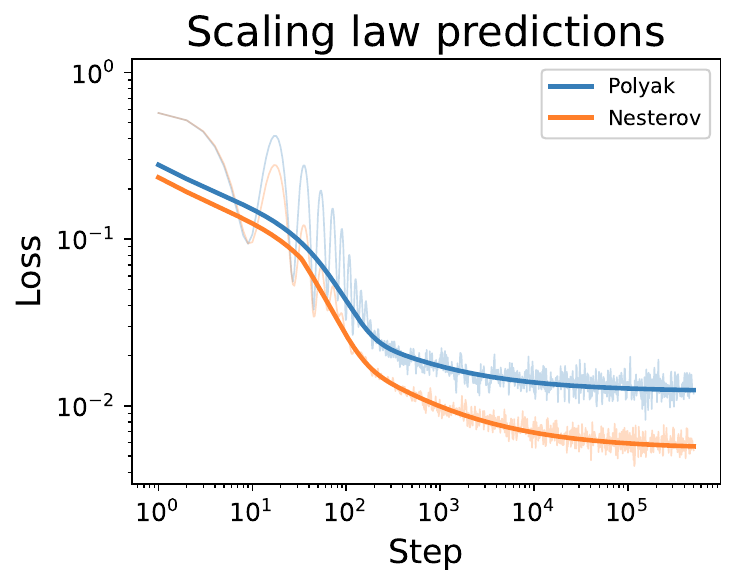}
        \caption{\((s,q)=(0.5,4)\)}
    \end{subfigure}
    \hfill
    \begin{subfigure}[t]{0.32\textwidth}
        \centering
        \includegraphics[width=\linewidth]{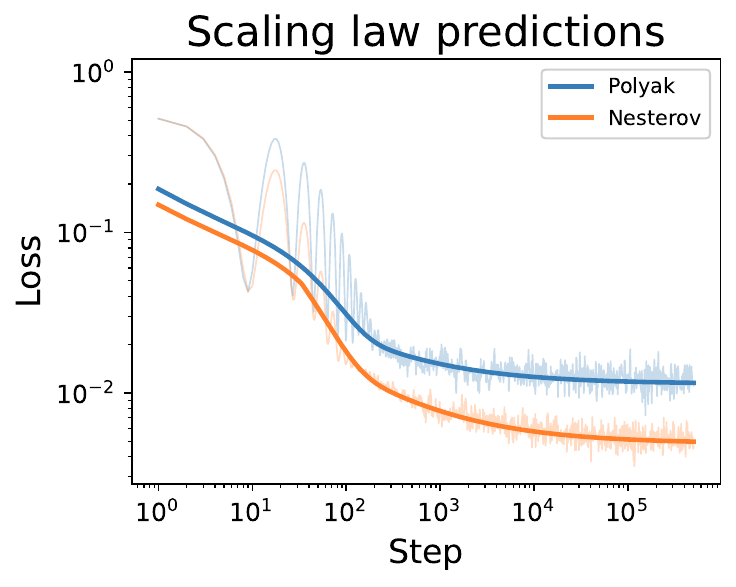}
        \caption{\((s,q)=(0.5,6)\)}
    \end{subfigure}
    \caption{\textbf{Additional validation of the risk-dynamics scaling laws for Polyak and Nesterov.} Thin translucent lines show empirical  excess risks, while thick solid curves show the fitted scaling-law predictions. All experiments use \(\sigma=1\), \(B=32\), \(\eta=0.03\), and \(\beta=0.98\), run for \(5\times10^5\) iterations, and report averages over \(100\) independent runs.}
    \label{fig:additional-scaling-beta2}
\end{figure}

\subsubsection{Batch-size phase diagram at a fixed data budget}
\label{app:critical-batch-size}
\label{app:cbs-plkr}

We simulate PLK regression with \((s,q)=(0.3,4)\) and \(\sigma=0.1\), with all methods initialized at \(\btheta_{-1}=\btheta_0=0\). We fix the total data budget at \(D=2^{14}\). 
For each batch size, we search locally around the hyperparameter scalings predicted by Appendix~\ref{app:sec:data-scaling}. Each candidate is evaluated over 50 independent runs. To reduce sensitivity to oscillations of the last iterate, we use the excess risk averaged over the final \(1\%\) of iterations as a numerical proxy for the final-step excess risk. Let \(W=\max\{1,\lceil0.01 N\rceil\}\). We score each candidate by the mean excess risk over the final \(W\) iterations and across runs:
\[
\overline{\mathcal E}=\frac1{50}\sum_{r=1}^{50}\frac1 W\sum_{k=N-W+1}^{N}\mathcal E(\btheta_k^{(r)}).
\]
For each method and batch size, Figure~\ref{fig:intro-result} (right) reports the minimum \(\overline{\mathcal E}\) over the corresponding search grid.

Figure~\ref{fig:additional-cbs} shows additional  results on the batch-size phase diagram across \((s,q)\).

\begin{figure}[H]
    \centering
    \captionsetup[subfigure]{skip=0pt, font=small}
    \begin{subfigure}[t]{0.325\textwidth}
        \centering
        \includegraphics[width=\linewidth]{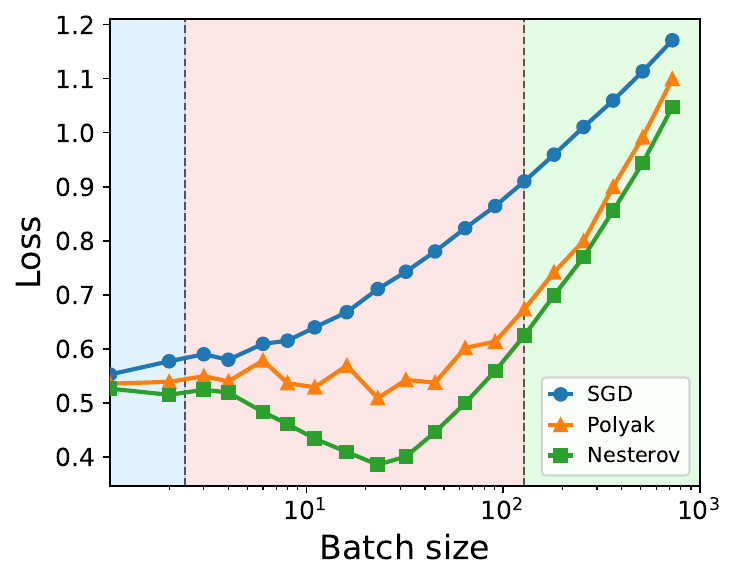}
        \caption{\(s=0.1,q=2\)}
    \end{subfigure}
    \hfill
    \begin{subfigure}[t]{0.325\textwidth}
        \centering
        \includegraphics[width=\linewidth]{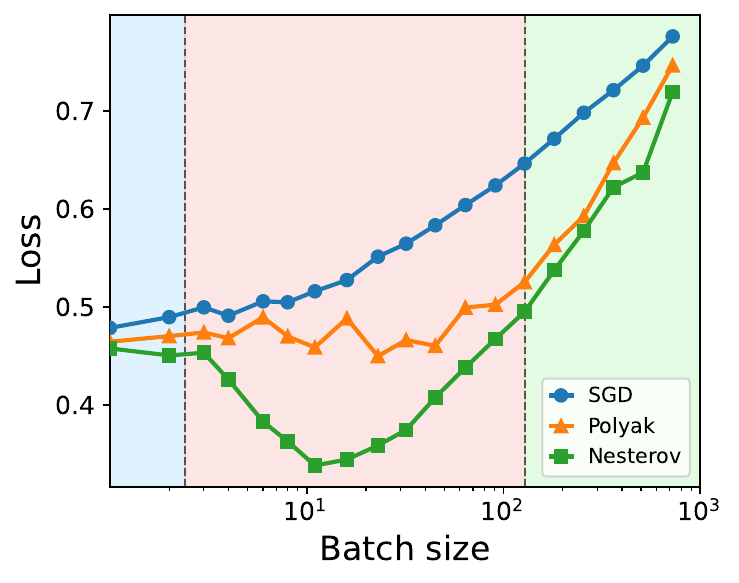}
        \caption{\(s=0.1,q=4\)}
    \end{subfigure}
    \begin{subfigure}[t]{0.325\textwidth}
        \centering
        \includegraphics[width=\linewidth]{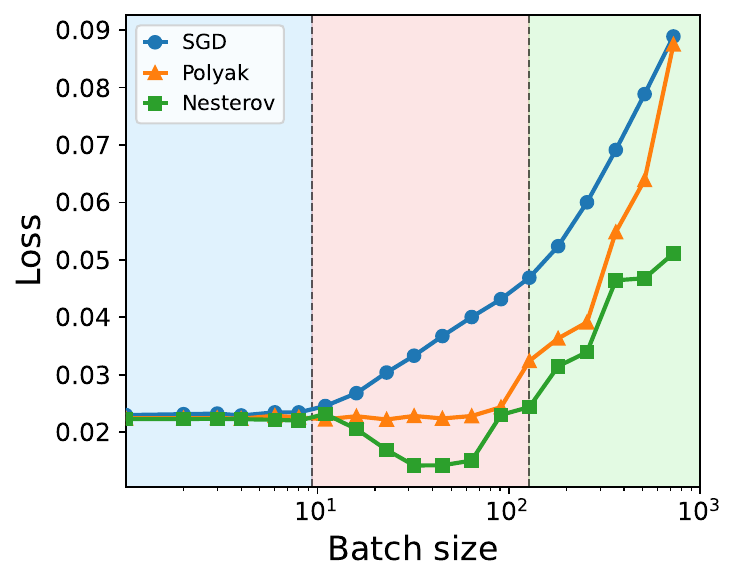}
        \caption{\(s=0.3,q=6\)}
    \end{subfigure}
    \caption{Additional  experiments on the batch-size phase diagram under  different \((s,q)\).}
    \label{fig:additional-cbs}
\end{figure}

\subsubsection{Data-scaling exponents across batch-size regimes}
\label{app:data-scaling-experiments}

In Figure~\ref{fig:three-regime-plk}, we simulate PLK regression with \((s,q)=(0.3,4)\), \(\sigma=0.1\), and \(\btheta_{-1}=\btheta_0=0\), using data budgets \(D\in\{2000,4000,8000,16000,32000,64000,128000\}\). For a prescribed batch-size exponent \(b\), we set \(B=\operatorname{round}(C_BD^b)\) and \(N=\lfloor D/B\rfloor\), where \(C_B\) is an order-one constant. For each method and \(D\), we tune the multiplicative constants around the learning-rate and momentum scalings prescribed by Appendix~\ref{app:sec:data-scaling}, and report the best stable configuration. Each run is evaluated by averaging the risk over the final \(1\%\) of optimization steps, and each point is averaged over 10 independent runs.

We evaluate one representative batch-size exponent from each of the three regimes.

\begin{itemize}
\item \textbf{Small batch.} We take \(b_{\rm small}=0\). For all three methods, we search around \(\eta\eqsim D^{-s/(s+1)}\). For Polyak and Nesterov, we additionally use \(1-\beta\eqsim1\).

\item \textbf{Large batch.} We take \(b_{\rm large}=b_2-\delta\), where \(b_2=\frac{2sq+1}{2sq+q+1}\) and \(\delta>0\) is small. SGD uses \(\eta\eqsim1\). For Polyak, we search around \(\eta\eqsim1\) and \(1-\beta\eqsim D^{-(b_{\rm large}-s/(s+1))}\). For Nesterov, we search around \(\eta\eqsim1\) and
\[
1-\beta\eqsim D^{-m_{\nesterov}},\qquad m_{\nesterov}=b_{\rm large}-\frac{sq-(q-1)b_{\rm large}}{sq+1}.
\]

\item \textbf{Ultra-large batch.} We choose \(b_3<b_{\rm ultra}<1\). SGD uses \(\eta\eqsim1\), while for both momentum methods we search around \(\eta\eqsim1\) and \(1-\beta\eqsim D^{-(1-b_{\rm ultra}-\delta)}\), with \(\delta>0\) chosen so that \(N(1-\beta)\to\infty\).
\end{itemize}

\paragraph{Additional results.} Figures~\ref{fig:app data s0.1beta4} and \ref{fig:app data s0.3beta2} show additional data-scaling results across \((s,q)\).

\begin{figure}[H]
    \centering
    \begin{subfigure}[t]{0.325\textwidth}
        \centering
        \includegraphics[width=\linewidth]{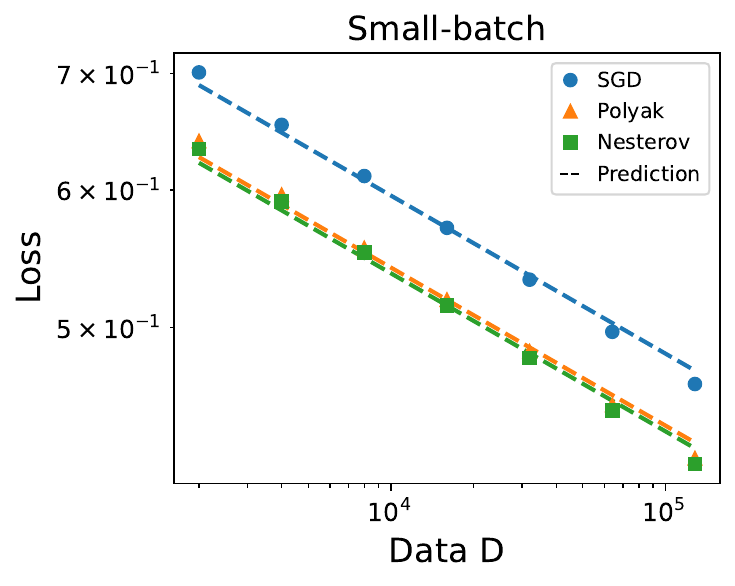}
    \end{subfigure}
    \hfill
    \begin{subfigure}[t]{0.325\textwidth}
        \centering
        \includegraphics[width=\linewidth]{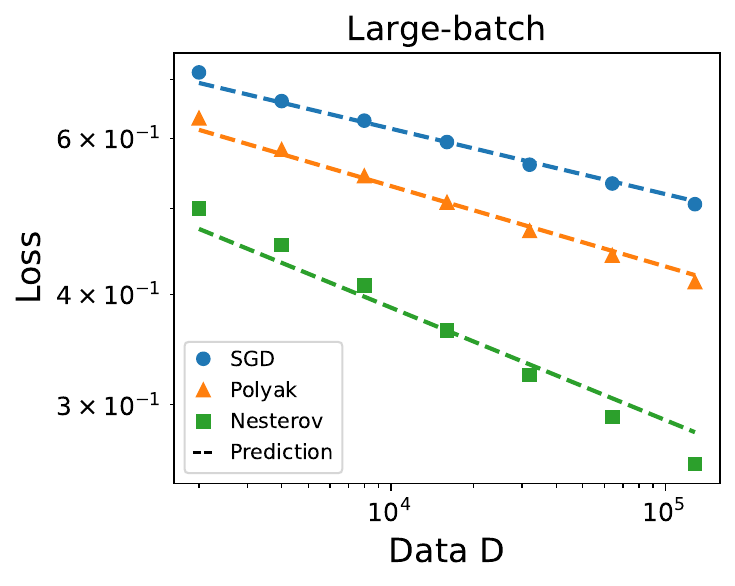}
    \end{subfigure}
    \hfill
    \begin{subfigure}[t]{0.325\textwidth}
        \centering
        \includegraphics[width=\linewidth]{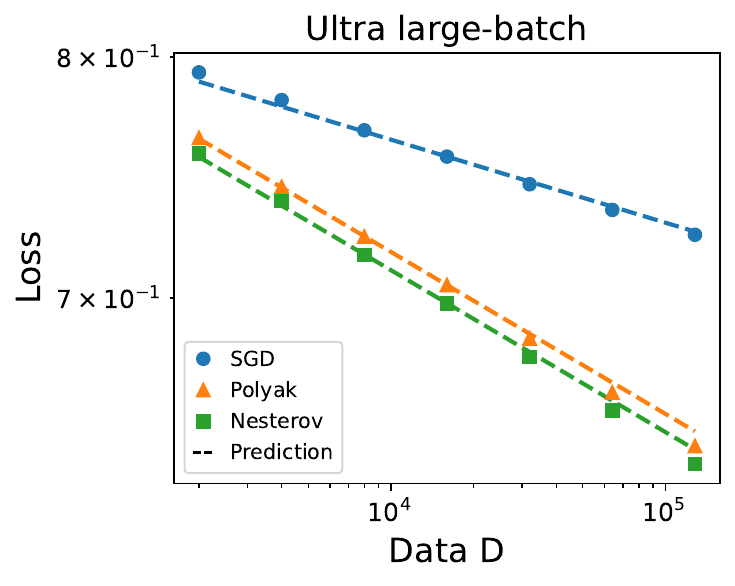}
    \end{subfigure}
    \vspace{-1em}
    \caption{Data scaling with \((s,q)=(0.1,4)\).}
    \label{fig:app data s0.1beta4}
    \vspace{1em}
    \begin{subfigure}[t]{0.325\textwidth}
        \centering
        \includegraphics[width=\linewidth]{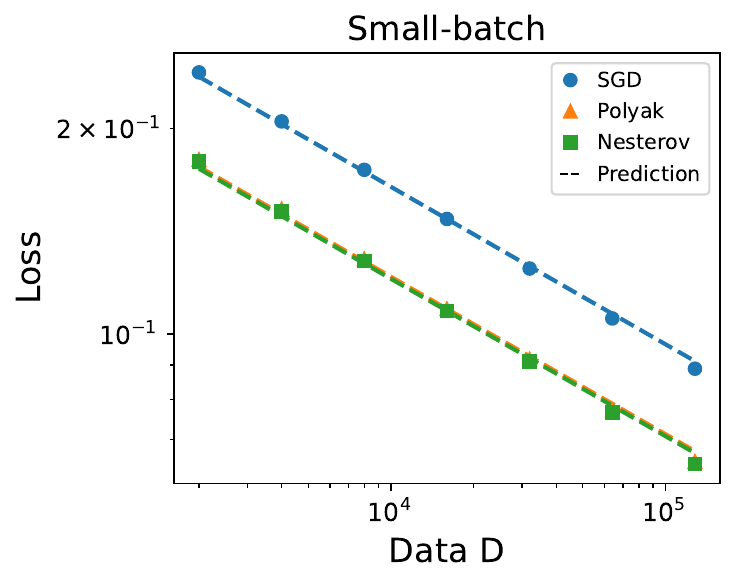}
    \end{subfigure}
    \hfill
    \begin{subfigure}[t]{0.325\textwidth}
        \centering
        \includegraphics[width=\linewidth]{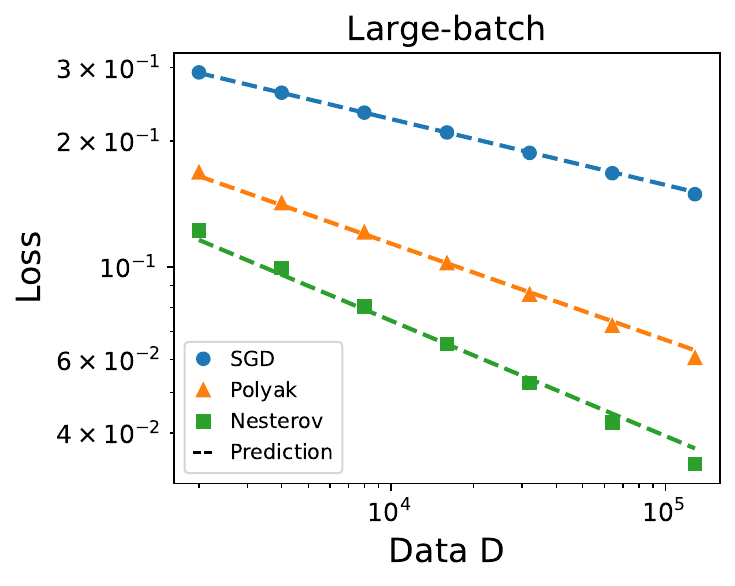}
    \end{subfigure}
    \hfill
    \begin{subfigure}[t]{0.325\textwidth}
        \centering
        \includegraphics[width=\linewidth]{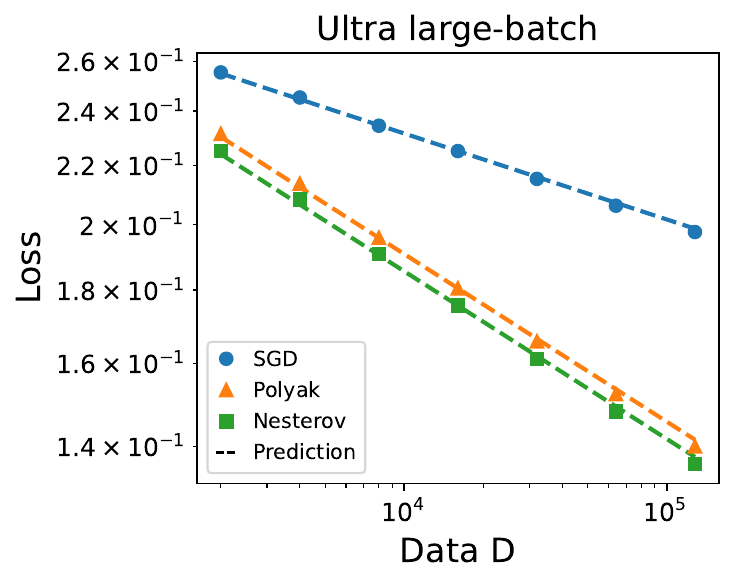}
    \end{subfigure}
     \vspace{-1em}
    \caption{Data scaling with \((s,q)=(0.3,2)\).}
    \label{fig:app data s0.3beta2}
\end{figure}

\subsection{Neural networks}
\label{app:neural-experiments}

\subsubsection{ImageNet}
\label{app:cbs-imagenet}

We train ResNet-18 from scratch on ImageNet-1K, which contains \(1{,}281{,}167\) training images and \(50{,}000\) validation images
across \(1{,}000\) classes. We replace batch normalization with group normalization to avoid dependence on batch statistics when varying the batch size.

We consider \(B\in\{2^0,2^1,\ldots,2^{12}\}\). For each method and batch size, we fix the total sample budget at \(D=3\times1{,}281{,}167\), equivalent to three training epochs, and perform \(N=\lfloor D/B\rfloor\) optimizer updates. We train using cross-entropy loss with a constant learning rate and momentum factor, without learning-rate warmup, weight decay, or label smoothing. All runs use BF16 mixed precision and the same random seed. For \(B>128\), we accumulate gradients over micro-batches of size \(128\) to obtain the prescribed batch size.

For each method and batch size, we perform a grid search over the learning rate and, for Polyak and Nesterov, the momentum factor. We evaluate each final model on the full validation set. Figure~\ref{fig:imagenet-cbs} reports the lowest validation cross-entropy among the completed stable configurations for each method and batch size.

\begin{figure}[H]
    \centering
    \includegraphics[width=0.4\textwidth]
    {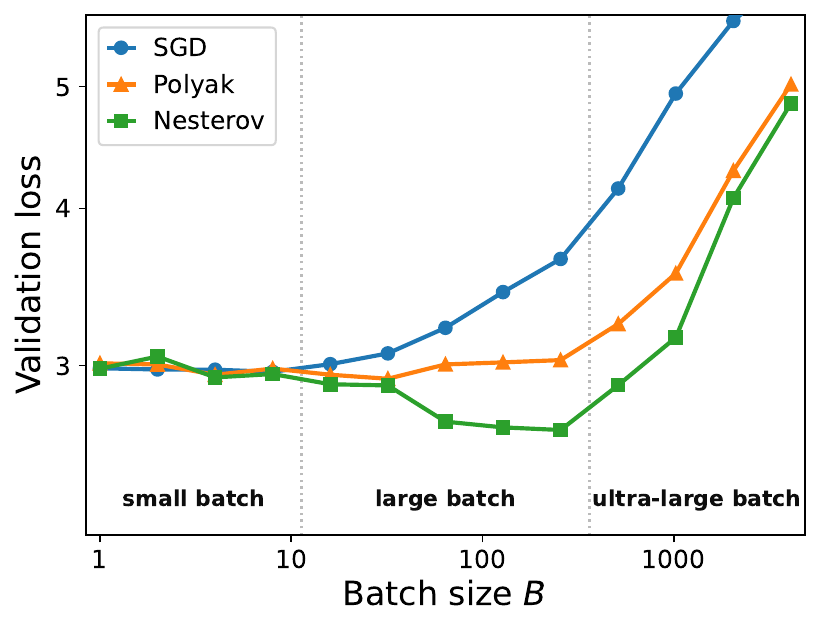}
    \caption{Batch-size dependence on ImageNet-1K under
    a fixed sample budget.
    Each point reports the lowest final validation cross-entropy among the evaluated stable hyperparameter configurations for the corresponding method and batch size.}
    \label{fig:imagenet-cbs}
\end{figure}

\subsubsection{CIFAR-10}
\label{app:stability-neural}

We further test whether the stability behavior predicted by the PLK analysis in Section~\ref{sec:stability} also appears in neural-network training. We train ResNet-18 on CIFAR-10 using cross-entropy loss,
replacing batch normalization with group normalization
and adapting the initial convolution to \(32\times32\) inputs. Each run starts from a random initialization and performs \(N=2000\) updates with constant learning rates and momentum factors, without warmup or weight decay. For the batch-size scan, we fix \(1-\beta=2^{-6}\) and vary \(B\in\{2^0,2^1,\ldots,2^8\}\). For the damping scan, we fix \(B=4\) and vary \(1-\beta\in\{2^{-6},2^{-5.5},\ldots,2^{-2}\}\).
Learning rates are scanned logarithmically from \(2^{-8}\) to \(8\), with additional learning rates and random seeds near the observed transitions.

We evaluate loss on a fixed subset of \(1024\) training
images every \(50\) updates. A completed run is stable if its average loss over the final ten evaluations is at most twice its initial loss and \(1.25\) times the average over the preceding ten evaluations. Runs terminated by numerical divergence or safety limits are classified as unstable.

Figure~\ref{fig:cifar-stability-boundaries} shows the results. We can see that the SGD boundary becomes approximately constant beyond the smallest batch sizes, while the momentum boundaries increase with \(B\) before saturation. At fixed \(B=4\), both momentum boundaries initially increase with \(1-\beta\) and then approach plateaus, with Nesterov saturating earlier.
These observations are qualitatively consistent with the
growth and saturation predicted by the theory, while
the precise scaling may vary in neural network training.

\begin{figure}[H]
    \centering
    \includegraphics[width=\textwidth]
    {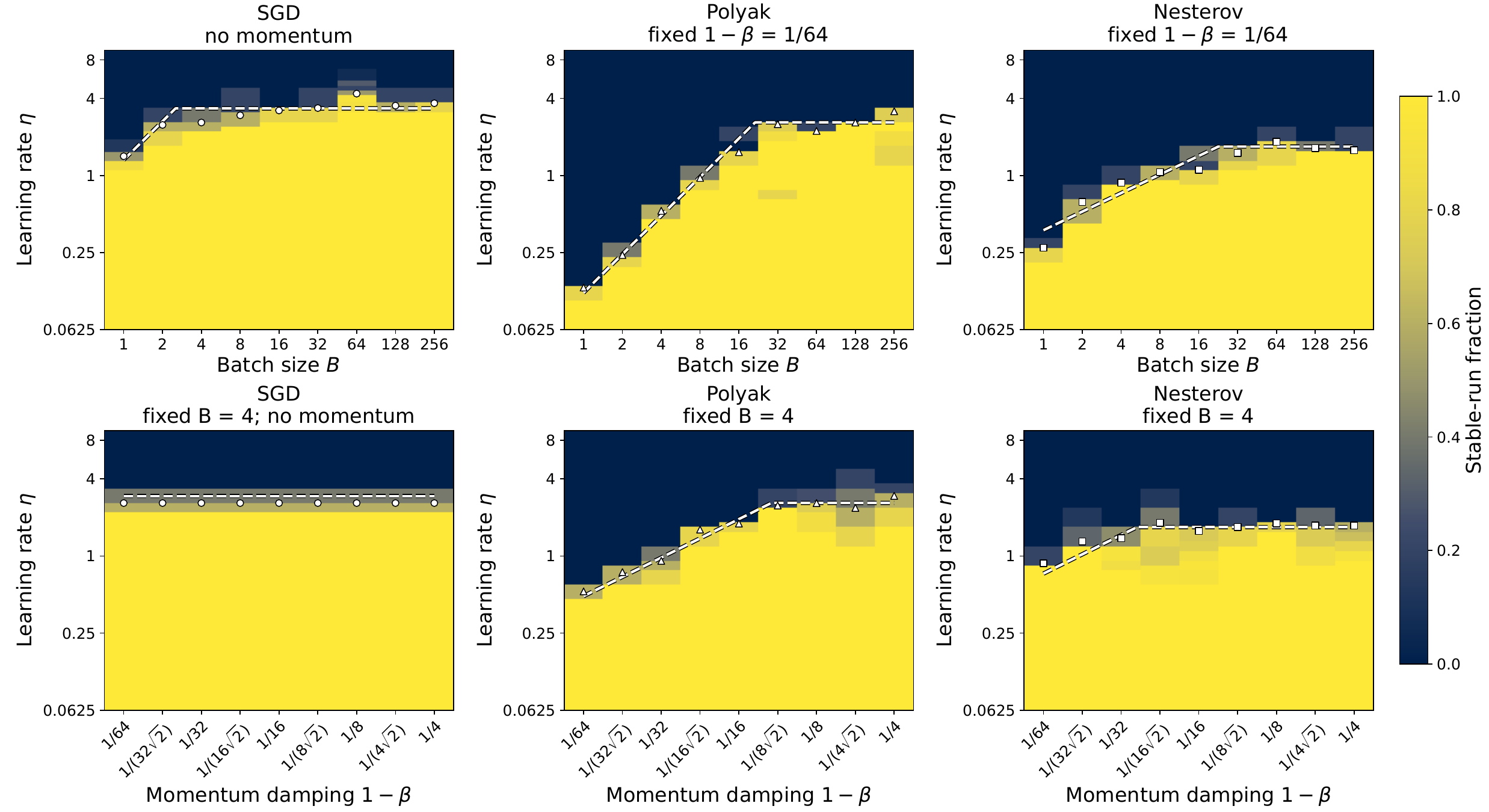}
    \caption{
        Empirical risk stability of ResNet-18 on CIFAR-10
        over \(N=2000\) updates.
        \textbf{Top:} varying \(B\), with \(1-\beta=2^{-6}\) for the momentum methods.
        \textbf{Bottom:} varying \(1-\beta\) at fixed \(B=4\). Colors show stable-run fractions, white markers indicate estimated \(50\%\) stability boundaries, and dashed curves show fitted relations.
    }
    \label{fig:cifar-stability-boundaries}
\end{figure}

\end{document}